\documentclass[dvipsnames,table]{article}

\usepackage{microtype}
\usepackage{graphicx}
\usepackage{subfigure}
\usepackage{booktabs} 

\usepackage{hyperref}

\usepackage[accepted]{icml2023}

\usepackage{amsmath}
\usepackage{amssymb}
\usepackage{mathtools}
\usepackage{amsthm}

\usepackage[capitalize,noabbrev]{cleveref}

\usepackage[textsize=tiny]{todonotes}
\usepackage{makecell}

\usepackage[utf8]{inputenc} 
\usepackage[T1]{fontenc}    
\usepackage{hyperref}       
\usepackage{url}            
\usepackage{booktabs}       
\usepackage{amsfonts}       
\usepackage{nicefrac}       
\usepackage{microtype}      
\usepackage{mathtools}
\usepackage{natbib}
\usepackage{xcolor}
\usepackage{amsmath}
\usepackage{mathtools}
\usepackage{amsthm}
\usepackage{amssymb}

\usepackage[algo2e, ruled, vlined, noend]{algorithm2e} 
\usepackage{algorithm}

\SetCommentSty{mycommfont}

\newcommand{\calA}{{\mathcal{A}}}

\newcommand{\calC}{{\mathcal{C}}}
\newcommand{\calV}{{\mathcal{V}}}

\newcommand{\calX}{{\mathcal{X}}}
\newcommand{\calS}{{\mathcal{S}}}
\newcommand{\calF}{{\mathcal{F}}}
\newcommand{\calG}{\mathcal{G}}
\newcommand{\calI}{{\mathcal{I}}}

\newcommand{\calK}{{\mathcal{K}}}
\newcommand{\calL}{{\mathcal{L}}}
\newcommand{\calD}{{\mathcal{D}}}
\newcommand{\calE}{{\mathcal{E}}}
\newcommand{\calR}{{\mathcal{R}}}
\newcommand{\calT}{{\mathcal{T}}}

\newcommand{\calU}{{\mathcal{U}}}

\newcommand{\calM}{{\mathcal{M}}}
\newcommand{\calN}{{\mathcal{N}}}

\newcommand{\one}{\boldsymbol{1}}

\DeclareMathOperator*{\argmin}{argmin}

\newcommand{\eat}[1]{}

\newcommand{\rbr}[1]{\left(#1\right)}
\newcommand{\sbr}[1]{\left[#1\right]}
\newcommand{\cbr}[1]{\left\{#1\right\}}

\newcommand{\abr}[1]{\left|#1\right|}
\newcommand{\bigO}[1]{\order\left( #1 \right)}
\newcommand{\tilO}[1]{\otil\left( #1 \right)}
\newcommand{\lowO}[1]{\lorder\left( #1 \right)}
\newcommand{\bigo}[1]{\order( #1 )}
\newcommand{\tilo}[1]{\otil( #1 )}
\newcommand{\lowo}[1]{\lorder( #1 )}

\newcommand{\frA}{\mathfrak{A}}

\newcommand{\dev}{\textsc{Dev}}

\newcommand{\var}{\textsc{Var}}

\renewcommand{\P}{\bar{P}}

\newcommand{\optV}{V^{\star}}
\newcommand{\optQ}{Q^{\star}}

\newcommand{\sumi}{\sum_{i=1}^{I_k}}
\newcommand{\hatQ}{\widehat{Q}}
\newcommand{\hatV}{\widehat{V}}

\newcommand{\sumk}{\sum_{k=1}^K}

\newcommand{\barpi}{\bar{\pi}}

\newcommand{\tilP}{\widetilde{P}}

\newcommand{\tilpi}{{\widetilde{\pi}}}

\newcommand{\N}{\mathbf{N}} 

\renewcommand{\ln}{\log}
\newcommand{\true}{\textsc{True}\xspace}
\newcommand{\false}{\textsc{False}\xspace}
\newcommand{\pc}{\textsc{PolicyConsolidation}\xspace}
\newcommand{\algop}{\algo{}\textsuperscript{+}\xspace}
\newcommand{\algo}{\textsc{LASD}\xspace}
\newcommand{\algopc}{\textsc{PC}\xspace}
\newcommand{\ucbexplore}{\textsc{UcbExplore}\xspace}
\newcommand{\disco}{\textsc{DisCo}\xspace}
\newcommand{\valae}{\textsc{VALAE}\xspace}
\newcommand{\algoax}{\textsc{LAE}\xspace}
\newcommand{\Snext}{\calS_{\text{next}}}
\newcommand{\tset}{\calK} 
\newcommand{\rS}[2]{\calT_{#2}(#1)}
\newcommand{\fillc}{\textsc{Explore}\xspace}
\newcommand{\nmin}{n_{\min}}
\newcommand{\rtest}{\textsc{RTest}\xspace}
\newcommand{\tilp}{\widetilde{p}}
\newcommand{\frakN}{\mathfrak{N}}

\newcommand{\bV}{\bar{V}}
\newcommand{\bpi}{\bar{\pi}}
\newcommand{\calW}{\mathcal{W}}
\newcommand{\tot}{\text{tot}}
\newcommand{\sumkp}{\sum_{k=1}^{K'}}
\newcommand{\VI}{\text{VI}\xspace}
\newcommand{\VISGO}{\textsc{VISGO}\xspace}
\newcommand{\calO}{\mathcal{O}}
\newcommand{\acalO}{\calO^{\rightarrow}}
\newcommand{\aS}{S^{\rightarrow}}
\newcommand{\bcalU}{\bar{\calU}}
\newcommand{\dV}{V_{\dagger}}
\newcommand{\calUstar}{\calU^{\star}}
\newcommand{\AX}{\text{AX}}
\newcommand{\reset}{\textsc{RESET}}

\newcommand{\acalS}{\calS^{\rightarrow}}
\newcommand{\gstar}{g^{\star}}
\newcommand{\hattau}{\widehat{\tau}}
\newcommand{\calKstar}{\calK^{\star}}

\newcommand{\hatpi}{\widehat{\pi}}

\newcommand{\field}[1]{\mathbb{#1}}

\newcommand{\fR}{\field{R}}

\newcommand{\fN}{\field{N}}
\newcommand{\E}{\field{E}}
\newcommand{\fV}{\field{V}}

\newcommand{\Ind}{\field{I}}

\newcommand{\norm}[1]{\left\|{#1}\right\|}

\newtheorem{lemma}{Lemma}
\newtheorem{theorem}{Theorem}
\newtheorem{corollary}[theorem]{Corollary}
\newtheorem{remark}{Remark}

\newtheorem{definition}{Definition}
\newtheorem{assumption}{Assumption}

\newcommand{\order}{\ensuremath{\mathcal{O}}}
\newcommand{\lorder}{\ensuremath{\Omega}}
\newcommand{\otil}{\ensuremath{\tilde{\mathcal{O}}}}

\usepackage{prettyref}
\newcommand{\pref}[1]{\prettyref{#1}}
\newcommand{\pfref}[1]{Proof of \prettyref{#1}}
\newcommand{\savehyperref}[2]{\texorpdfstring{\hyperref[#1]{#2}}{#2}}
\newrefformat{eq}{\savehyperref{#1}{Eq.~\textup{(\ref*{#1})}}}
\newrefformat{eqn}{\savehyperref{#1}{Equation~\ref*{#1}}}
\newrefformat{lem}{\savehyperref{#1}{Lemma~\ref*{#1}}}
\newrefformat{def}{\savehyperref{#1}{Definition~\ref*{#1}}}
\newrefformat{line}{\savehyperref{#1}{Line~\ref*{#1}}}
\newrefformat{thm}{\savehyperref{#1}{Theorem~\ref*{#1}}}
\newrefformat{corr}{\savehyperref{#1}{Corollary~\ref*{#1}}}
\newrefformat{cor}{\savehyperref{#1}{Corollary~\ref*{#1}}}
\newrefformat{sec}{\savehyperref{#1}{Section~\ref*{#1}}}
\newrefformat{subsec}{\savehyperref{#1}{Section~\ref*{#1}}}
\newrefformat{app}{\savehyperref{#1}{Appendix~\ref*{#1}}}
\newrefformat{assum}{\savehyperref{#1}{Assumption~\ref*{#1}}}
\newrefformat{ex}{\savehyperref{#1}{Example~\ref*{#1}}}
\newrefformat{fig}{\savehyperref{#1}{Figure~\ref*{#1}}}
\newrefformat{alg}{\savehyperref{#1}{Algorithm~\ref*{#1}}}
\newrefformat{rem}{\savehyperref{#1}{Remark~\ref*{#1}}}
\newrefformat{conj}{\savehyperref{#1}{Conjecture~\ref*{#1}}}
\newrefformat{prop}{\savehyperref{#1}{Proposition~\ref*{#1}}}
\newrefformat{proto}{\savehyperref{#1}{Protocol~\ref*{#1}}}
\newrefformat{prob}{\savehyperref{#1}{Problem~\ref*{#1}}}
\newrefformat{claim}{\savehyperref{#1}{Claim~\ref*{#1}}}
\newrefformat{que}{\savehyperref{#1}{Question~\ref*{#1}}}
\newrefformat{op}{\savehyperref{#1}{Open Problem~\ref*{#1}}}
\newrefformat{fn}{\savehyperref{#1}{Footnote~\ref*{#1}}}
\newrefformat{tab}{\savehyperref{#1}{Table~\ref*{#1}}}
\newrefformat{fig}{\savehyperref{#1}{Figure~\ref*{#1}}}
\newrefformat{prop}{\savehyperref{#1}{Property~\ref*{#1}}}

\usepackage{xspace}
\usepackage{enumitem} 

\usetikzlibrary{shapes.geometric}

\definecolor{darkgreen}{rgb}{0.0, 0.6, 0.0}
\definecolor{lightgray}{rgb}{0.83, 0.83, 0.83}

\icmltitlerunning{Layered State Discovery for Incremental Autonomous Exploration}

\begin{document}

\twocolumn[
\icmltitle{Layered State Discovery for Incremental Autonomous Exploration}



\icmlsetsymbol{equal}{*}

\begin{icmlauthorlist}
\icmlauthor{Liyu Chen}{usc}
\icmlauthor{Andrea Tirinzoni}{meta}
\icmlauthor{Alessandro Lazaric}{meta}
\icmlauthor{Matteo Pirotta}{meta}
\end{icmlauthorlist}

\icmlaffiliation{usc}{University of Southern California}
\icmlaffiliation{meta}{Meta}

\icmlcorrespondingauthor{Liyu Chen}{liyuc@usc.edu}

\icmlkeywords{Machine Learning, ICML}

\vskip 0.3in
]



\printAffiliationsAndNotice{}  


\begin{abstract}
    We study the autonomous exploration (AX) problem proposed by~\citet{lim2012autonomous}. In this setting, the objective is to discover a set of $\epsilon$-optimal policies reaching a set $\mathcal{S}^{\rightarrow}_L$ of incrementally $L$-controllable states. We introduce a novel layered decomposition of the set of incrementally $L$-controllable states that is based on the iterative application of a state-expansion operator. 
    We leverage these results to design Layered Autonomous Exploration (LAE), a novel algorithm for AX that attains a sample complexity of $\tilde{\mathcal{O}}(LS^{\rightarrow}_{L(1+\epsilon)}\Gamma_{L(1+\epsilon)} A \log^{12}(S^{\rightarrow}_{L(1+\epsilon)})/\epsilon^2)$, where $S^{\rightarrow}_{L(1+\epsilon)}$ is the number of states that are incrementally $L(1+\epsilon)$-controllable, $A$ is the number of actions, and $\Gamma_{L(1+\epsilon)}$ is the branching factor of the transitions over such states.
    \textsc{LAE} improves over the algorithm of~\citet{tarbouriech2020improved} by a factor of $L^2$ and it is the first algorithm for AX that works in a countably-infinite state space.
    Moreover, we show that, under a certain identifiability assumption, LAE achieves minimax-optimal sample complexity of $\tilde{\mathcal{O}}(LS^{\rightarrow}_{L}A\log^{12}(S^{\rightarrow}_{L})/\epsilon^2)$, outperforming existing algorithms and matching for the first time the lower bound proved by~\citet{cai2022near} up to logarithmic factors.
\end{abstract}

\begin{table*}[t]
    \renewcommand{\arraystretch}{1.6}
    \centering
    \caption{\small Comparison between this work and previous work.
        Here, $L$ is the exploration radius, $S$ is the number of states, $\aS_{L(1+\epsilon)}$ is the number of incrementally $L(1+\epsilon)$-controllable states, $\Gamma_{L(1+\epsilon)}$ is the branching factor of transition over such states, $A$ is the number of actions, and $\epsilon$ is the target accuracy.
        The $\AX$ objectives are defined in~\pref{def:ax} and are such that $\AX^+ \Rightarrow \AX^\star \Rightarrow \AX_L$. We only display the dominating term in $1/\epsilon$. Note that $\aS_{2L}$ may be much larger (even exponentially) than $\aS_{L(1+\epsilon)}$ in certain MDPs (\pref{lem:example 2L}).
    }
    \label{tab:summary}
    \resizebox{\textwidth}{!}{
    \begin{tabular}{|c|c||c|c|c|}
    \hline
    \multicolumn{2}{|c||}{Algorithm} & Sample Complexity & Objective & $S$ dependency \\
    \hline
    \hline
    UcbExplore & \citep{lim2012autonomous} & $\tilO{L^3\aS_{L(1+\epsilon)}\Gamma_{L(1+\epsilon)}A/\epsilon^3}$ & $\AX_L$ & $\ln S$\\
    \hline
    DisCo & \citep{tarbouriech2020improved} & $\tilO{L^3\aS_{L(1+\epsilon)}\Gamma_{L(1+\epsilon)}A/\epsilon^2}$ & $\AX^{\star}$ & $\ln S$\\
    \hline
    \valae & \citep{cai2022near} & $\tilO{L\aS_{2L}A/\epsilon^2}$ & $\AX^{\star}$ & $\ln S$ \\
    \hline
    \hline
    \cellcolor{lightgray}
    \algoax (\pref{alg:ax.plus}) & 
    \cellcolor{lightgray}
    Ours &
    \cellcolor{lightgray}
    $\tilO{L\aS_{L(1+\epsilon)}\Gamma_{L(1+\epsilon)} A/\epsilon^2}$ & 
    \cellcolor{lightgray}
    $\AX^+$ &
    \cellcolor{lightgray}
    $\ln\aS_{L(1+\epsilon)}$ \\
    \hline
    \cellcolor{lightgray}
    \algoax with~\pref{assum:id} &
    \cellcolor{lightgray}
    Ours &
    \cellcolor{lightgray}
    $\tilO{L\aS_{L}A/\epsilon^2}$ &
    \cellcolor{lightgray}
    $\AX^+$ &
    \cellcolor{lightgray}
    $\ln\aS_{L}$ \\
    \hline
    \hline
    \makecell{Lower Bound\\ ($\aS_L = \aS_{L(1+\epsilon)}$ by construction)} & \citep{cai2022near} & $\lowO{L\aS_{L}A/\epsilon^2}$ & $\AX_L$ & - \\
    \hline
    \end{tabular}
    }
\end{table*}

\section{Introduction}
A distinctive feature of intelligent beings is the ability to explore an unknown environment without any supervision or extrinsic reward while learning skills that solve tasks (e.g., reaching goal states) of increasing difficulty. \citet{lim2012autonomous} first proposed a formal framework of \textit{autonomous exploration} in reinforcement learning (RL) as the process of progressively discovering states within a certain distance from an initial state $s_0$ at the same time as learning near-optimal policies to reach them. \citet{lim2012autonomous} also devised the first sample efficient exploration algorithm (\ucbexplore) for this setting, while its sample complexity and optimality guarantees were later improved by \disco~\citep{tarbouriech2020improved} and \valae~\citep{cai2022near}.

In this paper, we make several contributions to this problem:
\begin{itemize}[topsep=0pt,parsep=0pt,partopsep=0pt,leftmargin=10pt]
    \item Given an initial state $s_0$, the autonomous exploration objective is built upon the concept of incrementally $L$-controllable states, i.e., states that can be reached within $L$ steps from $s_0$ by only traversing incrementally $L$-controllable states\footnote{We say that a state $s$ is $L$-controllable if there exists a policy that reaches $s$ from $s_0$ in less than $L$ steps on average. In general an $L$-controllable state may be reached by policies traversing states that are not $L$-controllable themselves.}. While the original definition of the set of incrementally $L$-controllable states $\calS_L^{\rightarrow}$ involves considering all possible partial orders of states in the environment, we derive an equivalent constructive definition that reveals the \emph{layered} structure of $\calS_L^{\rightarrow}$, where each layer can be obtained as the set of states that can be reached in $L$ steps by only traversing states in the previous layers (see~\pref{sec:sL.constructive}). 
    \item We then leverage the layered structure of $\calS_L^{\rightarrow}$ to design Layered Autonomous Exploration (\algoax), a novel algorithm that keeps exploring the environment to learn policies to reach newly discovered states until a new layer can be consolidated and a new step of discovery and learning is started. We prove that the sample complexity of \algoax is bounded as $\tilo{L\aS_{L(1+\epsilon)}\Gamma_{L(1+\epsilon)}A/\epsilon^2}$, where $L$ is the exploration radius, $\aS_{L(1+\epsilon)}$ is the number of states that are incrementally controllable from the initial state within $L(1+\epsilon)$ steps, $\Gamma_{L(1+\epsilon)}$ is the branching factor of the transition function over such states, $A$ is the number actions, and $\epsilon$ is target accuracy. As illustrated in Table~\ref{tab:summary}, this improves the sample complexity of \disco by a factor of $L^2$ and it avoids the scaling with $\aS_{2L}$ of \valae, which in some MDPs may be much larger than $\aS_{L(1+\epsilon)}$, thus making the bound of \algoax preferable. Indeed, in~\pref{lem:example 2L} in appendix we show that $\aS_{2L}$ may be even exponentially larger than $\aS_{L(1+\epsilon)}$.
    \item Under a certain layer identifiability condition (see ~\pref{assum:id}), we further improve the sample complexity of \algoax to $\tilo{L\aS_{L}A/\epsilon^2}$, which improves w.r.t.\ \valae and matches the lower bound in~\citep{cai2022near}.
    \item Similar to existing algorithms, the sample complexity of \algoax still depends on the logarithm of the total number of states $S$. Since in autonomous exploration the state space is unknown and possibly unbounded, such dependency is highly undesirable. We then design an alternative version of \algoax, which preserves its original sample complexity but replaces the dependency on $\ln S$ with $\ln\aS_{L(1+\epsilon)}$, without requiring any prior knowledge of $\aS_{L(1+\epsilon)}$ (see~\pref{sec:log.adaptivity}).
    \item \algoax also leverages a novel procedure, \pc, that takes a set of states $\calK$ as input and returns goal-conditioned policies reaching each state in $\calK$ with \emph{multiplicative} $\epsilon$-optimality guarantees, which is stronger than previous algorithms and better suited to the autonomous exploration setting (see~\pref{sec:pc}).
\end{itemize}

\paragraph{Related Work}
In reinforcement learning (RL), several approaches to \emph{unsupervised exploration} have been proposed often grounded in concepts such as curiosity~\citep{schmidhuber1991possibility}, intrinsic motivation~\citep{singh2004intrinsically,oudeyer2009intrinsically,bellemare2016unifying,colas2020intrinsically} and with the objective of learning skills in an unsupervised fashion~\citep{gregor2016variational,eysenbach2018diversity,pong2019skew,bagaria2021skill,kamienny2021direct}. On the other hand, a rigorous formalization and theoretical understanding of unsupervised exploration has been rather sparse until recently. \citet{tarbouriech2020active} studied unsupervised exploration for model estimation, \citet{hazan2019provably} formalized the maximum entropy exploration objective, while reward-free RL~\citep[e.g.,][]{jin2020reward,kaufmann2021adaptive,menard2021fast,zhang2021near,tarbouriech2021provably,tarbouriech2022adaptive} studies how to efficiently explore an environment to solve any downstream task near-optimally. As autonomous exploration seeks to learn goal-conditioned policies, it also carries strong technical and algorithmic connections with exploration in the stochastic shortest path problem~\citep[e.g.][]{bertsekas2013stochastic,tarbouriech2020no,tarbouriech2021stochastic,chen2021finding,chen2022near}.

\section{Preliminaries}
We consider a reward-free Markov Decision Process $\calM=(\calS, \calA, s_0, P)$, where $\calS$ is a countable state space, 
$\calA$ is a finite action space, $s_0$ is the initial state, and $P=\{P_{s, a}\}_{(s, a)\in\calS\times\calA}$ with $P_{s,a}\in\Delta_{\calS}$ is the transition function, where $\Delta_{\calS}$ is the simplex over $\calS$.
In a general MDP, the learner may get stuck in undesirable states and be unable to return to $s_0$.
To avoid this issue, we make the following assumption.
\begin{assumption}
    The action space contains a $\reset$ action such that $P_{s, \reset}(s_0)=1$ for all $s\in\calS$.
\end{assumption}

A deterministic stationary policy $\pi\in\calA^{\calS}$ is a mapping that assigns an action $\pi(s)$ to each state $s$, and we define $\Pi=\calA^{\calS}$ as the set of all policies.
To explicitly characterize the behavior of a policy, we say a policy $\pi$ is \emph{restricted} on $\calX\subseteq\calS$ if $\pi(s)=\reset$ for any $s\notin\calX$, and we denote by $\Pi(\calX)$ the set of policies restricted on $\calX$.

We measure the performance of a policy in navigating the MDP as follows.
For any policy $\pi \in \Pi$ and a pair of states $(s, g) \in \mathcal{S}^2$, let $V^{\pi}_g(s) \in [0, + \infty]$ be the expected number of steps it takes to reach $g$ (that is, the \emph{hitting time} of $g$) starting from $s$ when executing policy $\pi$, that is,
\begin{align*}
    V^{\pi}_g(s) &\triangleq \mathbb{E}^\pi\sbr{\left. \omega_g \right|s_1=s}, \\
    \omega_g &\triangleq \inf \cbr{ i \geq 0: s_{i+1} = g }.
\end{align*}
Note that $V^{\pi}_g(s) = +\infty$ if $g$ is unreachable by playing $\pi$ starting from $s$.
For any subset $\calX \subseteq \mathcal{S}$
and any goal state $g$, define $V^{\star}_{\calX, g}(s) = \min_{\pi \in \Pi(\calX)} V^{\pi}_g(s)$ as the minimum hitting time of $g$ following a policy restricted on $\calX$. Note that, if $\calX\subseteq\calX'$, then $\optV_{\calX',g}(s)\leq \optV_{\calX,g}(s)$ for any $s,g\in\calS$.
The objective of the learner is to efficiently navigate in the vicinity of $s_0$.
A state $s$ is \textit{$L$-controllable} if there exists a policy $\pi$ such that $V^{\pi}_s(s_0)\leq L$.
While discovering all $L$-controllable states may be a reasonable objective for exploring the vicinity of $s_0$~\citep{tarbouriech2022adaptive}, \citet{lim2012autonomous} showed that this may still require the learner to explore the whole state space, since reaching a $L$-controllable state may require navigating through non-$L$-controllable states.
To this end, \citet{lim2012autonomous} propose to only focus on navigating among \textit{incrementally $L$-controllable states}: states that are $L$-controllable by policies restricted on other incrementally controllable states.

\begin{definition}[Incrementally $L$-controllable states $\mathcal{S}_L^{\rightarrow}$]\label{def:sl.original}
Given a partial order $\prec$ on $\calS$, we define $\calS_L^{\prec}$ recursively as 1) $s_0\in\mathcal{S}_L^{\prec}$ and 2) if there exists a policy $\pi \in \Pi\big(\{ s' \in \mathcal{S}_L^{\prec}: s' \prec s \}\big)$ with $V^{\pi}_s(s_0) \leq L$, then $s \in \mathcal{S}_L^{\prec}$. The set $\acalS_L$ of incrementally $L$-controllable states is defined as $\acalS_L \triangleq \cup_{\prec} \mathcal{S}_L^{\prec}$, where the union is over all partial orders.
\end{definition}

Instead of exploring the potentially infinite state space, the objective of the learner is to discover the \emph{finite} set $\acalS_L$~\citep[][Prop. 6]{lim2012autonomous} and learn a corresponding set of policies that reliably reach each state in $\acalS_L$. We introduce three different formulations of the objective.

\begin{definition}[AX sample complexity]\label{def:ax}
	For any given length $L\geq 1$, error threshold $\epsilon>0$, and confidence level $\delta\in(0,1)$, the sample complexities $\calC(\frA, L,\epsilon,\delta)$, $\calC^{\star}(\frA, L,\epsilon,\delta)$, and $\calC^{+}(\frA, L,\epsilon,\delta)$ are defined as the number of steps required by a learning algorithm $\frA$ to identify a set of states $\calK$ and a set of policies $\{\pi_s\}_{s\in\calK}$ such that, with probability at least $1-\delta$, we have $\acalS_L\subseteq\calK$ and

	
	
	\quad ($\AX_L$)\quad $\forall s\in\acalS_L$, $V^{\pi_s}_s(s_0)\leq L(1+\epsilon)$,
	
	\quad ($\AX^{\star}$)\quad $\forall s\in\acalS_L$, $V^{\pi_s}_s(s_0)\leq \optV_{\acalS_L,s}(s_0) + L\epsilon$,
	
	\quad ($\AX^+$)\quad $\forall s\in\acalS_L$, $V^{\pi_s}_s(s_0)\leq \optV_{\acalS_L,s}(s_0)(1 + \epsilon)$.
\end{definition}
Note that the three formulations above are increasingly more demanding.
$\AX_L$ only requires to reach each state in $\acalS_L$ within $L(1+\epsilon)$ steps, which could correspond to a quite poor performance for a state $s$ with $\optV_{\acalS_L,s}(s_0) \ll L$.
$\AX^{\star}$ requires to learn a near-optimal policy for reaching each state in $\acalS_L$.
However, the allowed error threshold (i.e., $L\epsilon$) is uniform across all goal states, which again could correspond to a bad performance for a state $s$ with $\optV_{\acalS_L,s}(s_0) \ll L$. $\AX^+$ solves this issue by requiring a \emph{multiplicative} threshold. This implies that  the allowed error for reaching state $s$ (i.e., $\optV_{\acalS_L,s}(s_0)\epsilon$) scales with the optimal value $\optV_{\acalS_L,s}(s_0)$ itself, hence making this formulation adaptive to the hardness of reaching each goal state. 
No existing algorithm is able to achieve $\AX^+$ guarantees, see \pref{tab:summary}.

Note that these conditions cannot be checked at algorithmic time since $\acalS_L$ is unknown to the algorithm. Existing algorithms verify these conditions directly on the computed set $\calK$. 
Since they guarantee that $\acalS_L \subseteq \calK$, $ \optV_{\calK,g}(s_0) \leq  \optV_{\acalS_L,g}(s_0)$ for any $g \in \acalS_L$ and thus they satisfy the performance in~\pref{def:ax}.

\textbf{Other notation} Let $S=|\calS|$ and $A=|\calA|$.
For any $L\geq 1$, define $\aS_L=|\acalS_L|$, $\calN^{s, a}_L=\{s'\in\acalS_L: P_{s, a}(s')>0\}$, $\Gamma^{s, a}_L=|\calN^{s, a}_L|$ and $\Gamma_L=\max_{s\in\acalS_L, a}\Gamma^{s, a}_L$.
For simplicity, we often write $a=\bigo{b}$ as $a\lesssim b$.
For $n\in\fN_+$, define $[n]=\{1,\ldots,n\}$.

\subsection{A Constructive Definition of $\acalS_L$}\label{sec:sL.constructive}
While~\citet[][Proposition 6]{lim2012autonomous} showed that there exists a partial order $\prec$ such that $\acalS_L = \calS_L^{\prec}$, no explicit characterization of such partial order is provided. In the following, we develop an alternative definition of $\acalS_L$ that leads to an explicit constructive procedure to build the set.
This alternative definition is the main inspiration for the design of our algorithms.

We introduce an operator $\calT_L$ which, given a set $\calX\subseteq \calS$, selects all the states that are reachable in $L$ steps by a policy restricted on $\calX$ and show its connection with $\acalS_L$.
\begin{lemma}\label{lem:SL.operator}
	Let $\mathsf{P}(\mathcal{S})$ be the set of all subsets of $\mathcal{S}$.
	For any $L\geq 1$, define the operator $\mathcal{T}_L : \mathsf{P}(\mathcal{S}) \rightarrow \mathsf{P}(\mathcal{S})$ as follows: for any $\calX \subseteq \calS$, $\mathcal{T}_L(\calX) = \{ s \in \calS : V_{\calX,s}^\star(s_0) \leq L \}$. Then, 
	\begin{enumerate}
		\item $\acalS_L$ is the fixed-point of $\calT_L$ of smallest cardinality, i.e., $\acalS_L\subseteq\calX$ if $\calX=\calT_L(\calX)$.
	\end{enumerate}
	Let us denote by $\{\calKstar_j\}_{j\in\mathbb{N}}$ the unique sequence such that $\calKstar_1=\{s_0\}$, $\calKstar_j = \calT_L(\calKstar_{j-1})$. Then,
	\begin{enumerate}
            \setcounter{enumi}{1}
		\item For any $j\geq 1$, $\calKstar_j\subseteq\calKstar_{j+1} \subseteq \acalS_L$;
		\item There exists $J\leq \aS_L$ such that $\calKstar_j=\acalS_L$ for all $j\geq J$ (i.e., $\calT^{J}_L(\calKstar_{1}) = \lim_{j\rightarrow\infty} \calT^j_L(\calKstar_{1}) = \acalS_L$).
	\end{enumerate}
\end{lemma}
\begin{proof}
    Note that there exists a partial ordering $\prec^\star$ such that $\acalS_L=\calS^{\prec^\star}_L$~\citep[][Proposition 6]{lim2012autonomous}. 

    Let $\calX$ be s.t.\ $\acalS_L \not\subseteq \calX$. If $\acalS_L \cap \calX = \emptyset$, then $s_0 \notin \calX$, which implies that $\calT_L(\calX) = \{s_0\}$ since $V^\star_{\calX,s_0}(s_0) = 0 \leq L$ and $V^\star_{\calX,g}(s_0) = \infty$ for all $g\neq s_0$. Thus, $\calX$ cannot be a fixed point of $\calT_L$. Then, assume that $\acalS_L \cap \calX \neq \emptyset$. Order the states in $\calX \cap \acalS_L$ according to the ordering $\prec^\star$. Let $s_i \in \calS_L^{\prec^\star}$ be the first state s.t.\ $s\notin\calX$ (it exists since $\acalS_L \not\subseteq \calX$). By definition of $\prec^\star$ and $\acalS_L$, $V_{\{s_0,\ldots, s_{i-1}\},s_i}^\star(s_0) \leq L$, which implies that $s_i \in \calT_L(\calX)$. As a consequence, $\calX \neq \calT_L(\calX)$. Thus, if $\calX = \calT_L(\calX)$, we must have $\acalS_L \subseteq \calX$. This proves the first point.

	Let us prove that $\calKstar_j\subseteq\calKstar_{j+1}$ for all $j\geq 1$. Clearly, $\calKstar_2 = \calT_L(\calKstar_1) = \{s\in\calS : V_{\{s_0\},s}(s_0)\leq L\} \supseteq \{s_0\} = \calKstar_1$. Then, suppose that $\calKstar_{j-1}\subseteq\calKstar_{j}$ for some $j \geq 2$. By definition, for all $ s \in \calKstar_j$, $\optV_{\calKstar_{j-1},s}(s_0) \leq L$, which implies that $\optV_{\calKstar_{j},s}(s_0) \leq L$ by the inductive hypothesis. Then, $\calKstar_{j+1} = \calT_L(\calKstar_j) = \{s\in\calS : V_{\calKstar_j,s}(s_0)\leq L\} \supseteq \calKstar_j$.

	Now let us prove that $\calKstar_j \subseteq \acalS_L$ for all $j\geq 1$. Clearly, $ \calKstar_1 \subseteq \acalS_L$. Suppose that $\calKstar_j \subseteq \acalS_L$ for some $j\geq 1$. Then, if $s\in\calKstar_{j+1}$ for some $s \notin \acalS_L$, it must be that $V_{\calKstar_j,s}(s_0)\leq L$. By the inductive hypothesis, this implies that we found an ordering of the states in which $s$ is reachable in $L$ steps by a policy restricted on states of $\acalS_L$. Hence, $s\in\acalS_L$, which is a contradiction. This proves point 2.

	Let us enumerate over $\acalS_L=\{s_0,\ldots,s_{\aS_L-1}\}$ in a way that obeys $\prec^\star$. We prove by induction that $s_{j}\in\calKstar_{j+1}$ for any $0 \leq j < \aS_L$. Given point 2, this implies point $3$.
    Clearly, $s_0\in\calKstar_1$.
    Now suppose that $\{s_0,\ldots,s_j\}\in\calKstar_{j+1}$ for $0 \leq j \leq \aS_L - 2$.
    Then, we clearly have $s_{j+1}\in\calKstar_{j+2}$ by the definition of $\calKstar_{j+2}$ and the fact that $s_{j+1}$ is $L$-controllable by a policy restricted on $\{s_0,\ldots,s_j\}$.
\end{proof}
This lemma shows that $\acalS_L$ is a fixed-point solution of $\calT_L$. Most importantly, it provides an iterative procedure to construct $\acalS_L$. 
Starting from $\{s_0\}$ or $\emptyset$, $\calT_L$ acts as an expansive operator over sets (i.e., $T^{j}(\{s_0\}) \subset T^{j+1}(\{s_0\})$) until the set $\acalS_L$ is built. From this point, $\calT_L$ acts as an identity map since $\acalS_L$ is a fixed point. In other words, this procedure builds $\acalS_L$ iteratevely starting from $\calKstar_1$, expanding it to $\calKstar_2 = \calT_L(\calKstar_1)$, and so on until reaching $\acalS_L$. For this reason, we shall refer to the sets $(\calKstar_j)_j$ as \emph{layers}.
This process is learnable since it evolves only through subsets of $\acalS_L$ and it is at the core of the design of our algorithm. 

It is worth noticing that not all the fixed-point solutions of $\calT_L$ are learnable. In fact, Proposition 4 of~\citet{lim2012autonomous} implies that there exist MDPs with fixed points $\calX = \calT_L(\calX) \neq \acalS_L$ which may require an exponential number of samples to be learned. For example, there exist MDPs where the whole set of states $\calS$ is itself a fixed point of $\calT_L$ (that is, all states are $L$-controllable) but $\calS$ is exponentially larger than $\acalS_L$. This reveals an interesting connection between the existence of a \emph{unique} iterative process to reach the fixed-point corresponding to $\acalS_L$ and its learnability.

\section{$\AX_L$ through Layer Discovery}

\begin{algorithm2e*}[t]
    \caption{Layer-Aware State Discovery (\algo)}
    \label{alg:LOGSSD}
    \DontPrintSemicolon
    \LinesNumbered
    \small
    \KwIn{$L\geq1$, $\epsilon\in(0, 1]$, $\delta\in(0, 1)$.}
    Let $\frakN=\{2^j\}_{j\geq 0}$, $\calK\leftarrow \varnothing, \calU \leftarrow \varnothing$, $\calK'\leftarrow \{s_0\}, \Pi_{\calK} = \{\tilpi_{s_0}\text{ a random policy}\}$, $\N(\cdot, \cdot)\leftarrow 0, \N(\cdot,\cdot,\cdot) \leftarrow 0$.\;
    \For{round $r=1,\ldots$}{\label{line:round.easy}
        $\epsilon_{\VI}\leftarrow 1/\max\{16, \sum_{s,a}\N(s,a)\}$.\;
        \tcc{Policy optimisation and goal selection}
        Let $\gstar=\argmin_{g\in\calU}\big\{V_{\calK,g}(s_0)\big\}$ where $(Q_{\calK,g}, V_{\calK,g}, \pi_g)=\VISGO(\calK, g, \epsilon_{\VI}, \N, \frac{\delta}{4r^2S^2})$ (see \pref{alg:VISGO}).\label{line:compute V.easy}\;
        \eIf{$\gstar$ does not exist or $V_{\calK,\gstar}(s_0)>L$}{\label{line:goal condition.easy}
            \tcc{Expand or Terminate}
            \lIf{$\calK'=\varnothing$}{\textbf{return} $\calK$ and $\Pi_{\calK}$.}\label{line:terminate.easy}
            Set $\calK\leftarrow\calK\cup\calK'$, $\calK'=\varnothing, \calU=\varnothing$. \label{line:update.K.easy}\;
            $(\_,\calU)\leftarrow\fillc(\calK,\Pi_{\calK}, 0, 2L\log(4SALr^2/\delta))$ (see~\pref{alg:fillc}). \label{line:compute calU'.easy}\;
            Set $\nmin \leftarrow N_0(\calK,\frac{\delta}{4r^2S^2}) \lesssim L^2|\calK|\ln(Sr/\delta)$ (defined in~\pref{lem:bounded error}).\;
            $(\N,\_) \leftarrow \fillc(\calK,\Pi_{\calK},\N,\nmin)$. \label{line:fill N.easy}
        }{
            \tcc{Policy evaluation}
            Let $\hattau\leftarrow 0$, $\lambda\leftarrow N_{\dev}(32L, \frac{\epsilon}{256}, \frac{\delta}{4 r^2}) \lesssim \frac{1}{\epsilon^2} \ln^4\Big(\frac{Lr}{\epsilon \delta}\Big)$ (defined in \pref{lem:V pi mean}).\label{line:PE.easy}\;
            \For{$j=1,\ldots,\lambda$}{\label{line:episode.easy}
                $k\overset{+}{\leftarrow}1$, $i\leftarrow 1$, and reset to $s^k_1\leftarrow s_0$ by taking action $\reset$.\;
                \While{$s^k_i\neq \gstar$}{
                    Take $a^k_i=\pi_{\gstar}(s^k_i)$, and transits to $s^k_{i+1}$.
                    Increase $\N(s^k_i, a^k_i)$, $\N(s^k_i, a^k_i, s^k_{i+1})$, and $i$ by $1$.\;
                    \lIf{ $\sum_{s,a}\N(s,a)\in\frakN$ or ($s^k_i\in\calK$ and $\N(s^k_i, a^k_i)\in\frakN$)}{ return to \pref{line:round.easy} (skip round).\label{line:skip.easy}}
                    Set $\hattau\overset{+}{\leftarrow} \frac{c(s^k_i, a^k_i)}{\lambda}$.
                }
                \lIf{$\hattau>V_{\calK,\gstar}(s_0) + \epsilon L/2$}{ return to \pref{line:round.easy} (failure round).  \label{line:failure.easy}}
            }
            $\calK'\leftarrow\calK'\cup\{\gstar\}$, $\calU\leftarrow\calU\setminus\{\gstar\}$, $\Pi_{\calK}=\Pi_{\calK} \cup \{\tilpi_{\gstar}:=\pi_{\gstar}\}$ 
            (success round).
        }
    }
\end{algorithm2e*}

\pref{alg:LOGSSD} illustrates Layer-Aware State Discovery (\algo), a novel algorithm for $\AX_L$ based on the iterative construction of $\acalS_L$ introduced in \pref{lem:SL.operator}.
In \pref{sec:pc}, we then introduce a policy consolidation procedure that achieves $\AX^+$ when combined with \algo, leading to the $\algoax$ algorithm.
\algo maintains a set $\calK$ of ``known'' states, i.e., states for which a policy $\tilpi_s \in \Pi(\calK)$ with $V^{\tilpi_s}_s(s_0) \leq L(1+\epsilon)$ has been learned. These policies are stored in $\Pi_{\calK}$. The set $\calK$ is updated only when the algorithm is confident enough to have identified a new layer. To this purpose, $\calK'$ is used as a buffer for the new layer, i.e., for states that have been found to be $L$-controllable by policies restricted on $\calK$ and that are waiting to be merged with $\calK$. Finally, any other state discovered over time (and potential candidate to be in $\acalS_L$) is stored in $\calU$.

At each round, \algo first uses the samples collected so far to compute an optimistic policy for each state in $\calU$ through \VISGO (\pref{alg:VISGO}), a slight variant of the state-of-the-art algorithm for exploration-exploitation in stochastic shortest paths~\citep{tarbouriech2021stochastic}, and it selects the state that is optimistically closer to $s_0$ as candidate goal $g^\star$. 

If the optimistic distance of $g^\star$ from $s_0$ is larger than $L$, then no additional state can be confidently added to the current layer $\calK'$ and a \emph{set expansion} round is triggered. \algo updates the set of known states by adding the new layer $\calK'$ ($\calK = \calK \cup \calK'$) and starts a discovery process where policies in $\Pi_{\calK}$ are used to reach all states in $\calK$, then it executes all possible actions in these states, and it adds newly observed states to $\calU$. Notice that the samples obtained during this process are not included in the policy improvement of \VISGO to avoid statistical dependencies.
The sequence of expansion rounds is designed to approximate the sequence $\{\calKstar_j\}_j$. With high probability, every update of $\calK$ is not smaller than the application of $\calT_L$, i.e., if, for some $j$, $\calKstar_j\subseteq \calK \not\supseteq \calKstar_{j+1}$ before an update (this holds for $\calKstar_1 = \{s_0\}$ at the first round), then $\calKstar_{j+1} = \calT_{L}(\calKstar_{j})\subseteq \calK$ after the update. Thus, $\calK'$ is the increment to $\calK$ to include the next layer. At the end of the expansion round \algo executes an additional exploration step to ensure that a minimum number of samples is available for each $(s,a) \in \calK \times \calA$ (see~\pref{line:fill N.easy}).

On the other hand, if the optimistic distance of $g^\star$ is smaller than $L$, \algo performs a \emph{policy evaluation} round by running $\pi_{\gstar}$ to estimate whether the current policy is indeed able to reach $g^\star$ in less than $L$ steps. If the number of visits to some state-action pair is doubled within the current round, then the current round is classified as a \emph{skip round}. If the test on the policy performance fails, then the current round is classified as a \emph{failure round}. In both cases, a new round is started.
Otherwise, the current round is classified as a success round and $\gstar$ is added to the new layer $\calK'$. The samples collected in policy evaluation rounds are stored and used in all estimation and planning steps of the algorithm.

\algo terminates whenever the candidate goal $g^\star$ has an optimistic distance larger than $L$ and the new layer is empty, indicating that previous policy evaluation rounds could not identify any good policy and, thus, all states in $\acalS_L$ have been identified with high probability. 

We prove that \algo achieves the following guarantee, the proof can be found in~\pref{app:sd and id}.
\begin{theorem}
    \label{thm:sd}
    Suppose $\calS$ is finite. For any $L \geq 1$, $\epsilon \in (0,1]$ and $\delta \in (0,1)$, with probability at least $1-\delta$, \algo (\pref{alg:LOGSSD}) outputs a set $\calK$ such that $\acalS_L\subseteq\calK\subseteq\acalS_{L(1+\epsilon)}$ 
    and $\Pi_{\calK}$ such that $V^{\pi_g}_g(s_0)\leq L(1+\epsilon)$ for any $\pi_g \in \Pi_{\calK}$, with sample complexity bounded by
    \[
        \resizebox{\columnwidth}{!}{%
        $
        \bigO{\frac{\aS_{L(1+\epsilon)}\Gamma_{L(1+\epsilon)}AL}{\epsilon^2}\iota  + \frac{{\aS_{L(1+\epsilon)}}^2AL}{\epsilon}\iota + L^3 {\aS_{L(1+\epsilon)}}^2A\iota}
        $}
    \]
    where $\iota =\ln^8\left( \frac{SAL}{\epsilon\delta} \right)$.
\end{theorem}
Compared to the lower bound (see~\pref{tab:summary}), \algo still suffers from an extra $\Gamma_{L(1+\epsilon)}$ dependence.
This is because in the analysis we use a Bernstein-like concentration inequality to control the deviation $(P- \P)V$, where $\P$ are the estimated transitions, for any value function $V$ restricted on $\calK$ (i.e., $V$ is constant on all states outside $\calK$). Unfortunately, we cannot leverage refined concentration inequalities since $\calK$ is random and can take an exponentially large amount of values throughout the execution of \algo.

However, by inspecting the proof of \cite{cai2022near}, we note that the construction of the lower bound leverages a certain separation condition defined as follows.
\begin{assumption}[identifiability of $\{\calKstar_j\}_j$]
    \label{assum:id}
    We say $\{\calKstar_j\}_j$ is $\epsilon$-identifiable, if for any $j \geq 2, g\notin\calKstar_j$, we have $\optV_{\calKstar_{j-1}, g}(s_0)>L(1+\epsilon)$.
\end{assumption}
This means that each layer $\calKstar_j$ can be identified exactly by an algorithm run with accuracy $\epsilon$
since states that do not belong to the immediate next layer are clearly separated, i.e., they are more than $L(1+\epsilon)$-steps away. This leads to following remark.
\begin{remark}
    \label{rem:id}
    \pref{assum:id} implies that $\acalS_L=\acalS_{L(1+\epsilon)}$.
\end{remark}
The fact that states $g \notin \calKstar_j$ are not reachable in $L(1+\epsilon)$ steps from $\calKstar_{j-1}$ allows \algo to uniquely identify the layers. Indeed, under~\pref{assum:id}, \algo behaves as the operator $\calT_L$ and, after each expansion, we have that $\calK = \calKstar_j$ for some $j \in [\acalS_L]$.
Thanks to this property, we can show that \algo is minimax optimal.\footnote{Minimax optimality holds for $\epsilon \leq \min\{1/\aS_L, 1/L\}$, which makes the first term in \pref{thm:sd id} dominant \citep{cai2022near}.}
\begin{theorem}
    \label{thm:sd id}
    Suppose that $\calS$ is finite. 
    For any $L \geq 1$, $\epsilon \in (0,1]$ and $\delta \in (0,1)$, if~\pref{assum:id} holds, with probability at least $1-\delta$, \algo (\pref{alg:LOGSSD}) outputs $\calK=\acalS_{L(1+\epsilon)} = \acalS_{L}$ and $\Pi_{\calK}$ such that $V^{\pi_g}_g(s_0)\leq L(1+\epsilon)$ for any $\pi_g \in \Pi_{\calK}$, with sample complexity bounded by
    \[
        \bigO{\frac{\aS_{L}AL}{\epsilon^2}\iota  + \frac{{\aS_{L}}^2AL}{\epsilon}\iota + L^3 {\aS_{L}}^2A\iota},
    \]
    where $\iota =\ln^8\left( \frac{SAL}{\epsilon\delta} \right)$.
\end{theorem}
The trick to remove the $\Gamma_{L(1+\epsilon)}$ from \pref{thm:sd} is that, since layers are uniquely identified by the algorithm, we only need to concentrate the term $(P- \P)V$ for any value function in the set $\{V^\star_{\calKstar_j}\}_{j \in [\aS_{L}]}$.

\subsection{Proof Sketch}
Here we report a sketch of the proof, while the detailed one can be found in~\pref{app:logs}. All the statements we report here are to be considered to hold with high probability.

The first step of the proof (see~\pref{lem:calK.easy}) is to show by induction that, at each round, $\calK \subseteq \acalS_{L(1+\epsilon)}$. Thanks to the fact that $\tilo{L^2|\calK|}$ samples are always available for each $(s,a) \in \calK \times \calA$ (\pref{line:fill N.easy}) and the properties of \VISGO, it is possible to show that, for the goal $\gstar$ selected at the current round, $\|V^{\pi_{\gstar}}_{\gstar}\| \leq 2\|V^{\pi_{\gstar}}_{\calK,\gstar}\|\leq 4L$ if \pref{line:goal condition.easy} is passed. 
Combining this with the properties of policy evaluation and the inductive hypothesis, we have that $\hattau\geq L(1+\epsilon/2) \geq V^{\pi_{\gstar}}_{\calK,\gstar}(s_0) - L\epsilon/2$ if $\gstar \in \calU \setminus \acalS_{L(1+\epsilon)}$. Thus a failure test is triggered and $\gstar$ is never added to $\calK$. This shows that states outside $\acalS_{L(1+\epsilon)}$ are not added to $\calK$. 
By the same reasoning, we can show that if a goal $\gstar$ is added to $\calK'$, the corresponding policy has bounded value function (important prerequisite for policy consolidation) and satisfies $\AX_L$.
Furthermore, by properly selecting the number of rollouts in the expansion phase (\pref{line:compute calU'.easy}), we can show that $\calU$ always contains at least those states that are reachable in $L$ steps from $\calK$ (see~\pref{lem:calU.easy}), i.e., $\calT_L(\calK) \setminus \calK \subseteq \calU$.

Combining these results with optimism restricted on $\calKstar_j$ (see~\pref{lem:V calK.easy}), we are able to show (see~\pref{lem:update calK.easy}) that $\calK$ always expands by at least one layer at each update. Formally, if $\calKstar_j \subseteq \calK$ at a certain update, then $\calK \cup \calK' \supseteq \calKstar_{j+1}$ at the next update in~\pref{line:update.K.easy} (i.e., $\calKstar_{j+1} = \calT_L(\calKstar_j) \subseteq \calK$), see~\pref{lem:calK}. 
If \pref{assum:id} holds, thanks to the identifiability of the layers, we show that $\calK = \calT_L(\calKstar_{j}) = \calKstar_{j+1}$, i.e., the algorithm replicates the $\calT_L$ operator (see~\pref{lem:calK id}). In this case, $\calK'$ is exactly the set of states needed to move from $\calKstar_j$ to $\calKstar_{j+1}$. By induction, we conclude that $\acalS_L \subseteq \calK$ when the algorithm stops, $\calK = \acalS_L$ with~\pref{assum:id}.

These results provide $\AX_L$ guarantees when the algorithm stops. For computing the sample complexity we use a reduction to a regret analysis of a stochastic shortest path problem (SSP). We define the SSP regret as $R=\sumk(I_k - V_k(s_0))$ where $K$ is the total number of episodes done in policy evaluation, $I_k$ is the length of episode $k$, and $V_k$ is the optimistic value function of the goal selected at episode $k$.
Then, $C_K = \sumk I_k$ is the sample complexity of policy evaluation. Through the SSP regret analysis we can show that $R \lesssim c_1\sqrt{K} +c_2$ and $C_K \lesssim LK$, where $c_1 = L\sqrt{\Gamma_{L(1+\epsilon)}\aS_{L(1+\epsilon)}A}$ (resp. $c_1 = L\sqrt{\aS_{L(1+\epsilon)}A}$ under~\pref{assum:id}) and $c_2 = L {\aS_{L(1+\epsilon)}}^2A$, see~\pref{lem:regret.easy} and~\pref{lem:regret-improved.easy}. To conclude the analysis of the sample complexity we need to bound $K$. 
We note that $K = r_{\tot}\lambda\lesssim r_{\tot}/\epsilon^2$ where $r_{\tot}$ is the total number of rounds and $\lambda$ is the maximum number of episodes per round.
Moreover, $r_{\tot} \lesssim \frac{c_1^2}{L^2} + \frac{c_2\epsilon}{L}$ can be controlled since the regret is sublinear (see~\pref{lem:bound r.easy}).

In the expansion phases we execute policies that reach any state $s \in \calK$ almost surely since, as mentioned above, $\|V^{\pi_{s}}_{s}\| \leq 4L$. By~\citep[][Lemma 6]{rosenberg2020adversarial} we can bound the number of steps required to reach the goal by $8L$. Then, considering the number of samples that needs to be collected and that there are $\bigo{\aS_{L(1+\epsilon)}}$ of such phases, the total sample complexity of the expansion phases is $\tilo{L^3 {\aS_{L(1+\epsilon)}}^2A}$. Summing everything together concludes the proof (see~\pref{thm:sd.easy}).

\section{Improved Algorithms}
In this section, we present two improvements to \algo that allow to \emph{i)} replace the $\ln(S)$ dependence with a much milder $\ln(\acalS_{L(1+\epsilon)})$; \emph{ii)} move from $\AX_L$ to $\AX^+$.

\subsection{Log-Adaptivity to $\acalS_{L(1+\epsilon)}$}\label{sec:log.adaptivity}
Inspired by intrinsically motivated learning agents, \citet{lim2012autonomous} originally focused on a learning scenario where the environment is possibly infinite or at least no prior knowledge about it is available. Unfortunately, all the existing algorithms fail in dealing with this scenario since they require prior  knowledge of the cardinality of the  state space $\calS$. While the sample complexity only depends logarithmically on $S$, this shows that inability of the algorithms to exclusively focus on the portion of environment discovered and consolidated over time and it thus prevents from dealing with arbitrarily large or infinite environments. 

In this section, we carefully identify all the aspects of the algorithm causing  this problem in \algo, and propose an improved algorithm \algop (\pref{alg:SD} in \pref{app:logsa}) that replaces the $\ln(S)$ dependency by $\ln(\aS_{L(1+\epsilon)})$. This is a much favorable dependency since $\aS_{L(1+\epsilon)}$ is finite even when $\calS$ is countably infinite~\citep[][Prop. 6]{lim2012autonomous}.
Below we list each source of $\ln(S)$ dependency and the corresponding modification to fix it.

\paragraph{A) Limiting the set of candidate goals.} In the expansion phase, \algo uses all the newly discovered states to build the set $\calU$ of candidates states for $\acalS_L$. This phase could potentially discover any state $s\in\calS$ as long as the transition probability to $s$ from $\calK$ is non-zero. This means that any $s \in \calS$ can be considered in the goal selection step (\pref{line:compute V.easy}), requiring a union bound over $\calS$ when analyzing the concentration of the estimated value functions. To overcome this issue, \algop performs a step of state filtering in the construction of $\calU$ (\pref{alg:SD}-\pref{line:filter calU'.improved}).\footnote{A similar filter is used in \disco to reduce computational complexity, but as it does not use fresh samples, it still requires a union bound over $\calS$ to deal with statistical dependencies.} The idea is to include in $\calU$ only goal states with estimated hitting time upper bounded by $L$. To break statistical dependencies we  estimate the hitting time of each candidate goal state using fresh samples (i.e., samples that are discarded after this step). It can be showed (see \pref{lem:bcalU}) that using this filtering scheme, $\calU$ only includes states that are $\bigo{L}$-controllable by policies restricted on $\calK$, which is a much smaller candidate set of order $\aS_{L(1+\epsilon)}$.

\paragraph{B) Scaling the confidence bounds.} 
While the state filtering step allows to consider only states in $\acalS_{L(1+\epsilon)}$ rather than $\calS$, the knowledge of $\aS_{L(1+\epsilon)}$ is required to properly set the confidence level when computing the estimated value functions (\pref{alg:SD}-\pref{line:compute V.improved}). We thus maintain an estimate $z$ of $\aS_{L(1+\epsilon)}$. Each attempt on a specific value of $z$ is a trial indexed by $\tau$ (\pref{alg:SD}-\pref{line:size.improved}) that ends when the total number of ``known'' states ($|\calK \cup \calK'|$) exceeds the estimated dimension $z$ (\pref{alg:SD}-\pref{line:z.improved}). In this case, we double the value of $z$. We can show (see~\pref{lem:bound z}) that the total number of trials is bounded $\tau \lesssim \log_2(\aS_{L(1+\epsilon)})$ and $z\lesssim \aS_{L(1+\epsilon)}$.

\paragraph{C) Controlling the policy quality.}
An important step in \algo is to gather a minimum number of samples for each ``known'' state (\pref{line:fill N.easy}) to ensure a reasonable performance of the policy being evaluated. The right number of samples also depends on $\aS_{L(1+\epsilon)}$. Unfortunately, we cannot leverage $z$ to compute this threshold since $z$ is likely to be smaller than $\aS_{L(1+\epsilon)}$ throughout the execution of the algorithm.
Using $z$ will invalidate the properties of policy evaluation that may lead to halt prematurely, without satisfying the $\AX$ properties (e.g., $\acalS_L \subseteq \calK$). This failure mode is not captured by the condition used in~\pref{alg:SD}-\pref{line:z.improved} to increase $z$.
We thus introduce a Monte-Carlo reachability test (\pref{alg:SD}-\pref{line:rtest.improved}) before policy evaluation. Intuitively, if the test fails \algop gathers new samples to improve the estimate of the MDP, otherwise the test guarantees that $\|V^{\pi_{\gstar}}_{\gstar}\|_{\infty} \lesssim L$ (see~\pref{lem:rtest}).

Combining these three changes, we are able to obtain the following sample complexity guarantee (see~\pref{app:sd id improved}), which is $S$-independent.
\begin{theorem} 
    \label{thm:sd id improved}
    For any $L \geq 1$, $\epsilon \in (0,1]$ and $\delta \in (0,1)$, with probability at least $1-\delta$, \algop (\pref{alg:SD}) outputs $\acalS_L\subseteq\calK\subseteq\acalS_{L(1+\epsilon)}$ and $\Pi_{\calK}$ such that $V^{\pi_g}_g(s_0)\leq L(1+\epsilon)$ for any $\pi_g \in \Pi_{\calK}$, with sample complexity bounded by
    \[
        \bigO{\frac{LMA\iota}{\epsilon^2}  +\frac{L\aS_{L(1+\epsilon)}A\iota}{\epsilon}+ L^3{\aS_{L(1+\epsilon)}}^3A\iota},
    \]
    where $\iota =\ln^{12}(\frac{\aS_{L(1+\epsilon)}AL}{\epsilon\delta})$ and $M = \Gamma_{L(1+\epsilon)}\aS_{L(1+\epsilon)}$. If~\pref{assum:id} holds, then $M = \aS_{L}$ and $\aS_{L(1+\epsilon)} = \aS_L$.
\end{theorem}

\begin{algorithm2e}[t]
    \caption{Policy Consolidation (\algopc)}
    \label{alg:PC}
    \DontPrintSemicolon
    \LinesNumbered
    \small

    \KwIn{$L\geq 1$, $\epsilon\in(0,1]$, $\delta\in(0, 1)$, target state space $\tset\subseteq\acalS_{L(1+\epsilon)}$, and initial policies $\Pi'=\{\pi'_g\}_{g\in\tset}$.}
    Set $k\leftarrow 1$, $\frakN=\{2^j\}_{j\geq 0}$, $\calL=\tset$, $\Pi^+_{\calK}= \{\tilpi_{s_0} \text{ a random policy}\}$, $\N(\cdot, \cdot), \N(\cdot,\cdot,\cdot) \leftarrow 0$.\;
    $(\N,\_)\leftarrow\fillc(\calK, \Pi', \N, N_1(|\tset|-1, \frac{\delta}{|\tset|}))$ (see \pref{alg:fillc}; $N_1 \lesssim L^2|\calK|\ln(\frac{|\calK|}{\delta})$ is defined in \pref{lem:bounded error fresh}).\label{line:nu}\;
    \For{$r=1,\ldots$}{\label{line:round mge}
            \lIf{$\calL=\varnothing$}{                
                \textbf{return} $\Pi^+_{\calK}$.
            }
            $\epsilon_{\VI}\leftarrow 1/\max\{16, \sum_{s,a}\N(s,a)\}$.\;            
            Pick $\gstar\in\calL$ arbitrarily and compute $(\hatQ, \hatV, \hatpi)=\VISGO(\tset\setminus\{g\}, g, \epsilon_{\VI}, \N, \frac{\delta}{|\tset|})$. 
            \label{line:pc.goal.selection}\; 
            Let $\lambda\leftarrow N_{\dev}(32L, \frac{\epsilon}{256}, \frac{\delta}{2r^2}) \lesssim \frac{1}{\epsilon^2} \ln^4\left(\frac{Lr}{\epsilon\delta}\right)$ (defined in \pref{lem:V pi mean}) and $\hattau\leftarrow 0$.\label{line:PE PC}\;
            \For{$j=1,\ldots,\lambda$}{\label{line:episode PC}
                $k\overset{+}{\leftarrow}1$, $i\leftarrow 1$, and reset to $s^k_1\leftarrow s_0$ by taking action $\reset$.\;                
                \While{$s^k_i\neq \gstar$}{
                    Take $a^k_i=\hatpi(s^k_i)$, and transits to $s^k_{i+1}$.\;
                    Increase $\N(s^k_i, a^k_i)$, $\N(s^k_i, a^k_i, s^k_{i+1})$, and $i$ by $1$.\;
                    \lIf{$\sum_{s,a}\N(s,a)\in\frakN$ or ($s^k_i\in\calK$ and $\N(s^k_i, a^k_i)\in\frakN$)}{ return to \pref{line:round mge} (skip round).}\label{line:skip PC} 
                    Set $\hattau\overset{+}{\leftarrow} \frac{c(s^k_i, a^k_i)}{\lambda}$.
                }

                \lIf{$\hattau>\hatV(s_0)(1+\epsilon/2)$}{ return to \pref{line:round mge} (failure round).} 
                
            }

            $\calL\leftarrow\calL\setminus\{\gstar\}$, $\Pi^+_{\calK} \leftarrow \Pi^+_{\calK} \cup \{\tilpi_{\gstar}= \hatpi\}$ (success round).

    }
\end{algorithm2e}
\subsection{Policy Consolidation}\label{sec:pc}
Both \algo and \algop discover a set $\calK$ such that $\acalS_L\subseteq\calK\subseteq\acalS_{L(1+\epsilon)}$ and a set of goal-conditioned policies satisfying $\AX_L$. We now introduce a procedure that, given a set $\calK\subseteq\acalS_{L(1+\epsilon)}$ and associated goal-reaching policies $\Pi_{\calK}$ with bounded value function, learns a set of goal-condition policies satisfying the $\AX^+$ condition.

\pc (\pref{alg:PC}) is an algorithm for Multi-Goal Exploration (MGE)~\citep[e.g.,][]{tarbouriech2022adaptive} over $\calK$.
In each round, \pc randomly selects an ``unknown'' goal state from $\calL$ and computes a policy to reach it (\pref{line:pc.goal.selection}).
It then evaluates the performance of this policy by $\tilo{\frac{1}{\epsilon^2}}$ rollouts, and based on the evaluation result, the current round is classified into success, skip, or failure round similar to that in \pref{alg:LOGSSD}.
While it shares a similar structure with \valae, the crucial difference is the condition of success round (\pref{line:round mge}), which has a form similar to $\AX^+$.
Thus, one can consider \pref{alg:PC} as an improved version of \valae.

Its simplicity and high sample efficiency, allow \pc to be integrated with any existing algorithm for $\AX_L$ or $\AX^\star$ at no cost. As showed in the following lemma, the sample complexity of policy consolidation matches the lower-bound for $\AX$, thus providing a ``minor'' contribution to the overall sample complexity.
Details are deferred to \pref{app:consolidation}.
\begin{theorem}
    \label{thm:PC}
    Given a target state space $\tset\subseteq \acalS_{L(1+\epsilon)}$ for some $\epsilon\in(0, 1)$ and a set of initial policies $\Pi'=\{\pi'_g\}_{g\in\calK}$ such that $\norm{V^{\pi'_g}_g}_{\infty}\lesssim L$, with probability at least $1-\delta$, 
    \textsc{PolicyConsolidation} (\pref{alg:PC}) outputs a set of policies $\{\tilpi_g\}_{g\in\tset}$ such that $V^{\tilpi_g}_g(s_0)\leq \optV_{\tset,g}(s_0)(1+\epsilon)$ for all $g\in\tset$, with sample complexity bounded by
    \[
         \tilO{\frac{L\aS_{L(1+\epsilon)}A\iota}{\epsilon^2} + \frac{L{\aS_{L(1+\epsilon)}}^2A\iota}{\epsilon} + L^3{\aS_{L(1+\epsilon)}}^2A\iota},
    \]
    where $\iota=\ln^{10}(\frac{\aS_{L(1+\epsilon)}AL}{\epsilon\delta})$.
\end{theorem}
To achieve this result we developed an improved regret-based analysis. Instead of bounding the total number of rounds as in \valae, we directly bound the total number of steps in all rounds, which takes varying length of trajectories in different rounds into consideration.
This enables \textsc{PolicyConsolidation} to achieve a better guarantee on the performance of the learned policies  compared to \valae, preserving the same sample complexity.

\subsection{$AX^+$ through Layer Discovery and Consolidation}
We combine all these improvement into Layered Autonomous Exploration (\algoax) whose pseudo code is reported in~\pref{alg:ax.plus}. Combining the previous results, we can state the following guarantee for $\AX^+$.
\begin{corollary}
    \label{cor:sd id}
    For any $L \geq 1$, $\epsilon \in (0,1]$ and $\delta \in (0,1)$, with probability at least $1-\delta$, \algoax (\pref{alg:ax.plus}) outputs $\acalS_L\subseteq\calK\subseteq\acalS_{L(1+\epsilon)}$ and $\Pi_{\calK}$ such that $V^{\pi_g}_g(s_0)\leq V_{\calK,g}^\star(s_0)(1+\epsilon)$, for any $\pi_g \in \Pi_{\calK}$, with sample complexity
    \[
        \bigO{\frac{LMA\iota}{\epsilon^2}  +\frac{L\aS_{L(1+\epsilon)}A\iota}{\epsilon}+ L^3{\aS_{L(1+\epsilon)}}^3A\iota}
    \]
    where $\iota=\ln^{12}\rbr{\frac{\aS_{L(1+\epsilon)}AL}{\epsilon\delta}}$ and $M = \Gamma_{L(1+\epsilon)}\aS_{L(1+\epsilon)}$. If~\pref{assum:id} holds, then $M = \aS_{L}$ and $\aS_{L(1+\epsilon)} = \aS_{L}$.
\end{corollary}
This shows that \algoax is the first algorithm able to i) achieve the strongest performance $\AX^+ \Rightarrow \AX^\star \Rightarrow \AX_L$, ii) match the lower-bound under certain settings, and iii) completely remove the dependence on $S$. In particular, the latter was an open problem since the initial work by~\citet{lim2012autonomous}.\footnote{\ucbexplore originally considered a countable, possibly infinite state space; however this leads to a
technical issue in the analysis~\citep[][Footnote 2]{tarbouriech2020improved}.}

\textbf{Comparisons.} 
\algo/\algop shares similarities with both \ucbexplore and \valae. While we leverage the same condition as in \valae for the failure test of policy evaluation, the policy evaluation in \valae is only for learning goal-conditioned policies and not for consolidating states. In fact, they first run \disco for state discovery, and then learn goal-conditioned policies on a potentially much larger set subsuming $\acalS_{2L}$. However, $\acalS_{2L}$ can be exponentially larger than $\aS_{L(1+\epsilon)}$ (see \pref{lem:example 2L}) in general and thus the sample complexity of \valae is incomparable to other algorithms.
Therefore, \valae only improves the sample complexity of policy learning but not that of state discovery.
Similarly to \ucbexplore, we perform state and policy identification simultaneously. Our evaluation phase is much more sample efficient compared to \ucbexplore, which saves a $L^2/\epsilon$ factor in the leading-order term.
Compared to \disco, our algorithm saves a $L^2$ factor by i) adaptively collecting samples to estimate state values instead of prescribing a fixed number of samples to guarantee a uniformly-accurate transition estimate over $\calK$, 
and ii) leveraging variance information.

The tool enabling all these improvements is a new Bernstein-type concentration inequality for restricted value functions (see \pref{lem:dPV}). The key difficulty in our analysis is that the set on which value functions are restricted is random since we learn $\calK$ and $\Pi_{\calK}$ simultaneously. In comparison, in \valae the set $\calK$ is fixed after the initial phase of state discovery, which makes the analysis much simpler.
Specifically, leveraging the fact that the learned goal-conditioned policies are all restricted on $\acalS_{L(1+\epsilon)}$, we are able to make use of the variance information without incurring a polynomial dependency on $S$.

\begin{algorithm2e}[t]
    \caption{Layered Autonomous Exploration (\algoax)}
    \label{alg:ax.plus}
    \DontPrintSemicolon
    \LinesNumbered
    \small
    \KwIn{$L\geq 1$, $\epsilon\in(0,1]$, and $\delta\in(0, 1)$.}
    $(\calK, \Pi_{\calK}^L) = \algop\big(L, \epsilon, \delta \big)$ see \pref{alg:SD} in appendix (or \algo for $\log S$).\tcp*{$\AX_L$}
    $\Pi^+_{\calK} = \algopc\big(L, \epsilon, \delta, \calK, \Pi_{\calK}^L\big)$.\tcp*{$\AX^+$}
    \Return{$\calK$ and $\Pi^+_{\calK}$.}
\end{algorithm2e}

\section{Conclusion}
We introduced a layered decomposition of the set of incrementally $L$-controllable states. We built on this decomposition and showed that our algorithm \algoax attains the strongest performance guarantee $\AX^+$, does not need to know $S$ and thus can be used with a countably-infinite state space, and is minimax-optimal when the layers can be uniquely identified.
The natural future directions include 1) designing an algorithm with minimax sample complexity without~\pref{assum:id}; 2) extending the problem to continuous states and function approximation; 3) identifying benchmarks that can be used to evaluate practical progresses towards the $\AX$ capability.

\clearpage
\bibliography{ref.bib}

\begin{thebibliography}{32}
\providecommand{\natexlab}[1]{#1}
\providecommand{\url}[1]{\texttt{#1}}
\expandafter\ifx\csname urlstyle\endcsname\relax
  \providecommand{\doi}[1]{doi: #1}\else
  \providecommand{\doi}{doi: \begingroup \urlstyle{rm}\Url}\fi

\bibitem[Bagaria et~al.(2021)Bagaria, Senthil, and Konidaris]{bagaria2021skill}
Bagaria, A., Senthil, J.~K., and Konidaris, G.
\newblock Skill discovery for exploration and planning using deep skill graphs.
\newblock In \emph{International Conference on Machine Learning}, pp.\
  521--531. PMLR, 2021.

\bibitem[Bellemare et~al.(2016)Bellemare, Srinivasan, Ostrovski, Schaul,
  Saxton, and Munos]{bellemare2016unifying}
Bellemare, M., Srinivasan, S., Ostrovski, G., Schaul, T., Saxton, D., and
  Munos, R.
\newblock Unifying count-based exploration and intrinsic motivation.
\newblock \emph{Advances in neural information processing systems}, 29, 2016.

\bibitem[Bertsekas \& Yu(2013)Bertsekas and Yu]{bertsekas2013stochastic}
Bertsekas, D.~P. and Yu, H.
\newblock Stochastic shortest path problems under weak conditions.
\newblock \emph{Lab. for Information and Decision Systems Report LIDS-P-2909,
  MIT}, 2013.

\bibitem[Cai et~al.(2022)Cai, Ma, and Du]{cai2022near}
Cai, H., Ma, T., and Du, S.~S.
\newblock Near-optimal algorithms for autonomous exploration and multi-goal
  stochastic shortest path.
\newblock In \emph{{ICML}}, volume 162 of \emph{Proceedings of Machine Learning
  Research}, pp.\  2434--2456. {PMLR}, 2022.

\bibitem[Chen \& Luo(2021)Chen and Luo]{chen2021finding}
Chen, L. and Luo, H.
\newblock Finding the stochastic shortest path with low regret: The adversarial
  cost and unknown transition case.
\newblock In \emph{International Conference on Machine Learning}, 2021.

\bibitem[Chen \& Luo(2022)Chen and Luo]{chen2022near}
Chen, L. and Luo, H.
\newblock Near-optimal goal-oriented reinforcement learning in non-stationary
  environments.
\newblock \emph{arXiv preprint arXiv:2205.13044}, 2022.

\bibitem[Chen et~al.(2021)Chen, Jafarnia-Jahromi, Jain, and
  Luo]{chen2021implicit}
Chen, L., Jafarnia-Jahromi, M., Jain, R., and Luo, H.
\newblock Implicit finite-horizon approximation and efficient optimal
  algorithms for stochastic shortest path.
\newblock \emph{Advances in Neural Information Processing Systems}, 2021.

\bibitem[Chen et~al.(2022{\natexlab{a}})Chen, Jain, and Luo]{chen2021improved}
Chen, L., Jain, R., and Luo, H.
\newblock Improved no-regret algorithms for stochastic shortest path with
  linear {MDP}.
\newblock In \emph{{ICML}}, volume 162 of \emph{Proceedings of Machine Learning
  Research}, pp.\  3204--3245. {PMLR}, 2022{\natexlab{a}}.

\bibitem[Chen et~al.(2022{\natexlab{b}})Chen, Luo, and
  Rosenberg]{chen2022policy}
Chen, L., Luo, H., and Rosenberg, A.
\newblock Policy optimization for stochastic shortest path.
\newblock In \emph{{COLT}}, volume 178 of \emph{Proceedings of Machine Learning
  Research}, pp.\  982--1046. {PMLR}, 2022{\natexlab{b}}.

\bibitem[Chen et~al.(2023)Chen, Tirinzoni, Pirotta, and
  Lazaric]{chen2022reaching}
Chen, L., Tirinzoni, A., Pirotta, M., and Lazaric, A.
\newblock Reaching goals is hard: Settling the sample complexity of the
  stochastic shortest path.
\newblock In \emph{International Conference on Algorithmic Learning Theory},
  2023.

\bibitem[Cohen et~al.(2020)Cohen, Kaplan, Mansour, and
  Rosenberg]{cohen2020near}
Cohen, A., Kaplan, H., Mansour, Y., and Rosenberg, A.
\newblock Near-optimal regret bounds for stochastic shortest path.
\newblock In \emph{Proceedings of the 37th International Conference on Machine
  Learning}, volume 119, pp.\  8210--8219. PMLR, 2020.

\bibitem[Colas et~al.(2020)Colas, Karch, Sigaud, and
  Oudeyer]{colas2020intrinsically}
Colas, C., Karch, T., Sigaud, O., and Oudeyer, P.
\newblock Intrinsically motivated goal-conditioned reinforcement learning: a
  short survey.
\newblock \emph{CoRR}, abs/2012.09830, 2020.

\bibitem[Eysenbach et~al.(2019)Eysenbach, Gupta, Ibarz, and
  Levine]{eysenbach2018diversity}
Eysenbach, B., Gupta, A., Ibarz, J., and Levine, S.
\newblock Diversity is all you need: Learning skills without a reward function.
\newblock In \emph{The International Conference on Learning Representations},
  2019.

\bibitem[Gregor et~al.(2016)Gregor, Rezende, and
  Wierstra]{gregor2016variational}
Gregor, K., Rezende, D.~J., and Wierstra, D.
\newblock Variational intrinsic control.
\newblock \emph{arXiv preprint arXiv:1611.07507}, 2016.

\bibitem[Hazan et~al.(2019)Hazan, Kakade, Singh, and
  Van~Soest]{hazan2019provably}
Hazan, E., Kakade, S., Singh, K., and Van~Soest, A.
\newblock Provably efficient maximum entropy exploration.
\newblock In \emph{International Conference on Machine Learning}, pp.\
  2681--2691, 2019.

\bibitem[Jin et~al.(2020)Jin, Krishnamurthy, Simchowitz, and Yu]{jin2020reward}
Jin, C., Krishnamurthy, A., Simchowitz, M., and Yu, T.
\newblock Reward-free exploration for reinforcement learning.
\newblock In \emph{International Conference on Machine Learning}, pp.\
  4870--4879. PMLR, 2020.

\bibitem[Kamienny et~al.(2022)Kamienny, Tarbouriech, Lamprier, Lazaric, and
  Denoyer]{kamienny2021direct}
Kamienny, P., Tarbouriech, J., Lamprier, S., Lazaric, A., and Denoyer, L.
\newblock Direct then diffuse: Incremental unsupervised skill discovery for
  state covering and goal reaching.
\newblock In \emph{{ICLR}}. OpenReview.net, 2022.

\bibitem[Kaufmann et~al.(2021)Kaufmann, M{\'e}nard, Domingues, Jonsson,
  Leurent, and Valko]{kaufmann2021adaptive}
Kaufmann, E., M{\'e}nard, P., Domingues, O.~D., Jonsson, A., Leurent, E., and
  Valko, M.
\newblock Adaptive reward-free exploration.
\newblock In \emph{Algorithmic Learning Theory}, pp.\  865--891. PMLR, 2021.

\bibitem[Lim \& Auer(2012)Lim and Auer]{lim2012autonomous}
Lim, S.~H. and Auer, P.
\newblock Autonomous exploration for navigating in {MDP}s.
\newblock In \emph{Conference on Learning Theory}, pp.\  40--1. JMLR Workshop
  and Conference Proceedings, 2012.

\bibitem[M{\'e}nard et~al.(2021)M{\'e}nard, Domingues, Jonsson, Kaufmann,
  Leurent, and Valko]{menard2021fast}
M{\'e}nard, P., Domingues, O.~D., Jonsson, A., Kaufmann, E., Leurent, E., and
  Valko, M.
\newblock Fast active learning for pure exploration in reinforcement learning.
\newblock In \emph{International Conference on Machine Learning}, pp.\
  7599--7608. PMLR, 2021.

\bibitem[Oudeyer et~al.(2009)Oudeyer, Baranes, and
  Kaplan]{oudeyer2009intrinsically}
Oudeyer, P.-Y., Baranes, A., and Kaplan, F.
\newblock \emph{Intrinsically Motivated Exploration for Developmental and
  Active Sensorimotor Learning}, volume 264, pp.\  107--146.
\newblock 12 2009.
\newblock ISBN 978-3-642-05180-7.
\newblock \doi{10.1007/978-3-642-05181-4_6}.

\bibitem[Pong et~al.(2020)Pong, Dalal, Lin, Nair, Bahl, and
  Levine]{pong2019skew}
Pong, V., Dalal, M., Lin, S., Nair, A., Bahl, S., and Levine, S.
\newblock Skew-fit: State-covering self-supervised reinforcement learning.
\newblock In \emph{{ICML}}, volume 119 of \emph{Proceedings of Machine Learning
  Research}, pp.\  7783--7792. {PMLR}, 2020.

\bibitem[Rosenberg \& Mansour(2021)Rosenberg and
  Mansour]{rosenberg2020adversarial}
Rosenberg, A. and Mansour, Y.
\newblock Stochastic shortest path with adversarially changing costs.
\newblock In \emph{{IJCAI}}, pp.\  2936--2942. ijcai.org, 2021.

\bibitem[Schmidhuber(1991)]{schmidhuber1991possibility}
Schmidhuber, J.
\newblock A possibility for implementing curiosity and boredom in
  model-building neural controllers.
\newblock In Meyer, J.~A. and Wilson, S.~W. (eds.), \emph{Proc. of the
  International Conference on Simulation of Adaptive Behavior: From Animals to
  Animats}, pp.\  222--227. MIT Press/Bradford Books, 1991.

\bibitem[Singh et~al.(2004)Singh, Barto, and Chentanez]{singh2004intrinsically}
Singh, S., Barto, A.~G., and Chentanez, N.
\newblock Intrinsically motivated reinforcement learning.
\newblock In \emph{{NIPS}}, pp.\  1281--1288, 2004.

\bibitem[Tarbouriech et~al.(2020{\natexlab{a}})Tarbouriech, Garcelon, Valko,
  Pirotta, and Lazaric]{tarbouriech2020no}
Tarbouriech, J., Garcelon, E., Valko, M., Pirotta, M., and Lazaric, A.
\newblock No-regret exploration in goal-oriented reinforcement learning.
\newblock In \emph{International Conference on Machine Learning}, pp.\
  9428--9437. PMLR, 2020{\natexlab{a}}.

\bibitem[Tarbouriech et~al.(2020{\natexlab{b}})Tarbouriech, Pirotta, Valko, and
  Lazaric]{tarbouriech2020improved}
Tarbouriech, J., Pirotta, M., Valko, M., and Lazaric, A.
\newblock Improved sample complexity for incremental autonomous exploration in
  {MDP}s.
\newblock In \emph{Advances in Neural Information Processing Systems},
  volume~33, pp.\  11273--11284. Curran Associates, Inc., 2020{\natexlab{b}}.

\bibitem[Tarbouriech et~al.(2020{\natexlab{c}})Tarbouriech, Shekhar, Pirotta,
  Ghavamzadeh, and Lazaric]{tarbouriech2020active}
Tarbouriech, J., Shekhar, S., Pirotta, M., Ghavamzadeh, M., and Lazaric, A.
\newblock Active model estimation in markov decision processes.
\newblock In \emph{Conference on Uncertainty in Artificial Intelligence}, pp.\
  1019--1028. PMLR, 2020{\natexlab{c}}.

\bibitem[Tarbouriech et~al.(2021{\natexlab{a}})Tarbouriech, Pirotta, Valko, and
  Lazaric]{tarbouriech2021provably}
Tarbouriech, J., Pirotta, M., Valko, M., and Lazaric, A.
\newblock A provably efficient sample collection strategy for reinforcement
  learning.
\newblock In \emph{NeurIPS}, pp.\  7611--7624, 2021{\natexlab{a}}.

\bibitem[Tarbouriech et~al.(2021{\natexlab{b}})Tarbouriech, Zhou, Du, Pirotta,
  Valko, and Lazaric]{tarbouriech2021stochastic}
Tarbouriech, J., Zhou, R., Du, S.~S., Pirotta, M., Valko, M., and Lazaric, A.
\newblock Stochastic shortest path: Minimax, parameter-free and towards
  horizon-free regret.
\newblock In \emph{NeurIPS}, pp.\  6843--6855, 2021{\natexlab{b}}.

\bibitem[Tarbouriech et~al.(2022)Tarbouriech, Domingues, M{\'e}nard, Pirotta,
  Valko, and Lazaric]{tarbouriech2022adaptive}
Tarbouriech, J., Domingues, O.~D., M{\'e}nard, P., Pirotta, M., Valko, M., and
  Lazaric, A.
\newblock Adaptive multi-goal exploration.
\newblock In \emph{International Conference on Artificial Intelligence and
  Statistics}, pp.\  7349--7383. PMLR, 2022.

\bibitem[Zhang et~al.(2021)Zhang, Du, and Ji]{zhang2021near}
Zhang, Z., Du, S., and Ji, X.
\newblock Near optimal reward-free reinforcement learning.
\newblock In \emph{International Conference on Machine Learning}, pp.\
  12402--12412. PMLR, 2021.

\end{thebibliography}
\bibliographystyle{icml2023}

\newpage
\appendix
\onecolumn
\newpage

\section*{Contents}
\input{main.toc}


\section{Notation}\label{app:notation}

Let $(x)_+=\max\{0, x\}$ and $\Ind_s(s')=\Ind\{s'=s\}$.
We say that a value function $V$ is \textbf{restricted} on a subset $\calX \subseteq \calS$, if there exists $v>0$ such that $V(s)=v$ for any $s\notin\calX$.
When value function $V$ takes the same value within a subset of states $y$, we define $V(y)=V(s)$ for any $s\in y$.
For any subset $y\subseteq\calS$ and distribution $P\in\Delta_{\calS}$, define $P(y)=\sum_{s'\in y}P(s')$.

\paragraph{Trial} In \pref{alg:SD}, a trial is indexed by $\tau$, and each trial corresponds to a value of $z$ estimating $\aS_{L(1+\epsilon)}$ (\pref{line:trial}).
In \pref{alg:LOGSSD} and \pref{alg:PC}, we assume the whole learning procedure lies in an artificial trial.

\begin{table*}[h]
    \centering
    \caption{The notation adopted in this paper.}
    \label{tab:notation}
    \resizebox{\textwidth}{!}{
    \begin{tabular}{ll} 
    \toprule
    Symbol & Meaning \\
    \toprule
    $\calS$ & State Space\\
    $\calA$ & Action Space (including the \reset action)\\
    $P$ & Transition function\\
    $\pi : \calS \to \calA$ & A policy\\
    $\Pi(\calX)$ & Policies restricted to $\calX$, \reset{} is taken outside $\calX$\\
    $L$ & Exploration radius\\
    $\acalS_{L}$ & Incrementally $L$-controllable states\\
    $\calN^{s, a}_L=\{s'\in\acalS_L: P_{s, a}(s')>0\}$ & States in $\acalS_L$ reachable from $(s,a)$\\
    $\Gamma^{s, a}_L=|\calN^{s, a}_L|, \Gamma_L=\max_{s\in\acalS_L, a}\Gamma^{s, a}_L$ & Cardinality of $\calN^{s, a}_L$ and maximum value\\
    $\calT_L(\calX)=\{g \in \calS: \optV_{\calX,g}(s_0)\leq L\}$ & Set of $L$ controllable states restricted on $\calX\subseteq\calS$\\
    $\{\calKstar_j\}_j : \calKstar_1 = \{s_0\}, \calKstar_j =\calT_L(\calKstar_{j-1})$ & Layers defining $\acalS_L$\\
    $\acalO_L=(s_1,\ldots,s_n)$ & Ordering of states in $\acalS_L$ defining the layer $\{\calKstar_j\}$\\
    $\calKstar_{z,j}$ & $\calKstar_{z,j}=\calKstar_j$ when $|\calKstar_j|< z$, and $\calKstar_{z,j}=\{s_1,\ldots,s_z\}$ when $|\calKstar_j|\geq z$\\
    $\calKstar_{z,z}=(s_1,\ldots,s_{z})$ & The first $z$ elements of $\acalO_L$ or $\acalS_L$\\
    $\calUstar_z=\rS{\calKstar_{z,z}}{2L}$ & States reachable in $2L$ steps from $\calKstar_{z,z}$\\
    $\calN(\calX, p)=\{s'\notin\calX: P(s'|s,a)\geq p\text{ for some }(s, a)\in\calX\times\calA \}$ & States not in $\calX$ reachable with high probability from $\calX$\\
    $\bcalU=\{s'\in\calS: \exists s\in \acalS_{L(1+\epsilon)}, a\in\calA, P(s'|s, a)\geq \frac{1}{2L}\}$ & States that are reachable from $\acalS_{L(1+\epsilon)}$ with high probability\\
    \toprule
    \multicolumn{2}{c}{Learning Algorithm} \\
    \cmidrule{1-2}
    $r \in \mathbb{N}_+$ & Round\\
    $\tau \in \mathbb{N}_+$ & Trial\\
    $z$ & An estimate of $|\aS_{L(1+\epsilon)}|$.
    The value of $z$ is updated at the beginning of each trial.\\
    $\epsilon$ & accuracy\\
    $\calK$ & Set of ``known'' states, such that $\calKstar_j \subseteq \calK$ for some $j$\\
    $\calU$ & Set of ``unknown'' states\\
    $\calK'$ & Increment to $\calK$ leading to include layer $j+1$\\
    $\N(s,a,s')$ & Number of visits to $(s,a,s')$\\
    $\lambda$ & Number of episodes for policy evaluation\\
    $\hattau$ & Average number of steps to reach the goal by policy $\pi_{\gstar}$\\
    \bottomrule
    \end{tabular}
    }
\end{table*}
\newpage
\section{Analysis of VISGO}\label{app:visgo}

\begin{algorithm2e}[t]
    \DontPrintSemicolon
    \caption{\VISGO}
    \label{alg:VISGO}
    \KwIn{state subset $\calX$, goal state $g\notin\calX$, precision $\epsilon_{\VI}$, counter $n$, and failure probability $\delta$.}

    \textbf{Require:} $\norm{\optV_{\calX,g}}_{\infty}\leq 8L$.

    Let $c_1=3$, $c_2=512$, and $\iota_{s,a}=\ln\left(\frac{2|\calX|An(s,a)}{\delta}\right)$ for all $(s, a)$.
    
    Let $\P_{s,a}(s')=\frac{n(s,a,s')}{n^+(s,a)}$ and $\tilP_{s,a}(s')=\frac{n(s,a)}{n(s,a)+1}\P_{s,a}(s') + \frac{\Ind\{s'=g\}}{n(s,a)+1}$ for all $(s,a,s')$.
    
    \textbf{Initialize:} $V^{(0)}(\cdot)\leftarrow 0$, $i\leftarrow 0$.
    
    \While{$i=0$ or $\norm{V^{(i)}-V^{(i-1)}}_{\infty}>\epsilon_{\VI}$}{
        \nl \lIf{$\norm{V^{(i)}}_{\infty} > 2L$}{\textbf{return} $(\infty, \infty, \pi)$ with $\pi$ being a random policy.}\label{line:bound V}
    
        $i\leftarrow i + 1$.
    
        \For{$s\in\calX$}{
            $b^{(i)}(s,a)\leftarrow \max\cbr{c_1\sqrt{\frac{\fV(\P_{s, a}, V^{(i-1)})\iota_{s,a}}{n^+(s, a)}}, \frac{c_2L\iota_{s,a}}{n^+(s, a)} }$.
        
            $Q^{(i)}(s, a) \leftarrow \max\cbr{0, 1 + \tilP_{s, a}V^{(i-1)} - b^{(i)}(s,a) }$ for $a\in\calA$.

            $V^{(i)}(s) \leftarrow \min_aQ^{(i)}(s, a)$
        }

        $V^{(i)}(s)\leftarrow (1 + V^{(i-1)}(s_0))\Ind\{s\neq g\}$ for $s\notin\calX$.
    }

    \Return{$(Q^{(i)}, V^{(i)}, \pi)$ with $\pi(s)=\argmin_aQ^{(i)}(s, a)$ for $s\in\calX$ and $\pi_g(s)=\reset$ for $s\notin\calX$.}
    
\end{algorithm2e}

The convergence of \VISGO has been proved in \citep[Lemma C.4]{cai2022near}.
We further introduce some properties of the algorithm.

\begin{lemma}[Optimism]
    \label{lem:opt}

    Let $\calX\subseteq\calS$, $g\in\calS\setminus\calX$, $n$ be a counter incrementally collecting samples from transition function $P$, and $\delta\in(0,1)$ be such that $\|V_{\calX,g}^\star\|_{\infty} \leq 8L$. For any precision $\xi > 0$, define $(Q_{\xi}, V_{\xi},\_) = \VISGO(\calX,g,\xi,n,\delta)$ as the output of \pref{alg:VISGO}. Let $\mathbb{P}$ be the probability operator on the process generating the counter $n$ and assume that $\calX$ and $g$ are independent of $n$. Then,
    \begin{align*}
        \mathbb{P}\Big( \forall \xi > 0, s\in\calS,a\in\calA : Q_{\xi}(s, a)\leq\optQ_{\calX,g}(s, a), V_{\xi}(s)\leq\optV_{\calX,g}(s) \Big) \geq 1 - \delta.
    \end{align*}
\end{lemma}
\begin{proof}
    First, by \pref{lem:anytime bernstein} and a union bound over $(s,a)\in\calX\times\calA$, we have with probability at least $1-\delta$, for any $(s,a)\in\calX\times\calA$,
    \begin{align}
        \abr{(\P_{s,a}-P_{s,a})\optV_{\calX,g}} &\leq 2\sqrt{\frac{2\fV(\P_{s,a},\optV_{\calK,g})\ln\frac{2|\calX|An(s,a)}{\delta}}{n^+(s,a)}} + \frac{19\cdot 8L\ln\frac{2|\calX|An(s,a)}{\delta}}{n^+(s, a)}\notag\\
        &\leq \frac{c_1}{2}\sqrt{\frac{\fV(\P_{s, a}, \optV_{\calX,g})\iota_{s,a}}{n^+(s, a)}} + \frac{c_2L\iota_{s,a}}{2n^+(s, a)},\label{eq:optV}
    \end{align}
    with $\iota_{s,a}$, $c_1$, and $c_2$ are defined in \pref{alg:VISGO}. We then carry out the proof assuming that such event holds.

    Fix a configuration $(\calX,g,\xi,n,\delta)$ of the inputs of VISGO and let $(Q^{(i)}, V^{(i)})_{i\geq 0}$ be the iterates of the algorithm. It suffices to show that for any $i\geq 0$, $Q^{(i)}(s, a) \leq \optQ_{\calX,g}(s,a)$  for all $(s,a)\in\calX\times\calA$ and $V^{(i)}(s)\leq\optV_{\calX,g}(s)$ for all $s\in\calS$. We prove it by induction.

    Note that $Q^{(0)}(\cdot) = V^{(0)}(\cdot) = 0$, thus the statement clearly holds for the base case $i=0$. Suppose it holds at some iteration $i-1 \geq 0$. Under event of \pref{eq:optV}, for any $i>0$ and $(s,a)\in\calX\times\calA$,
    \begin{align*}
        &1 + \tilP_{s, a}V^{(i-1)} - \max\cbr{c_1\sqrt{\frac{\fV(\P_{s, a}, V^{(i-1)})\iota_{s,a}}{n^+(s, a)}}, \frac{c_2L\iota_{s,a}}{n^+(s, a)} }\\
        &\leq 1 + \tilP_{s, a}\optV_{\calX,g} - \max\cbr{c_1\sqrt{\frac{\fV(\P_{s, a}, \optV_{\calX,g})\iota_{s,a}}{n^+(s, a)}}, \frac{c_2L\iota_{s,a}}{n^+(s, a)} } \tag{induction step and \pref{lem:mvp}}\\
        &\leq 1 + \P_{s,a}\optV_{\calX,g}- \max\cbr{c_1\sqrt{\frac{\fV(\P_{s, a}, \optV_{\calX,g})\iota_{s,a}}{n^+(s, a)}}, \frac{c_2L\iota_{s,a}}{n^+(s, a)} } \tag{definition of $\tilP_{s,a}$}\\
        &\leq 1 + P_{s,a}\optV_{\calX,g} + (\P_{s,a}-P_{s,a})\optV_{\calX,g} - \frac{c_1}{2}\sqrt{\frac{\fV(\P_{s, a}, \optV_{\calX,g})\iota_{s,a}}{n^+(s, a)}} - \frac{c_2L\iota_{s,a}}{2n^+(s, a)} \tag{$\max\{a,b\}\geq\frac{a+b}{2}$}\\
        &\leq \optQ_{\calX,g}(s, a). \tag{\pref{eq:optV}}
    \end{align*}
    This also proves that $V^{(i)}(s)\leq \optV_{\calX,g}(s)$ for all $s\in\calX$. Moreover, for $s\notin\calX, s\neq g$, $V^{(i)}(s) = 1 + V^{(i-1)}(s_0) \leq 1 + \optV_{\calX,g}(s_0) = \optV_{\calX,g}(s)$. Finally, $ V^{(i)}(g) = \optV_{\calX,g}(g) = 0$. This proves that $V^{(i)}(s)\leq \optV_{\calK,g}(s)$ for all $s\in\calS$, thus concluding the proof.
\end{proof}

\begin{lemma}[Bounded Error]
    \label{lem:bounded error}
    There exists a function $N_0(z_0,z'_0,\delta_0,\delta)\lesssim L^2z_0\ln\frac{z'_0}{\delta_0\delta}$ such that, for goal set $\calG$ with $\acalS_{L(1+\epsilon)}\subseteq\calG\subseteq\calS$ and $\delta_0\in(0,1)$, with probability at least $1-\delta$ over the randomness of a counter $n$ incrementally collecting samples from transition function $P$, for any $\calX\subseteq\acalS_{L(1+\epsilon)}$ with $|\calX|\leq z_0$, $g\in\calG\setminus\calX$, precision $\xi\in(0, \frac{1}{8})$, and $\delta'\in[\delta_0,1)$, if $z'_0\geq|\calG|$ and $n(s, a)\geq N_0(z_0,z'_0,\delta_0,\delta)$ for all $(s, a)\in\calX\times\calA$, then $V^{\pi_g}_g(s)\leq 2V(s)$ for all $s\in\calS$, where $(\_, V,\pi_g)=\VISGO(\calX,g,\xi,n,\delta')$ is the output of \pref{alg:VISGO}.
    Also define $N_0(z_0,\delta)=N_0(z_0,S,\delta,\delta)$ and $N^{\rightarrow}_0(\delta)=N_0(\aS_{L(1+\epsilon)},|\bcalU|,\delta,\delta)$ (recall that $|\bcalU|\leq 2LA\aS_{L(1+\epsilon)}$).
\end{lemma}
\begin{proof}
    Note that the statement clearly holds if VISGO returns a value function $V=\infty$. Otherwise, $\norm{V^{(i)}}_{\infty}\leq 2L$ for any $i\leq l$, where $l$ is the index of the last iteration in \pref{alg:VISGO}. By \pref{lem:dPV}, with probability at least $1-\delta$\footnote{this holds under the same good event of \pref{lem:dPV}, which does not depend on the chosen $\calX,g,\delta',\xi$}, for any status of $n$, $(s, a)\in\calX\times\calA$, and $V$ s.t. $\|V\|_\infty \leq 2L$,
    \begin{align*}
        \abr{(P_{s,a} - \tilP_{s,a})V} &\leq \abr{(P_{s,a} - \P_{s,a})V} + \abr{(\P_{s,a}-\tilP_{s,a})V}\\
        &\lesssim L\sqrt{\frac{z_0\iota'}{n(s, a)}} + \frac{Lz_0\iota'}{n(s, a)} + \frac{(\P_{s,a}+\Ind_g)V}{n(s,a)+1},
    \end{align*}
    where $\tilP_{s,a}$ and $\P_{s,a}$ are as defined in \pref{alg:VISGO} with counter $n$ and $\iota'=\tilo{\ln\frac{z'_0}{\delta}}$ by $|\calG|\leq z'_0$. Clearly, there exists $n_1=\tilo{L^2z_0\ln(|\calG|/\delta)}$, such that when $n(s, a)\geq n_1$, we have $|(P_{s,a}-\tilP_{s,a})V|\leq \frac{1}{8}$.
    Moreover, we have
    \begin{align*}
        b^{(l)}(s,a) \lesssim \max\cbr{\sqrt{\frac{\fV(\P_{s, a}, V^{(l-1)})}{n(s, a)}}, \frac{L}{n(s, a)} } \lesssim \frac{L}{\sqrt{n(s, a)}}.
    \end{align*}
    Then there exist $n_2=\tilo{L^2\ln(1/\delta_0)}$ such that when $n(s,a)\geq n_2$, $b^{(l)}(s,a)\leq\frac{1}{8}$.
    Thus when $n(s,a)\geq\max\{n_1,n_2\}$ for all $s\in\calX,a\in\calA$, we can apply the same conclusion as in the proof of \pref{lem:bounded error fresh} as get the desired result.
\end{proof}

\begin{lemma}[Bounded Error with Fresh Samples]
    \label{lem:bounded error fresh}
    There exists a function $N_1(x,\delta_0,\delta)\lesssim L^2x\ln \frac{x}{\delta_0\delta}$ (also define $N_1(x,\delta)=N_1(x,\delta,\delta)$) such that for $\calX\subseteq\calS$, $g\in\calS\setminus\calX$, $\delta_0\in(0, 1)$, $\delta\in(0,1)$, $n$ a counter incrementally collecting samples from transition function $P$, and assume that $\calX,g,\delta_0$ are independent of $n$,
    with probability at least $1-\delta$, for any precision $\xi\in(0, \frac{1}{8})$ and $\delta'\in[\delta_0,1)$, if $n(s, a)\geq N_1(|\calX|,\delta_0,\delta)$ for all $(s, a)\in\calX\times\calA$, then $V^{\pi_g}_g(s)\leq 2V(s)$ for all $s\in\calS$, where $(\_, V,\pi_g)=\VISGO(\calX,g,\xi,n,\delta')$ is the output of \pref{alg:VISGO}.
\end{lemma}
\begin{proof}
    Let $y=\calS\setminus(\calX\cup\{g\})$ and $\iota^n_{s,a}=\ln\frac{4|\calX|^2 A n(s,a)}{\delta}$.
    Consider the following events:
    \begin{align*}
        E_1 &:= \left\{ \forall s\in\calX, a\in\calA, s' \in \calX, n(s,a) \geq 1 : |P_{s,a}(s') - \P_{s,a}(s')| \leq 2\sqrt{\frac{2P_{s,a}(s') \iota^n_{s,a}}{n(s,a)}} + \frac{2\iota^n_{s,a}}{n(s,a)} \right\},
        \\ E_2 &:= \left\{\forall s\in\calX, a\in\calA, n(s,a) \geq 1 :  |P_{s,a}(y) - \P_{s,a}(y)| \leq 2\sqrt{\frac{2 P_{s,a}(y)\iota^n_{s,a}}{n(s,a)}} + \frac{2\iota^n_{s,a}}{n(s,a)} \right\}.
    \end{align*}
    By \pref{lem:anytime bernstein} and a union bound, they hold simultaneously with probability at least $1-\delta$. We carry out the proof conditioned on these events holding.
    
    For any $\calX,g,\xi,n,\delta'$, the statement clearly holds if $V=\infty$.
    Otherwise, $\norm{V^{(i)}}_{\infty}\leq 2L$ for any $i\leq l$, where $l$ is the index of the last iteration in \pref{alg:VISGO}. Take any status of counter $n$, precision $\xi \in(0, \frac{1}{8})$, $\delta'\in[\delta_0,1)$. Let $V$ and $\pi_g$ be the output of \pref{alg:VISGO} with these parameters such that $\|V\|_\infty \leq 2L$. Since $V$ is restricted on $\calX\cup\{g\}$, we have $V(s') = 1 + V^{(l-1)}(s_0)$ for any $s' \notin \calX\cup\{g\}$. Then, for any $(s, a)\in\calX\times\calA$,
    \begin{align*}
        &\abr{(P_{s,a} - \tilP_{s,a})V} \leq \abr{(P_{s,a} - \P_{s,a})V} + \abr{(\P_{s,a}-\tilP_{s,a})V}
        \\ &\leq \abr{\sum_{s'\in\calX} (P_{s,a}(s') - \P_{s,a}(s'))V(s')} + \abr{(P_{s,a}(y) - \P_{s,a}(y))(1 + V^{(l-1)}(s_0))} + \abr{(\P_{s,a}-\tilP_{s,a})V}
        \\ &\leq 2L\sum_{s'\in\calX}\abr{ P_{s,a}(s') - \P_{s,a}(s')} + 2L\abr{P_{s,a}(y) - \P_{s,a}(y)} + \abr{(\P_{s,a}-\tilP_{s,a})V}
        \\ &\lesssim\frac{L\sqrt{|\calX| \ln(|\calX|)}}{\sqrt{n(s, a)}} + \frac{L|\calX| \ln(|\calX|)}{n(s, a)} + \frac{(\P_{s,a}+\Ind_g)V}{n(s,a)+1},
    \end{align*}
    where in the last step we applied Cauchy-Schwarz inequality, the good events, the definition of $\tilP_{s,a}$, and removed logarithmic terms and constants.
    Clearly, there exists $n_1=\tilo{L^2|\calX|\ln (|\calX|/\delta)}$, such that when $n(s, a)\geq n_1$, we have $|(P_{s,a}-\tilP_{s,a})V|\leq \frac{1}{8}$.
    Moreover, we have
    \begin{align*}
        b^{(l)}(s,a) \lesssim \max\cbr{\sqrt{\frac{\fV(\P_{s, a}, V^{(l-1)})}{n(s, a)}}, \frac{L}{n(s, a)} } \lesssim \frac{L}{\sqrt{n(s, a)}}.
    \end{align*}
    Then there exist $n_2=\tilo{L^2\ln(1/\delta_0)}$ such that when $n(s,a)\geq n_2$, $b^{(l)}(s,a)\leq\frac{1}{8}$.
    Thus when $n(s,a)\geq\max\{n_1,n_2\}$ for all $s\in\calX,a\in\calA$, for any $s\in\calX$,
    \begin{align*}
        &V(s) = V^{(l)}(s) \geq 1 + \tilP_{s,\pi_g(s)}V^{(l-1)}(s) - b^{(l)}(s, \pi_g(s)) 
        \\ &\geq 1 - \xi + \tilP_{s,\pi_g(s)}V^{(l)} - b^{(l)}(s,\pi_g(s))
        \\ &\geq 1 - \xi + P_{s,\pi_g(s)}V - \abr{(P_{s,\pi_g(s)} - \tilP_{s,\pi_g(s)})V} - b^{(l)}(s,\pi_g(s))
        \geq \frac{1}{2} + P_{s,\pi_g(s)}V(s),
    \end{align*}
    where we used the definition of $V^{(l)}$, the stopping condition of VISGO, and the previously derived bounds.
    For $s\notin\calX$, we have $V(s)=(1 + V^{(l-1)}(s_0))\Ind\{s\neq g\}\geq (\frac{1}{2}+V(s_0))\Ind\{s\neq g\}$.
    Applying this recursively gives $V(s)\geq \frac{1}{2}V^{\pi_g}_g(s)$.
    This completes the proof.
\end{proof}

\begin{lemma}
    \label{lem:subset opt}
    For any subsets $\calX$ and $\calX'$ such that $\calX\subseteq\calX'\subseteq\calS$, any $g\in\calS\setminus\calX'$, $\xi > 0$, counter $n$, and $\delta\in(0, 1)$, we have $V_{\calX'}(s)\leq V_{\calX}(s)$ for any $s\in\calS$, where we define $V_{\calX''}=\VISGO(\calX'',g,\xi,n,\delta)$ (see \pref{alg:VISGO}) for any $\calX''\subseteq\calS$.
\end{lemma}
\begin{proof}
    For any $\calX''\subseteq\calS$,
    denote by $Q^{(i)}_{\calX''}$ and $V^{(i)}_{\calX''}$ the values of $Q^{(i)}$ and $V^{(i)}$ in \pref{alg:VISGO} respectively when computing $V_{\calX''}$.
    It suffices to prove that $V^{(i)}_{\calX'}(s)\leq V^{(i)}_{\calX}(s)$ for any $s\in\calS$ and $i\geq 0$ by induction.
    The base case $i=0$ is clearly true by initialization.
    When $i>0$, we consider three disjoint cases:
    1) if $s\in\calX$, by the induction step and \pref{lem:mvp}, for any $a\in\calA$,
    \begin{align*}
        &1 + \tilP_{s, a}V^{(i-1)}_{\calX'} - \max\cbr{c_1\sqrt{\frac{\fV(\P_{s, a}, V^{(i-1)}_{\calX'})\iota_{s,a}}{n^+(s, a)}}, \frac{c_2L\iota_{s,a}}{n^+(s, a)} }\\
        &\leq 1 + \tilP_{s, a}V^{(i-1)}_{\calX} - \max\cbr{c_1\sqrt{\frac{\fV(\P_{s, a}, V^{(i-1)}_{\calX})\iota_{s,a}}{n^+(s, a)}}, \frac{c_2L\iota_{s,a}}{n^+(s, a)} }.
    \end{align*}
    This implies that $V^{(i)}_{\calX'}(s)\leq V^{(i)}_{\calX}(s)$ for $s\in\calX$.
    2) if $s\in\calX'\setminus\calX$, we have:
    $V^{(i)}_{\calX'}(s) \leq Q^{(i)}_{\calX'}(s, \reset) \leq 1 + \tilP_{s,\reset}V^{(i-1)}_{\calX'} \overset{\text{(i)}}{\leq} 1 + V^{(i-1)}_{\calX'}(s_0) \overset{\text{(ii)}}{\leq} 1 + V^{(i-1)}_{\calX}(s_0) = V^{(i)}_{\calX}(s)$,
    where step (i) is by $P_{s,\reset}(s_0)=1$ and step (ii) is by the induction step.
    3) if $s\in\calS\setminus\calX'$, by the induction step we have $V^{(i)}_{\calX'}(s)=(1+V^{(i-1)}_{\calX'}(s_0))\Ind\{s\neq g\}\leq (1+V^{(i-1)}_{\calX}(s_0))\Ind\{s\neq g\}=V^{(i)}_{\calX}(s)$.
    Combining these three cases completes the proof.
\end{proof}
\clearpage

\section{Analysis of \pref{alg:LOGSSD}}\label{app:logs}
In this section, we assume the state space is finite (i.e., $S = |\calS| < \infty$).

\subsection{Properties of the sets built by \pref{alg:LOGSSD}}

\begin{lemma}
    \label{lem:calK.easy}
    Denote by $\calK_r$ the set $\calK$ at the end of each round $r$, by $g^\star_r$ the goal selected in such a round, and by $\pi_{g^\star_r,r}$ its corresponding policy (computed by VISGO in \pref{line:compute V.easy}). With probability at least $1-\delta$ over the randomness of \pref{alg:LOGSSD}, we have that, for any round $r$,
    \begin{itemize}
        \item $\calK_{r} \subseteq \acalS_{L(1+\epsilon)}$;
        \item if \pref{line:goal condition.easy} is False, then $\|V^{\pi_{\gstar_r,r}}_{\gstar_r}\|_{\infty} \leq 4L$ which implies $\|V^{\star}_{\calK_{r-1},\gstar_r}\|_\infty \leq 4L$;
        \item for all $g \in \calK_r$, $\|V^{\tilpi_{g}}_{g}\|_{\infty} \leq 4L$ and $V^{\tilpi_g}_g(s_0)\leq L(1+\epsilon)$.
    \end{itemize}
\end{lemma}
\begin{proof}
    Clearly, $\calK_1 = \{s_0\} \subseteq \acalS_{L(1+\epsilon)}$. Then, consider a round $r\geq 2$ and suppose $\calK_{r-1}\subseteq \acalS_{L(1+\epsilon)}$ (inductive hypothesis). If, in this round, the algorithm selects a goal $\gstar_r\in \calU\setminus\acalS_{L(1+\epsilon)}$, \pref{line:goal condition.easy} is False, and a skip round is not triggered, then \pref{line:failure.easy} is reached. We now prove that the ``failure test'' in that line triggers.

    Note that every time $\calK$ is updated, the sampling at \pref{line:fill N.easy} guarantees that for all $(s,a) \in \calK_{r-1} \times \calA$, $\N_{r-1}(s,a) \geq O(L^2|\calK_{r-1}|\log(S/\delta))$. By~\pref{lem:bounded error}, since $\calK_{r-1} \subseteq \acalS_{L(1+\epsilon)}$ (inductive hypothesis), we have that
    \begin{equation}\label{eq:inductive.error.bound.easy}
      \mathbb{P}\left(\forall g \in \calS \setminus \calK_{r-1} : V^{\pi_g}_g(s) \leq 2 V_{\calK_{r-1}, g}(s) \right) \geq 1-\frac{\delta}{4r^2}.
    \end{equation}
    where $(\_, V_{\calK_{r-1}, g},\_) = \VISGO(\calK_{r-1},g,\xi_r,\N_{r-1},\frac{\delta}{4r^2 S^2})$ and $\xi_r$ is the value of $\epsilon_{\mathrm{VI}}$ used in round $r$.

    Note that \VISGO returns a value function that is either $\infty$ or bounded by $2L$ for all states (see Alg.~\ref{alg:VISGO}). Since $\gstar_r$ passes the test of \pref{line:goal condition.easy}, then $V^{\pi_{\gstar_r,r}}_{\gstar_r}(s) \leq 2 V_{\calK_{r-1},\gstar_r }(s)\leq 4L$, for all $s \in \calS$. Combining this with \pref{lem:V pi mean} and definition of $\lambda = N_{\dev}(32L, \frac{\epsilon}{256}, \frac{\delta}{4 r^2})$, we have $\hattau\geq V^{\pi_{\gstar}}_{\gstar}(s_0) - L\epsilon/2$ with probability at least $1-\frac{\delta}{4r^2}$. By assumption on $g^\star_r$ and since $\pi_{\gstar_r,r}$ is restricted on $\calK_{r-1}\subseteq\acalS_{L(1+\epsilon)}$, we have $V^{\pi_{\gstar_r,r}}_{\gstar_r}(s_0) \geq V^{\star}_{\calK_{r-1},\gstar_r}(s_0) \geq V^{\star}_{\acalS_{L(1+\epsilon)},\gstar_r}(s_0) > L(1+\epsilon)$, which implies that $\hattau\geq L(1+\epsilon/2) \geq V_{\calK_{r-1}, \gstar_r}(s_0) + \epsilon L/2$ with the same probability, where the last inequality is from the goal-selection rule.  Therefore, the failure test of \pref{line:failure.easy} triggers and $g^\star_r$ is not added to $\calK_r'$ or $\calK_r$. Therefore, by the inductive hypothesis $\calK_{r}\subseteq \acalS_{L(1+\epsilon)}$. A union bound over all $r \geq 1$ yields the first statement with probability at least $1-\delta$.

    To prove the second statement, note that we already proved above that $V^{\pi_{\gstar_r}}_{\gstar_r,r}(s) \leq 4L$ at any round $r$ where \pref{line:goal condition.easy} is False (i.e., where $\gstar_r$ reaches the policy evaluation step). Since $\pi_{\gstar_r,r}$ is restricted on $\calK_{r-1}$, we clearly have $V^{\star}_{\calK_{r-1},\gstar_r}(s) \leq V^{\pi_{\gstar_r,r}}_{\gstar_r}(s) \leq 4L$. This proves the second statement for any round $r$, which holds with the same $1-\delta$ probability.

    Finally, the third statement is a simple consequence of the fact that any goal $g\in\calK_r$ must have reached the policy evaluation step in some round $r' < r$ and the round was successful, and thus $\|V^{\tilpi_{g}}_{g}\|_{\infty} \leq 4L$ by the second statement. Moreover, by the definition of success round, value of $\lambda$ and \pref{lem:V pi mean}, we have that, for each $g\in\calK_r$, there exists $r' < r$ such that $V^{\tilpi_g}_g(s_0)=V^{\pi_{\gstar_{r'},r'}}_{\gstar_{r'}}(s_0)\leq \hattau + \frac{L\epsilon}{2} \leq V_{\calK_{r'-1},\gstar_{r'}}(s_0) + L\epsilon \leq L(1+\epsilon)$. This holds with the same $1-\delta$ probability as above since we have already union bounded across the application of \pref{lem:V pi mean} for all $g^\star_r$ at all $r\geq 1$.
\end{proof}

\begin{lemma}\label{lem:calU.easy}

    With probability at least $1-2\delta$, for any round $r \geq 1$ in which $\calK_r$ is updated (i.e., \pref{line:compute calU'.easy} is executed), $\calT_L(\calK_r) \setminus \calK_r \subseteq \calU_r$.
\end{lemma}
\begin{proof}
    For any round $r$, let $\calF_{r-1}$ denote the sigma-algebra generated by the history up to the previous round. Let $H_k$ denote the event ``\pref{line:compute calU'.easy} is executed at round $k$''. Note that $H_k$ is $\calF_{k-1}$-measurable since no random step happens before \pref{line:compute calU'.easy} in round $r$. Moreover, define the events $E_r := \{\forall g\in\calK_r : \|V_g^{\tilde\pi_g}\|_\infty \leq 4L\}$ and $E := \{\forall r\geq 1 : E_r\}$. Note that $E$ holds with probability at least $1-\delta$ by \pref{lem:calK.easy}. We have
    \begin{align*}
        \mathbb{P}\left( \exists r \geq 1: H_r, \calT_L(\calK_r) \setminus \calK_r \not\subseteq \calU_r \right) 
        &\leq \mathbb{P}\left( \exists r \geq 1: H_r, \calT_L(\calK_r) \setminus \calK_r \not\subseteq \calU_r, E \right) + \mathbb{P}\left( \neg E \right) \tag{union bound}
        \\ &\leq \mathbb{P}\left( \exists r \geq 1: H_r, \calT_L(\calK_r) \setminus \calK_r \not\subseteq \calU_r, E_r \right) + \delta \tag{\pref{lem:calK.easy}}
        \\ &\leq \sum_{r\geq 1}\mathbb{P}\left( \calT_L(\calK_r) \setminus \calK_r \not\subseteq \calU_r, E_r, H_r \right) + \delta. \tag{union bound}
        \\ &\leq \sum_{r\geq 1}\mathbb{P}\left( \calN(\calK_r, \frac{1}{2L}) \not\subseteq \calU_r, E_r, H_r \right) + \delta. \tag{$\calT_L(\calK_r) \setminus \calK_r \subseteq \calN(\calK_r, \frac{1}{2L})$}
    \end{align*}
    Now take any round $r\geq 1$. Recall that $\calU_r$ is built by sampling from each $(s,a) \in \calK_r \times \calA$ exactly $\mu_r := 2L\log(4SALr^2/\delta)$ times. For each $(s,a) \in \calK_r \times \calA$, let $s_{i,s,a}$ be the $i$-th sample (i.e., $s_{i,s,a} \sim P_{s,a}$) for $i\in[\mu_r]$. In order to collect each sample $s_{i,s,a}$, we must play the policy $\tilde\pi_s$ from $s_0$ until reaching $s$. Note that, under event $E_r$, $\|V_s^{\tilde\pi_s}\|_\infty \leq 4L$ for all $s\in\calK_r$, hence all the states in $\calK_r$ are reached with probability one (so $s_{i,s,a}$ is well defined for all $s,a,i$). Then, for any fixed $\calK_r$,
    \begin{align*}
        \mathbb{P}\left( \calN(\calK_r, \frac{1}{2L}) \not\subseteq \calU_r, E_r, H_r \mid \calK_r \right) 
        &\leq \mathbb{P}\left(\exists s' \in \calN(\calK_r, \frac{1}{2L}), \forall (s,a) \in \calK_r \times \calA, \forall i \in [\mu_r]: s_{i,s,a} \neq s' \mid \calK_r \right) \\
        &\leq \sum_{s' \in \calN(\calK_r, \frac{1}{2L})} \mathbb P\left(\forall  (s,a) \in \calK_r \times \calA, \forall i \in [\mu]: s_{i,s,a} \neq s'\right) \tag{union bound}
        \\ &\leq \sum_{s' \in \calN(\calK_r, \frac{1}{2L})} \max_{(s,a) \in \calK_r \times \calA} \mathbb P\left(\forall i \in [\mu]: s_{i,s,a} \neq s'\right) \tag{trivial}
        \\ & \leq \sum_{s' \in \calN(\calK_r, \frac{1}{2L})} \max_{(s,a) \in \calK_r \times \calA}  \prod_{i \in [\mu_r]} (1-P(s'|s,a))  \tag{all $s_{i,s,a}$ are i.i.d.}
        \\ & \leq \sum_{s' \in \calN(\calK_r, \frac{1}{2L})} \left(1-\frac{1}{2L}\right)^{\mu_r} \tag{definition of $\calN(\calK_r, \frac{1}{2L})$}
        \\ & \leq \sum_{s' \in \calN(\calK_r, \frac{1}{2L})}\frac{\delta}{4LASr^2} \leq \frac{\delta}{2r^2}.
    \end{align*}
    Now let $\Omega_{r-1}$ denote the sample space under which $\calF_{r-1}$ is generated, such that $\sum_{\omega\in\Omega_{r-1}}\mathbb{P}(\omega) = 1$. Noting that $\calK_r$ is measurable w.r.t. $\calF_{r-1}$, define $\calK_r(\omega)$ as the set $\calK_r$ obtained after history $\omega$. Then,
    \begin{align*}
        \mathbb{P}\left( \calN(\calK_r, \frac{1}{2L}) \not\subseteq \calU_r, E_r, H_r \right) 
        & = \sum_{\omega\in\Omega_{r-1}}  \mathbb{P}\left( \calN(\calK_r, \frac{1}{2L}) \not\subseteq \calU_r, E_r, H_r \mid \omega \right) \mathbb{P}(\omega)
        \\ & = \sum_{\omega\in\Omega_{r-1} : E_r,H_r}  \mathbb{P}\left( \calN(\calK_r, \frac{1}{2L}) \not\subseteq \calU_r \mid \omega \right) \mathbb{P}(\omega)
        \\ &= \sum_{\omega\in\Omega_{r-1} : E_r, H_r}  \mathbb{P}\left( \calN(\calK_r, \frac{1}{2L}) \not\subseteq \calU_r \mid \calK_r(\omega), E_r, H_r \right) \mathbb{P}(\omega) \leq \frac{\delta}{2r^2}.
    \end{align*}
    Plugging this into our initial inequality, we get $\mathbb{P}\left( \exists r \geq 1: H_r, \calT_L(\calK_r) \setminus \calK_r \not\subseteq \calU_r \right) \leq 2\delta$.
\end{proof}

\begin{lemma}[Restricted Optimism]
    \label{lem:V calK.easy}
    With probability at least $1-\delta$ over the randomness of \pref{alg:LOGSSD}, for any $j\in[S]$ and any round $r\geq 1$, after executing \pref{line:compute V.easy}, if $\calKstar_{j}\subseteq\calK_r$, then
    $V_{\calK_r,g}(s) \leq \optV_{\calKstar_{j},g}(s)$ for any $s\in\calS$ and $g\in\calKstar_{j+1}\setminus\calK_r$, where $\calK_r$ is the set $\calK$ immediately after the execution of \pref{line:compute V.easy}.
\end{lemma}
\begin{proof}
    Let $j \in [S]$ and $g\in \calKstar_{j+1} \setminus \calKstar_{j}$. Fix some round $r\geq 1$ s.t. $\calKstar_{j}\subseteq\calK_r$. Let $\delta_r = \frac{\delta}{4r^2S^2}$ and $(Q_{\xi}, V_{\xi},\_) = \VISGO(\calKstar_{j},g,\xi,\N,\delta_r)$. By \pref{lem:opt} \footnote{Note that, by definition, $\|V_{\calKstar_{j},g}^\star\|_{\infty} \leq L + 1 \leq 2L$ for all $g \in \calKstar_{j+1} \setminus \calKstar_{j}$ (which is a prerequisite of Lemma~\ref{lem:opt}).},
    
    \begin{equation}
        \label{eq:union.bound.restricted.opt.easy}
        \begin{aligned}
            \mathbb{P}\Big(\forall \xi > 0, s\in\calS : 
             V_{\xi}(s)\leq\optV_{\calKstar_{j},g}(s) \Big) \geq 1 - \delta_r.
        \end{aligned}
    \end{equation}
    Then, from a union bound and $|\calKstar_{j+1} \setminus \calKstar_{j}| \leq S$, the event above holds simultaneously across all $j\in[S]$, and $g\in \calKstar_{j+1} \setminus \calKstar_{j}$ with probability at least $1-\frac{\delta}{4r^2}$. This implies that the same result holds for all $g\in\calKstar_{j+1}\setminus\calK_r$ since $\calKstar_{j+1}\setminus\calK_r \subseteq \calKstar_{j+1}\setminus\calKstar_j$. A union bound implies that this holds at all rounds simultaneously with probability at least $1 - \delta$.

    Now consider the execution of \pref{line:compute V.easy} and let $\calK_r,\delta_r,\xi_r,\N_r$ be the values of the parameters used by VISGO in such a round, such that $\calKstar_{j}\subseteq\calK_r$ for some $j\in[S]$. For any $g\in\calKstar_{j+1}\setminus\calK_r$, let $(\_, V_{\calK_r,g},\_) = \VISGO(\calK_r,g,\xi_r, \N_r,\delta_r)$ and $(\_, V_{\calKstar_{j},g},\_) = \VISGO(\calKstar_{j},g,\xi_r, \N_r,\delta_r)$.
    Then, Eq.~\ref{eq:union.bound.restricted.opt.easy} implies that, for any $s\in\calS$, $V_{\calKstar_{j},g}(s)\leq\optV_{\calKstar_{j},g}(s)$. If $\calKstar_{j}\subseteq\calK_r$, by the update rule of \pref{alg:VISGO} and \pref{lem:subset opt}, we also have $V_{\calK_r,g}(s) \leq V_{\calKstar_{j},g}(s) \leq \optV_{\calKstar_{j},g}(s)$.
\end{proof}

The following lemma shows that if a set $\calK^\star_{j}\subseteq \calK$ at some round, at the next update of $\calK$ it must be that $\calK^\star_{j+1}\subseteq \calK$ (if the algorithm does not terminate) and ensures correctness, in the sense that the algorithm returns a set of states including $\acalS_L$ with high probability.

\begin{lemma}[Correctness]\label{lem:update calK.easy}
    Denote by $\calK_r$ (resp $\calU_r$) the set $\calK$ (resp. $\calU$) at the end of each round $r$. With probability at least $1-3\delta$, for any $j \geq 1$ and round $r \geq 1$ in which $\calK_r$ is updated or returned (i.e., \pref{line:compute calU'.easy} is executed) and $\calK_{r-1} \supseteq \calKstar_{j}$, we have $\calK^\star_{j+1} \subseteq \calK_r$. Moreover, under the same probability, we have that, for any $r\geq 1$, $\acalS_L\subseteq\calK_{r}$ if the algorithm terminates at round $r$.
\end{lemma}
\begin{proof}
    Define the event $E := \{ \forall r\geq 1 \text{ in which $\calK_r$ is updated}: \calT_L(\calK_r) \setminus \calK_r \subseteq \calU_r\}$. By \pref{lem:calU.easy}, it holds with probability at least $1-2\delta$. Let us carry out the proof conditioned on $E$ holding. 
    
    Take some round $r$ such that \pref{line:compute calU'.easy} is executed and $\calK_{r-1} \supseteq \calKstar_{j}$. Let $r'$ be the last round where $\calK_{r'}$ was updated (and thus $\calU_{r'}$ was created). Note that $\calK_{r'} = \calK_{r-1} \supseteq \calKstar_j$. Then, event $E$ and the definition of the sets $(\calKstar_j)_j$ directly imply that $\calKstar_{j+1} := \calT_L(\calKstar_j) \subseteq \calT_L(\calK_{r'}) \subseteq \calU_{r'} \cup \calK_{r'}$. Since $\calK_r$ can only be formed by adding states in $\calU_{r'}$ to $\calK_{r'}$, and the union of these sets contains $\calKstar_{j+1}$, if $\calKstar_{j+1} \not\subseteq \calK_r$, it must be that there exists $g\in\calU_{r-1} \cap \calKstar_{j+1}$ s.t. $V_{\calK_{r-1},g}(s_0) > L$. However, \pref{lem:V calK.easy}, which holds with probability $1-\delta$, implies that, at any round $r\geq 1$, if $\calKstar_{j}\subseteq\calK_{r-1}$, then
    $V_{\calK_{r-1},g}(s_0) \leq \optV_{\calKstar_{j},g}(s_0) \leq L$ for any $g\in\calKstar_{j+1}\setminus\calK_{r-1}$. This is a contradiction, which implies that $\calU_{r-1} \cap \calKstar_{j+1} = \emptyset$ and, thus, all states in $\calKstar_{j+1}$ must have been added to $\calK_r$. A union bound over the application of \pref{lem:calU.easy} and \pref{lem:V calK.easy} yields the statement.

    To prove the second statement, let us use the same events as above. First note that, since $\calK_1 = \calKstar_1 = \{s_0\}$, it must be that, at any round $r$, $\calK_r \supseteq \calKstar_j$ for some $j \geq 1$. Now take any round $r$ in which the algorithm terminates and suppose $\calK_{r-1} \not\supseteq \acalS_L$. Let $j^\star$ be the largest $j$ s.t. $\calK_r \supseteq \calKstar_j$. By \pref{lem:SL.operator}, it must be that $j < J$, hence $\calKstar_{j^\star +1} \supset \calKstar_{j^\star}$. Let $r'$ be the last round at which $\calK_{r'}$ was updated. Since the algorithm terminates at round $r$ it must be that $\calK_{r-1}' = \emptyset$, i.e., no state in $\calU_{r-1} = \calU_{r'}$ has been found to be added to $\calK_r$. From the same argument as above, under $E$ it must be that $\calKstar_{j^\star+1}  \subseteq \calU_{r'} \cup \calK_{r'}$. Since $\calK_{r-1} \not\supseteq \acalS_L$, and no addition to $\calK_{r-1}$ is performed as the algorithm stops at $r$, it must be that there exists $g\in\calU_{r-1} \cap \calKstar_{j^\star+1}$ s.t. $V_{\calK_{r-1},g}(s_0) > L$. However, in the first part of the proof, we already found a contradiction for this case under the event of \pref{lem:V calK.easy}. This implies that the algorithm cannot stop at $r$ since some state must be added. Hence, whenever the algorithm stops it must be that $\calK_r \supseteq \acalS_L$. This completes the proof.
\end{proof}

\begin{lemma}[Correctness under \pref{assum:id}]
    \label{lem:calK id.easy}
    Denote by $\calK_r$ the set $\calK$ at the end of each round $r$.
    With \pref{assum:id}, with probability  at least $1-5\delta$ over the randomness of Algorithm \ref{alg:LOGSSD}, for any round $r \geq 1$, we have that $\calK_r = \calKstar_j$ for some $j \in [\aS_L]$ and $\calK_{r}= \acalS_L$ if the algorithm terminates at round $r$.
\end{lemma}
\begin{proof}
    By \pref{lem:calK.easy} and \pref{lem:update calK.easy}, with probability at least $1-4\delta$, we have $\acalS_L \subseteq \calK_r \subseteq \acalS_{L(1+\epsilon)}$ if the algorithm terminates at round $r$.
    By~\pref{rem:id}, $\calK = \acalS_L$.
    Thus, it suffices to show that, at any round $r$, $\calK_r=\calKstar_j$ for some $j \leq |\acalS_L|$.

    The algorithm is such that $\calK_1 = \calKstar_1 = \{s_0\}$. Suppose at, in some round $r\geq 1$, we have that $\calK_{r}=\calKstar_j$ for some $j\geq 1$. By \pref{lem:update calK.easy}, with the same probability as above, if the condition of \pref{line:update.K.easy} becomes True for the first time in some round $r'>r$ (i.e., the set $\calK$ is updated in such round), then we must have $\calKstar_{j+1}\subseteq\calK_{r'}$ at then end of round $r'$. We shall prove that we also have $\calK_{r'}\subseteq\calKstar_{j+1}$, which implies the statement.

    Take any round $r$ such that $\calK_{r-1}=\calKstar_j$ and $\gstar_r\in \calU\setminus\calKstar_{j+1}$. Since, the last time $\calK$ was updated \pref{line:fill N.easy} was called, we must have $\N_{r-1}(s,a) \geq O(L^2|\calKstar_j|\log(S/\delta))$ for all $(s,a) \in \calKstar_j \times \calA$. Then, by~\pref{lem:bounded error}, with probability at least $1-\frac{\delta}{4r^2}$, for all $s\in\calS$, $V^{\pi_{\gstar_r}}_{\gstar_r}(s) \leq 2 V_{\calK_{r-1}, {\gstar_r}}(s) \leq 4L$  due to properties of \VISGO if \pref{line:goal condition.easy} is False. If a skip round is not triggered, 
    combining this with \pref{lem:V pi mean} and definition of $\lambda$, we have $\hattau\geq V^{\pi_{\gstar_r}}_{\gstar_r}(s_0) - L\epsilon/2$ with probability at least $1-\frac{\delta}{4r^2}$.

    By \pref{assum:id}, assumption on $g^\star_r$, and since $\pi_{\gstar_r}$ is restricted on $\calK_{r-1}=\calKstar_j$, we have $V^{\pi_{\gstar_r}}_{\gstar_r}(s_0) \geq V^{\star}_{\calKstar_j,\gstar_r}(s_0) > L(1+\epsilon)$, which implies that $\hattau\geq L(1+\epsilon/2) \geq V_{\calK_{r-1}, \gstar_r}(s_0) + \epsilon L/2$ with the same probability, where the last inequality is from the fact that \pref{line:goal condition.easy} is False.  Therefore, the failure test triggers and $g^\star_r$ is not added to $\calK_r'$ or $\calK_{r}$ since a failure round is triggered. This holds with probability at least $1-\delta$ across all rounds by a union bound. Therefore, for any round $r$ in which $\calK$ is updated and $\calK_{r-1} = \calKstar_j$, we must have $\calK_{r}\subseteq \calKstar_{j+1}$. This concludes the proof, and the statement holds with probability at least $1-5\delta$ by a union bound.
\end{proof} 

\subsection{Analysis of Policy Evaluation}

We consider the regret over the trajectories generated in the policy evaluation phase. We concatenate all policy evaluation episodes in all rounds and index them with $k \geq 1$.
To make the notation consistent with \pref{alg:SD}, we treat the whole learning procedure as an artificial trial.
Let $\calK_k$, $V_k$, and $Q_k$ be the $\calK$, $V_{\calK,\gstar}$, and $Q_{\calK,\gstar}$ in episode $k$. Let $\pi_k$ and $g_k$ be the corresponding policy $\pi_{\gstar}$ and goal $\gstar$.
Denote by $\calF_k$ the $\sigma$-algebra of events up to episode $k$.
Let $K$ be the total number of episodes throughout the execution of \pref{alg:LOGSSD}.
For any sequence of indicators $\calI=\{\one_k\}_k$ with $\one_k\in\calF_{k-1}$, define $R_{K',\calI}=\sumkp(I_k - V_k(s_0))\one_k$ and $C_{K'}=\sumkp I_k$ for $K'\in[K]$.
Define $P^k_i=P_{s^k_i, a^k_i}$.
In episode $k$, when $s^k_i\in\calK$, denote by $\P^k_i$, $\tilP^k_i$, $\N^k_i$, $b^k_i$ the values of $\P_{s^k_i,a^k_i}$, $\tilP_{s^k_i, a^k_i}$, $n^+(s^k_i, a^k_i)$, and $b^{(l)}(s^k_i, a^k_i)$, where $\P$, $n^+$, $b^{(l)}$ are used in \pref{alg:VISGO} to compute $V_k$ and $l$ is the final value of $i$ in \pref{alg:VISGO};
when $s^k_i\notin \calK$, define $\P^k_i=\Ind_{s_0}$, $\N^k_i=\infty$, and $b^k_i=0$.
Also define $\epsilon_k,\delta_k$ as the value of $\epsilon_{\VI},\delta$ used in \pref{alg:VISGO} to compute $V_k$.
Note that $I_k<\infty$ with probability $1$ by \pref{line:skip.easy}, and $s^k_{I_k+1}\neq g$ only when a skip round is triggered in episode $k$.


\subsubsection{Regret bound without \pref{assum:id}}

\begin{lemma}
    \label{lem:regret.easy}
    For any sequence of indicators $\calI=\{\one_k\}_k$ with $\one_k\in\calF_{k-1}$, we have, with probability at least $1-6\delta$, for any $K'\in[K]$,
    $$R_{K',\calI} \lesssim  L \log(SAL/\delta)^2 \log(K)\sqrt{\aS_{L(1+\epsilon)}\Gamma_{L(1+\epsilon)}AK'} + L{\aS_{L(1+\epsilon)}}^2A(\log K')^2 \log(SAL/\delta)^3.$$
    Moreover, $C_{K'} \lesssim LK' + L{\aS_{L(1+\epsilon)}}^2A(\log K')^2 \log(SAL/\delta)^3$.
\end{lemma}
\begin{proof}
    We start by decomposing the regret as
    \begin{align*}
        \sumkp\rbr{I_k - V_k(s_0)}\one_k &\leq \sumkp\sum_{i=1}^{I_k}\rbr{1 + V_k(s^k_{i+1}) - V_k(s^k_i)}\one_k \tag{$\pm\sumi V_k(s_{i+1}^k)$}\\
        &\leq \sumkp\sum_{i=1}^{I_k}\rbr{(\Ind_{s^k_{i+1}} - P^k_i)V_k + (P^k_i - \P^k_i)V_k + (\P^k_i - \tilP^k_i)V_k+ b^k_i + \epsilon_k}\one_k, \tag{definition of $V_k$}
    \end{align*}
    where the last inequality uses that $V_k^{(l)}(s) = 1 + \tilP_{s,a}^k V_{k}^{(l-1)} - b_{s,a}^k$ for any $s\in\calK_k,a\in\calA$, where $l$ is the index of the last iteration of VISGO when called with $(\_, V_k, \pi_g)=\VISGO(\calK_k, g_k, \epsilon_k, \N_k, \delta_k)$, and $\|V_k^{(l)} - V_k^{(l-1)}\|_{\infty} \leq \epsilon_k$ by definition of its termination condition (recall that $V_k$ is bounded since \pref{line:goal condition.easy} was passed). 
    Note that, if $s_i^k \notin \calK_k$, then the $i,k$ term in the sum of the second line is clearly an upper bound to the corresponding term in the first line. We bound the terms above separately.

    \paragraph{First term}

    By \pref{lem:anytime freedman} and $\norm{V_k}_{\infty}\leq 2L$ (by \VISGO and since \pref{line:goal condition.easy} was passed), with probability at least $1-\delta$,
    \begin{align*}
        \sumkp\sumi (\Ind_{s^k_{i+1}} - P^k_i)V_k\one_k \leq \sqrt{\sumkp\sumi \one_k \fV(P^k_i, V_k)\iota} + L \iota,
    \end{align*}
    where $\iota = 9\log(16L^2 C_{K'}^3/\delta)$.

    \paragraph{Second term}
    Note that, by the event of \pref{lem:calK.easy}, $\calK_k \subseteq \acalS_{L(1+\epsilon)}$ in all episodes $k$. Moreover, when $s_i^k\notin \calK_k$, the $k,i$ term in the sum is zero by definition of $P_i^k$ and $\P_i^k$. Therefore, we have all the preconditions to apply \pref{lem:dPV} on terms $(P^k_i - \P^k_i)V_k$ for all $i,k$ s.t. $s_i^k\in \calK_k$, which yields, with probability $1-\delta$, 
    \begin{align*}
        \sumkp\sumi(P^k_i - \P^k_i)V_k\one_k &\lesssim \sumkp\sumi\left(\sqrt{\frac{\Gamma_{L(1+\epsilon)}\fV(P^k_i, V_k)\iota'}{\N^k_i}} + \frac{L\aS_{L(1+\epsilon)}\iota'}{\N^k_i}\right),
    \end{align*}
    where $\iota' = O(\log \frac{SALC_{K'}}{\delta})$. Note that \pref{lem:dPV} already union bounds across all possible counts, value functions and state-action pair, so we do not need an extra union bound over episodes and steps here.
    
    Then, by \pref{lem:sum N} and Cauchy-Schwarz inequality, with the same probability,
    \begin{align*}
        \sumkp\sumi(P^k_i - \P^k_i)V_k\one_k \lesssim \sqrt{\aS_{L(1+\epsilon)}\Gamma_{L(1+\epsilon)}A\sumkp\sumi \fV(P^k_i, V_k)\iota''} + L{\aS_{L(1+\epsilon)}}^2A\iota'',
    \end{align*}
    where $\iota'' = O(\log(SALC_{K'}/\delta)\log(C_{K'}))$.

    \paragraph{Third term}

    By the expressions of $\tilP_i^k$ and $\P_i^k$ (cf. \pref{alg:VISGO}) and \pref{lem:sum N},
    \begin{align*}
        \sumkp\sumi(\P^k_i - \tilP^k_i)V_k\one_k \leq \sumkp\sumi\one_k \frac{(\P^k_i+\Ind_g)V_k}{\N_i^k+1} \lesssim L\aS_{L(1+\epsilon)}A \log(C_{K'}). \tag{$\Ind_g(s')\triangleq\Ind\{s'=g\}$}
    \end{align*}

    \paragraph{Fourth and fifth term}

    By \pref{lem:sum b} and \pref{lem:sum eps}, with probability at least $1-\delta$,
        \begin{align*}
            \sumkp\sumi (b^k_i+\epsilon_k)\one_k \lesssim \sqrt{\aS_{L(1+\epsilon)}A\sumkp\sumi\fV(P^k_i, V_k)\iota'} + L{\aS_{L(1+\epsilon)}}^{1.5}A\iota'.
        \end{align*}
    
    \paragraph{Combining all terms}

    Note that all the derived bounds can be absorbed into the one of the second term. Plugging everything back to our initial expression of the regret,
\begin{align*}
    \sumkp\rbr{I_k - V_k(s_0)}\one_k 
    &\lesssim \sqrt{\aS_{L(1+\epsilon)}\Gamma_{L(1+\epsilon)}A\sumkp\sumi\fV(P^k_i, V_k)\iota''}  + L{\aS_{L(1+\epsilon)}}^2A\iota''
    \\ &\lesssim \sqrt{L\aS_{L(1+\epsilon)}\Gamma_{L(1+\epsilon)}AC_{K'}\iota''} + L{\aS_{L(1+\epsilon)}}^2A\iota''. \tag{\pref{lem:sum var Vk}}
\end{align*}
    Note that $\iota'' \lesssim \log(SAL/\delta) (\log C_{K'})^2$. Now assuming $\one_k=1$ for all $k$, we can solve an inequality to find $C_K$. First, using that $\log(x) \leq x^\alpha/\alpha$ for any $x,\alpha > 0$ together with the derived regret bound, we can find the crude bound on $C_K$,
    \begin{align*}
        C_{K'} \lesssim \left(\sumk V_k(s_0) + L{\acalS_{L(1+\epsilon)}}^2A\log(SAL/\delta) \right)^4 \leq \left(K'L + L{\acalS_{L(1+\epsilon)}}^2A\log(SAL/\delta) \right)^4.
    \end{align*}
    This implies that $\iota'' \lesssim (\log K')^2 \log(SAL/\delta)^3$. Plugging this into the regret bound, we get a quadratic inequality in $C_{K'}$. Solving it yields
    \begin{align*}
        C_{K'} \lesssim \sumkp V_k(s_0) + L{\aS_{L(1+\epsilon)}}^2A(\log K')^2 \log(SAL/\delta)^3 \leq LK' + L{\aS_{L(1+\epsilon)}}^2A(\log K')^2 \log(SAL/\delta)^3.
    \end{align*}
    Plugging this back into the regret bound gives the stated bound. Throughout the proof we used following events with the corresponding probabilities:
    \begin{itemize}
        \item \pref{lem:anytime freedman}: $1-\delta$
        \item \pref{lem:calK.easy}: $1-\delta$
        \item \pref{lem:dPV}: $1-\delta$
        \item \pref{lem:sum b}: $1-\delta$
        \item \pref{lem:sum var Vk}: $1-2\delta$
    \end{itemize}
    A union bound concludes the proof.
\end{proof}

\subsubsection{Regret bound under \pref{assum:id}}

\begin{lemma}
    \label{lem:regret-improved.easy}
    Under \pref{assum:id}, for any sequence of indicators $\calI=\{\one_k\}_k$ with $\one_k\in\calF_{k-1}$, we have, with probability at least $1-14\delta$, for any $K'\in[K]$,
    $$R_{K',\calI} \lesssim  L \log(SAL/\delta)^2 \log(K')\sqrt{\aS_{L(1+\epsilon)}AK'} + L{\aS_{L(1+\epsilon)}}^2A(\log K')^2 \log(SAL/\delta)^3.$$
    Moreover, $C_{K'} \lesssim LK' + L{\aS_{L(1+\epsilon)}}^2A(\log K')^2 \log(SAL/\delta)^3$.
\end{lemma}
\begin{proof}
	Note that, under \pref{assum:id} and by \pref{lem:calK id.easy}, in any episode, $\calK=\calKstar_j$ for some $j\leq J \leq |\acalS_{L(1+\epsilon)}| \leq S$ (cf. \pref{lem:SL.operator}). Moreover, by \pref{lem:calK.easy}, for any round in which $g^\star$ reaches the policy evaluation step, $\|\optV_{\calK,g^\star}\|_\infty \leq 4L$, which implies that $\|\optV_{\calKstar_j,g^\star}\|_\infty \leq 4L$ for some $j$ in that round. Let $\calG_j := \{g\in\calS : \|\optV_{\calKstar_j,g}\|_\infty \leq 4L\}$. Consider the event

    \begin{align*}
        E := \left\{ \forall s\in\calS,a\in\calA, j\in[S], g\in\calG_j, \forall n(s,a) \geq 1 : |(\P_{s,a}^n-P_{s,a})\optV_{\calKstar_j,g}| \leq \sqrt{\frac{\fV(P_{s, a}, \optV_{\calKstar_j,g})\iota_{s,a}'}{n(s,a)}} + \frac{L\iota'_{s,a}}{n(s,a)} \right\},
    \end{align*}
    where $\iota'_{s,a} = 8\log(2S^3An(s,a)/\delta)$. Clearly, by \pref{lem:anytime bernstein} and a union bound, $E$ holds with probability at least $1-\delta$. Then, assuming $E$ and the events of \pref{lem:calK id.easy} and \pref{lem:calK.easy} hold, we clearly have, for all episodes $k$ and steps $i$,
    \begin{align}\label{eq:bernstein-Vkstar}
        (P^k_i -\P^k_i)\optV_k\lesssim \sqrt{\frac{\fV(P^k_i, \optV_k)\iota'}{\N^k_i}} + \frac{L\iota'}{\N^k_i},
    \end{align}
    where $\iota' = O(\log(SALC_{K'}/\delta))$. Note that we inflated the $\iota'$ term with an extra $L$ since it will simplify the bounds later. Now we split the regret as
    \begin{align*}
        \sumkp\rbr{I_k - V_k(s_0)}\one_k &\leq \sumkp\sum_{i=1}^{I_k}\rbr{1 + V_k(s^k_{i+1}) - V_k(s^k_i)}\one_k \tag{$\pm\sumi V_k(s_{i+1}^k)$}\\
        &\leq \sumkp\sum_{i=1}^{I_k}\rbr{(\Ind_{s^k_{i+1}} - P^k_i)V_k + (P^k_i - \P^k_i)V_k + (\P^k_i - \tilP^k_i)V_k+ b^k_i + \epsilon_k}\one_k, \tag{definition of $V_k$}
    \end{align*}
    where the last inequality uses that $V_k^{(l)}(s) = 1 + \tilP_{s,a}^k V_{k}^{(l-1)} - b_{s,a}^k$ for any $s\in\calK_k,a\in\calA$, where $l$ is the index of the last iteration of VISGO when called with $(\_, V_k, \pi_g)=\VISGO(\calK_k, g_k, \epsilon_k, \N_k, \delta_k)$, and $\|V_k^{(l)} - V_k^{(l-1)}\|_{\infty} \leq \epsilon_k$ by definition of its termination condition (recall that $V_k$ is bounded since \pref{line:goal condition.easy} was passed). Note that, if $s_i^k \notin \calK_k$, then the $i,k$ term in the sum of the second line is clearly an upper bound to the corresponding term in the first line.
    
    We bound the terms above separately.

    \paragraph{First term}

    By \pref{lem:anytime freedman} and $\norm{V_k}_{\infty}\leq 2L$ (by \VISGO and since \pref{line:goal condition.easy} was passed), with probability at least $1-\delta$,
    \begin{align*}
        \sumkp\sumi (\Ind_{s^k_{i+1}} - P^k_i)V_k\one_k \leq \sqrt{\sumkp\sumi \one_k \fV(P^k_i, V_k)\iota} + L \iota,
    \end{align*}
    where $\iota = 9\log(16L^2 C_{K'}^3/\delta)$.

    \paragraph{Second term}

    Note that, from \eqref{eq:bernstein-Vkstar},
    \begin{align*}
        \sumkp\sum_{i=1}^{I_k} |(P^k_i - \P^k_i)V_k| \one_k 
        &\leq \sumkp\sum_{i=1}^{I_k} |(P^k_i - \P^k_i)V_k|
        \\ &= \sumkp\sumi \rbr{|(P^k_i-\P^k_i)\optV_k| + |(P^k_i-\P^k_i)(V_k-\optV_k)|}
        \\ &\leq \sumkp\sumi \rbr{\sqrt{\frac{\fV(P^k_i, \optV_k)\iota'}{\N^k_i}} + \frac{L\iota'}{\N^k_i} + |(P^k_i-\P^k_i)(V_k-\optV_k)|}.
    \end{align*}

    Note that, by the event of \pref{lem:calK.easy}, $\calK_k \subseteq \acalS_{L(1+\epsilon)}$ in all episodes $k$. Moreover, for all $k,i$, either $(s_i^k,a_i^k)\in \calK_k \times \calA$ or the second term above is zero. Since $\|V_k - V_k^\star\|_\infty \leq 6L$, we have all the preconditions to apply \pref{lem:dPV} on the terms $|(P^k_i-\P^k_i)(V_k-\optV_k)|$, which yields, with probability $1-\delta$, for all $i,k$,
\begin{align*}
    |(P^k_i-\P^k_i)(V_k-\optV_k)| &\lesssim \sqrt{\frac{\aS_{L(1+\epsilon)}\fV(P^k_i, V_k-\optV_k)\iota'}{\N^k_i}} + \frac{L\aS_{L(1+\epsilon)}\iota'}{\N^k_i},
\end{align*} 
where $\iota'$ was defined above. Note that \pref{lem:dPV} already union bounds across all possible counts, value functions and state-action pair, so we do not need an extra union bound over episodes and steps here. By $\var[X+Y]\leq 2(\var[X]+\var[Y])$, we have that $\fV(P^k_i, \optV_k) \leq 2\fV(P^k_i, V_k - \optV_k) + 2\fV(P^k_i, V_k)$ and thus
\begin{align*}
    \sumkp\sum_{i=1}^{I_k} |(P^k_i - \P^k_i)V_k| \leq \sumkp\sumi \rbr{\sqrt{\frac{\fV(P^k_i, V_k)\iota'}{\N^k_i}} + \sqrt{\frac{\aS_{L(1+\epsilon)}\fV(P^k_i, V_k-\optV_k)\iota'}{\N^k_i}} + \frac{L\aS_{L(1+\epsilon)}\iota'}{\N^k_i}}.
\end{align*}
Then, by Cauchy-Schwarz inequality, with the same probability and \pref{lem:sum N},
\begin{align*}
    \sumkp\sum_{i=1}^{I_k} |(P^k_i - \P^k_i)V_k|
    \lesssim \sqrt{\aS_{L(1+\epsilon)}A\sumkp\sumi  \fV(P^k_i, V_k)\iota''} + \sqrt{{\aS_{L(1+\epsilon)}}^2A\sumkp\sumi \fV(P^k_i, V_k-\optV_k)\iota''} + L{\aS_{L(1+\epsilon)}}^2A\iota'',
\end{align*}
where $\iota'' = O(\log(SALC_{K'}/\delta)\log(C_{K'}))$. Now by \pref{lem:sum dV.easy}, with probability at least $1-2\delta$,
\begin{align*}
    \sumkp\sumi \fV(P^k_i,V_k^\star - V_k) \lesssim L\sumkp\sumi |(P^k_i-\P^k_i)V_k| + L\sqrt{\aS_{L(1+\epsilon)}A\sumkp\sumi \fV(P^k_i, V_k)\iota'} + L^2{\aS_{L(1+\epsilon)}}^{2}A\iota',
\end{align*}
where $\iota'$ was defined above. Let $Z_K := \sumkp\sum_{i=1}^{I_k} |(P^k_i - \P^k_i)V_k|$. Plugging this into the previous inequality, using $\sqrt{xy} \leq x + y$ and $\iota'\leq\iota''$, we get
\begin{align*}
    Z_{K'}
    \lesssim 
    \sqrt{{\aS_{L(1+\epsilon)}}^2AL \iota''Z_{K'}}
     +
    \sqrt{\aS_{L(1+\epsilon)}A\sumkp\sumi \fV(P^k_i, V_k)\iota''} + L{\aS_{L(1+\epsilon)}}^2A\iota''.
\end{align*}
Solving thi quadratic inequality for $Z_{K'}$, we conclude with
\begin{align*}
    \sumkp\sum_{i=1}^{I_k} |(P^k_i - \P^k_i)V_k|
    \lesssim 
    \sqrt{\aS_{L(1+\epsilon)}A\sumkp\sumi \fV(P^k_i, V_k)\iota''} + L{\aS_{L(1+\epsilon)}}^2A\iota''.
\end{align*}

\paragraph{Third term}

By the expressions of $\tilP_i^k$ and $\P_i^k$ (cf. \pref{alg:VISGO}) and \pref{lem:sum N},
\begin{align}\label{eq:regret-term3}
    \sumkp\sumi(\P^k_i - \tilP^k_i)V_k\one_k \leq \sumkp\sumi\one_k \frac{(\P_i+\Ind_g)V_k}{\N_i^k+1} \lesssim L\aS_{L(1+\epsilon)}A \log(C_{K'}).
\end{align}

\paragraph{Fourth and fifth term}

By \pref{lem:sum b} and \pref{lem:sum eps}, with probability at least $1-\delta$,
    \begin{align}\label{eq:regret-term4}
        \sumkp\sumi (b^k_i+\epsilon_k)\one_k \lesssim \sqrt{\aS_{L(1+\epsilon)}A\sumkp\sumi \fV(P^k_i, V_k)\iota'} + L{\aS_{L(1+\epsilon)}}^{1.5}A\iota'.
    \end{align}

\paragraph{Combining all terms}

Note that all the derived bounds can be absorbed into the one of the second term. Plugging everything back to our initial expression of the regret,
\begin{align*}
    \sumkp\rbr{I_k - V_k(s_0)}\one_k 
    &\lesssim \sqrt{\aS_{L(1+\epsilon)}A\sumkp\sumi\fV(P^k_i, V_k)\iota''}  + L{\aS_{L(1+\epsilon)}}^2A\iota''
    \\ &\lesssim \sqrt{L\aS_{L(1+\epsilon)}AC_{K'}\iota''} + L{\aS_{L(1+\epsilon)}}^2A\iota''. \tag{\pref{lem:sum var Vk}}
\end{align*}
    Note that $\iota'' \lesssim \log(SAL/\delta) (\log C_{K'})^2$. Now assuming $\one_k=1$ for all $k$, we can solve an inequality to find $C_{K'}$. First, using that $\log(x) \leq x^\alpha/\alpha$ for any $x,\alpha > 0$ together with the derived regret bound, we can find the crude bound on $C_{K'}$,
    \begin{align*}
        C_{K'} \lesssim \left(\sumkp V_k(s_0) + L{\acalS_{L(1+\epsilon)}}^2A\log(SAL/\delta) \right)^4 \leq \left(K'L + L{\acalS_{L(1+\epsilon)}}^2A\log(SAL/\delta) \right)^4.
    \end{align*}
    This implies that $\iota'' \lesssim (\log K')^2 \log(SAL/\delta)^3$. Plugging this into the regret bound, we get a quadratic inequality in $C_{K'}$. Solving it yields
    \begin{align*}
        C_{K'} \lesssim \sumkp V_k(s_0) + L{\aS_{L(1+\epsilon)}}^2A(\log K')^2 \log(SAL/\delta)^3 \leq LK' + L{\aS_{L(1+\epsilon)}}^2A(\log K')^2 \log(SAL/\delta)^3.
    \end{align*}
    Plugging this back into the regret bound gives the stated bound. Throughout the proof we used following events with the corresponding probabilities:
    \begin{itemize}
        \item \pref{lem:calK id.easy}: $1-5\delta$
        \item \pref{lem:calK.easy}: $1-\delta$
        \item Event $E$ in this proof: $1-\delta$
        \item \pref{lem:anytime freedman}: $1-\delta$
        \item \pref{lem:dPV}: $1-\delta$
        \item \pref{lem:sum b}: $1-\delta$
        \item \pref{lem:sum dV.easy}: $1-2\delta$
        \item \pref{lem:sum var Vk}: $1-2\delta$
    \end{itemize}
    A union bound concludes the proof.
\end{proof}

\subsection{Auxiliary results for policy evaluation}

\begin{lemma}
    \label{lem:sum dV.easy}
    With probability at least $1-2\delta$, for any $K'\in[K]$, if 1) $\norm{V_k}_{\infty}=\bigo{L}$ for any $k\in[K']$, and 2) $V_k(s)\leq\optV_k(s)$ for any $k\in[K']$ and $s\in\calS$, then
    \begin{align*}
        \sumkp\sumi \fV(P^k_i,V_k^\star - V_k) \lesssim L\sumkp\sumi |(P^k_i-\P^k_i)V_k| + L\sqrt{\aS_{L(1+\epsilon)}A\sumkp\sumi \fV(P^k_i, V_k)\iota'} + L^2{\aS_{L(1+\epsilon)}}^{2}A\iota',
    \end{align*}
    where $\iota' = O(\log(SALC_{K'}/\delta))$.
\end{lemma}
\begin{proof}

    First note that, by Condition 1) and 2), for any $s\in\calS$, $V_k^\star(s) - V_k(s) \geq 0$ and $V_k^\star(s) - V_k(s) \leq O(L)$. Thus, by \pref{lem:sum var}, with probability at least $1-\delta$, 
    \begin{align*}
        \sumkp\sumi \fV(P^k_i,V_k^\star - V_k) \lesssim \underbrace{\sumkp (V_k^\star(s^k_{I_k+1}) - V_k(s^k_{I_k+1}))^2}_{(a)} + \underbrace{\sumkp\sumi \rbr{(\optV_k(s^k_i) - V_k(s^k_i))^2 - (P^k_i(\optV_k-V_k))^2}}_{(b)}  + L^2\iota,
    \end{align*}
    where $\iota = O(\log(LC_{K'}/\delta))$.

    \paragraph{Bounding (a)}

   Note that, since $\optV_k(g_k)=V_k(g_k)=0$, we must have $(a) \leq \sumkp\Ind\{s^k_{I_k+1}\neq g\}$. Since the event $\{s^k_{I_k+1}\neq g\}$ happens only in skip rounds, it must be that $(a) \lesssim \aS_{L(1+\epsilon)}A$. 
    
    \paragraph{Bounding (b)}

    Using that $V_k(s)\leq\optV_k(s)$ for all $s\in\calS$ (Condition 2), $(a+b)(a-b)_+$ for $a,b\geq 0$,
    \begin{align*}
        \sumkp\sumi \rbr{(\optV_k(s^k_i) - V_k(s^k_i))^2 - (P^k_i(\optV_k-V_k))^2}
        &\lesssim L\sumkp\sumi (\optV_k(s^k_i) - V_k(s^k_i) - P^k_i\optV_k + P^k_iV_k)_+ 
        \\
        &\lesssim L\sumkp\sumi (1 + P^k_iV_k - V_k(s^k_i))_+,
    \end{align*}
    where in the second inequality we used $\optV_k(s^k_i)\leq 1 + P^k_i\optV_k$ by definition of $\optV_k$. Since, for all $i,k$, $V_k(s^k_i) \geq 1 + \tilP_i^k V_k - b_i^k - \epsilon_k$ (cf. \pref{alg:VISGO}), we also have 
    \begin{align*}
        \sumkp\sumi \rbr{(\optV_k(s^k_i) - V_k(s^k_i))^2 - (P^k_i(\optV_k-V_k))^2}
        &\lesssim L\sumkp\sumi ((P^k_i - \tilP_i^k)V_k + b_i^k + \epsilon_k)_+
        \\ & = L\sumkp\sumi ((P^k_i-\P^k_i)V_k + (\P^k_i - \tilP_i^k)V_k + b_i^k + \epsilon_k)_+
        \\ & \leq L\sumkp\sumi (|(P^k_i-\P^k_i)V_k| + |(\P^k_i - \tilP_i^k)V_k| + b_i^k + \epsilon_k)
    \end{align*}

    All terms but the first one are bounded in \eqref{eq:regret-term3} and \eqref{eq:regret-term4}, which gives the following bound on (b) holding with probability at least $1-2\delta$,
    \begin{align*}
        (b) \lesssim L\sumkp\sumi |(P^k_i-\P^k_i)V_k| + L\sqrt{\aS_{L(1+\epsilon)}A\sumkp\sumi \fV(P^k_i, V_k)\iota'} + L^2{\aS_{L(1+\epsilon)}}^{2}A\iota',
    \end{align*}
    where $\iota' = O(\log(SALC_{K'}/\delta))$. Combining the bounds on (a) and (b) concludes the proof.
\end{proof}

\begin{lemma}
    \label{lem:bound r.easy}
    Assume that for any sequence of indicators $\calI=\{\one_k\}_k$ such that $\one_k\in\calF_{k-1}$, we have $R_{K',\calI}\lesssim c_1\sqrt{K'}\log^{p}(K')+c_2\log^{p}(K')$ and $C_{K'} \lesssim c_3 K' + \log^{p}(K')c_4$ for any $K'\in[K]$, where $c_1\geq L$ and $c_4 \gtrsim \aS_{L(1+\epsilon)}A/\epsilon$.
    Then, the total number rounds $r_{\tot}$ with at least one episode is of order
    \begin{align*}
        \frac{c_1^2}{L^2} \log^{2p}\left(\frac{c_1c_4}{\epsilon}\right) + \left(\frac{c_2\epsilon}{L} + \aS_{L(1+\epsilon)}A + \frac{c_1}{L}\sqrt{\aS_{L(1+\epsilon)}A}\right)\log^{p}\left(\frac{c_1c_2c_4}{\epsilon}\aS_{L(1+\epsilon)}A\right).
    \end{align*}
    Moreover, $C_K \lesssim \frac{c_3 r_{\tot}}{\epsilon^2} + c_4\log^p(r_\tot/\epsilon)$ with probability at least $1-4\delta$.
\end{lemma}
\begin{proof}
	Denote by $\bV_r$, $\barpi_r$ and $\bar{g}_r$ the values of $V_{\calK,\gstar}$, $\pi_{\gstar}$, and $\gstar$ used for policy evaluation in round $r$ respectively.
	For any $R'\geq 1$, let $K'$ be the total number of episodes in the first $R'$ rounds.
    Denote by $r'_{\tot}$ the total number of rounds with at least one episode and $r_f$ the number of failure rounds within the first $K'$ episodes.
    The number of success rounds is at most $\aS_{L(1+\epsilon)}$ by \pref{lem:calK.easy} (which holds with probability $1-\delta$), and the number of skip rounds is at most $\bigo{\aS_{L(1+\epsilon)}A \log(C_{K'})}$ since we have a skip round only when the total number of steps or the  number of visits of some state-action pair in $\calK\times\calA$ is doubled.
    Therefore, $r'_{\tot}\lesssim r_f + \aS_{L(1+\epsilon)}A \log(C_{K'}) \lesssim r_f + \aS_{L(1+\epsilon)}A \log(K') + \aS_{L(1+\epsilon)}A \log(c_4)$, where the last inequality is by assumption on $C_{K'}$.

    Define $\calW=\{r: V^{\barpi_r}_{\bar{g}_r}(s_0)> \bV_r(s_0)\}$.
    Note that $\calW$ includes all failure rounds with probability at least $1-\delta$. This is because, for any round $r\geq 1$ in which $V^{\barpi_r}_{g_r}(s_0)\leq \bV_r(s_0)$ and the skip round condition is not triggered, by \pref{lem:V pi mean} and the value of $\lambda$ in \pref{alg:LOGSSD} in round $r$, we have $\hattau\leq \bV_r(s_0) + \epsilon L/2$ with probability at least $1-\frac{\delta}{2r^2}$. This implies that a success round is triggered. A union bound over all rounds proves that all failure rounds are indeed included in $\calW=\{r: V^{\barpi_r}_{\bar{g}_r}(s_0)> \bV_r(s_0)\}$ with probability at least $1-\delta$.

    Define $\calI=\{\one_k\}_k$ such that $\one_k=\Ind\{r\in\calW\}\in\calF_{k-1}$ for any episode $k$ in round $r$, the regret within these rounds satisfies 
    \begin{align*}
        R_{K,\calI} &\lesssim \left(\frac{c_1}{\epsilon}\sqrt{r_f + \aS_{L(1+\epsilon)}A \log(K') + \aS_{L(1+\epsilon)}A \log(c_4)} + c_2 \right) \log^{p}(K')
        \\ & \lesssim \left(\frac{c_1}{\epsilon}\sqrt{r_f + \aS_{L(1+\epsilon)}A \log(r_f/\epsilon) + \aS_{L(1+\epsilon)}A \log(c_4)} + c_2 \right)\left( \log(r_f/\epsilon) + \log(c_4) \right)^p
    \end{align*}
    by $K = r'_{\tot} \lambda \lesssim \frac{r'_{\tot}}{\epsilon^2}$ (since $\lambda \lesssim 1/\epsilon^2$) and $\log(K') \lesssim \log(r_f/\epsilon) + \log(\aS_{L(1+\epsilon)}A/\epsilon) \lesssim \log(r_f/\epsilon) + \log(c_4)$ by assumption on $c_4$. This shows that if we bound $r'_{\tot}$ we can also control $C_{K'}$.
    
    Now we build a lower bound to $R_{K',\calI}$.
    For each failure round $r$, let $C$ be the total number of steps within this round and $m$ the number of episodes within this round.
    By definition, the regret within this round satisfies $C-m\bV_r(s_0) \geq C-\lambda \bV_r(s_0)=\lambda(\hattau-\bV_r(s_0))>\frac{\lambda\epsilon L}{2}=\lowo{L/\epsilon}$ (since $C/\lambda = \hattau >\bV_r(s_0) + \epsilon L/2$ in a failure round).

    %
    For any round $r\geq 1$, let $m$ be its number of episodes and $C$ be the total number of steps. By \pref{lem:V pi dev}, $mV^{\bpi_r}_{\bar{g}_r}(s_0) \leq C + L\sqrt{m}\ln^2\frac{mLr}{\delta}$ with probability at least $1-\frac{\delta}{2r^2}$. By a union bound, this holds simultaneously across all rounds with probability at least $1-\delta$. Then, with such probability, for each success and skip round $r$ in $\calW$,
    \begin{align*}
        \sum_{j=u_r}^{u'_r}\rbr{I_j - \bV_r(s_0)} \geq \sum_{j=u_r}^{u'_r-1}I_j - m V^{\bpi_r}_{\bar{g}_r}(s_0) - L \gtrsim -L\sqrt{\lambda}\log^2(\frac{\lambda r L}{\delta}) \gtrsim -\frac{L}{\epsilon},
    \end{align*}
    where $\{u_r,\ldots,u'_r\}$ are the episodes in round $r$, and we lower bound the regret in the last episode by $\lowo{-L}$ since the last trajectory in a skipped round is truncated. Note that the first inequality holds since $r\in\calW$.

    Since there are at most $\bigo{\aS_{L(1+\epsilon)}A \log(C_{K'})} = \bigo{\aS_{L(1+\epsilon)}A (\log(r_f/\epsilon) + \log(c_4))}$ of these rounds, we have
    \begin{align*}
        \frac{Lr_f}{\epsilon} &- \frac{L\aS_{L(1+\epsilon)}A (\log(r_f/\epsilon) + \log(c_4))}{\epsilon} \lesssim R_{K',\calI}
        \\ & \lesssim \left(\frac{c_1}{\epsilon}\sqrt{r_f + \aS_{L(1+\epsilon)}A \log(r_f/\epsilon) + \aS_{L(1+\epsilon)}A \log(c_4)} + c_2 \right)\left( \log(r_f/\epsilon) + \log(c_4) \right)^p.
    \end{align*}
    This implies,
    \begin{align*}
        r_f &\lesssim 
        \left(\frac{c_1}{L}\sqrt{r_f} + \frac{c_2\epsilon}{L} + \aS_{L(1+\epsilon)}A + \frac{c_1}{L}\sqrt{\aS_{L(1+\epsilon)}A}\right) (\log(r_f/\epsilon) + \log(c_4))^{p}.
        \\ &\lesssim \left(\underbrace{\frac{c_1}{L}}_{:= a}\sqrt{r_f} + \underbrace{\frac{c_2\epsilon}{L} + \aS_{L(1+\epsilon)}A + \frac{c_1}{L}\sqrt{\aS_{L(1+\epsilon)}A}}_{:= b} \right)  \log(r_f \underbrace{c_4 /\epsilon}_{:= c})^{p}.
    \end{align*}
    By Lemma 28 of \cite{chen2021improved}, $a,b,c$ as defined above, 
    \begin{align*}
        r_f \lesssim \frac{c_1^2}{L^2} \log^{2p}\left(\frac{c_1c_4}{\epsilon}\right) + \left(\frac{c_2\epsilon}{L} + \aS_{L(1+\epsilon)}A + \frac{c_1}{L}\sqrt{\aS_{L(1+\epsilon)}A}\right)\log^{p}\left(\frac{c_1c_2c_4}{\epsilon}\aS_{L(1+\epsilon)}A\right).
    \end{align*}
    The proof is concluded by $r'_\tot \lesssim r_f + \aS_{L(1+\epsilon)}A \log(r_f/\epsilon) + \aS_{L(1+\epsilon)}A \log(c_4)$ as showed above and setting $K'=K$ (that is, $r'_{\tot}=r_{\tot}$).
\end{proof}

\subsection{Proof of \pref{thm:sd} and \pref{thm:sd id}}
\label{app:sd and id}

We restate and prove the two theorems together.

\begin{theorem}[Unified statement of \pref{thm:sd} and \pref{thm:sd id}]
    \label{thm:sd.easy}
    With probability at least $1-23\delta$, after collecting $N_{\tot}$ samples, \pref{alg:LOGSSD} outputs $\calK$ and $\{\tilpi_g\}_{g\in\calK}$ such that $\acalS_L \subseteq \calK\subseteq\acalS_{L(1+\epsilon)}$ and $V^{\tilpi_g}_g(s_0)\leq L(1+\epsilon)$ for all $g\in\calK$, where
    \begin{itemize}
        \item $N_{\tot} = \bigO{\frac{\aS_{L(1+\epsilon)}\Gamma_{L(1+\epsilon)}AL}{\epsilon^2}\iota  + \frac{{\aS_{L(1+\epsilon)}}^2AL}{\epsilon}\iota + L^3 {\aS_{L(1+\epsilon)}}^2A\iota}$ in the general case;
        \item $N_{\tot} = \bigO{\frac{\aS_{L(1+\epsilon)}AL}{\epsilon^2}\iota  + \frac{{\aS_{L(1+\epsilon)}}^2AL}{\epsilon}\iota + L^3 {\aS_{L(1+\epsilon)}}^2A\iota}$ with \pref{assum:id}.
    \end{itemize}
    Here $\iota = \log^{8}\left(\frac{SAL}{\epsilon\delta}\right)$.
\end{theorem}
\begin{proof}
    By \pref{lem:calK.easy} and \pref{lem:update calK.easy}, with probability $1-4\delta$, the output $\calK$ and $\{\tilpi_g\}_{g\in\calK}$ clearly satisfy the first statement.

    Let us bound the sample complexity.
    Each round can be classified into one of the following cases: 1) expansion of the sets (\pref{line:goal condition.easy} is true), and 2) policy evaluation is performed (from \pref{line:PE.easy}, so \pref{line:goal condition.easy} is false). Note that the sample complexity of case 2 is given by $C_K$. We shall bound it later.

    In case 1), the algorithm terminates or at least one state is added into $\calK$.
    Thus, the number of rounds satisfying case 1) in each trial is at most $1+\aS_{L(1+\epsilon)}$ by \pref{lem:calK.easy}. In a round satisfying case 1), if the algorithm terminates, then no samples are collected.
    Otherwise, \pref{line:compute calU'.easy} and \pref{line:fill N.easy} are executed. Take any round $r$ in which this happens and denote by $\calK_r$ the set $\calK$ at the end of round $r$. Note that \pref{line:fill N.easy} collects at most $O(L^2|\calK_r|\ln(Sr/\delta))$ for each $s\in\calK_r$ and $a\in\calA$, while \pref{line:compute calU'.easy} collects $O(L\log(SALr/\delta))$ samples from each state $s\in\calK_r$ and $a\in\calA$, so the total number of samples collected from each $s\in\calK_r$ and $a\in\calA$ is at most $n_r = O(L^2|\calK_r|\ln(SALr/\delta))$.

    Since, by \pref{lem:calK.easy}, at any round $r$, $\|V^{\tilde\pi_g}_{g}\|_{\infty} \leq 4L$ for each $g\in\calK_r$, by \pref{lem:hitting}, with probability $1-\delta'$  it
    takes no more than $8L\log(2/\delta')$ steps to reach the goal state $g$ following $\tilde\pi_g$.
    Therefore, by setting $\delta'=\frac{\delta}{2r^2|\calK_r||\calA|n_r}$, with probability $1-\frac{\delta}{2r^2}$, all trajectories in round $r$ reach the goal within $8L\log(2/\delta')$ steps. Then, by a union bound over all rounds, with probability at least $1-\delta$, the total sample complexity is $\tilo{L^3 |\calK_r|^2|\calA|\log^2(SALr/\delta)}$ at any round $r$.

    Note that, among these samples, only $\tilo{L |\calK_r||\calA|\log^2(SALr/\delta)}$ cumulate over rounds. This is because the sampling of \pref{line:fill N.easy} is performed only if the current counters are below the sampling requirement. Since the number of rounds in case 1) is at most $1+\aS_{L(1+\epsilon)}$ and the total number of rounds $R$ performed by the algorithm satisfies $R \leq r_\tot + \aS_{L(1+\epsilon)} + 1$ (by summing the rounds in both cases) and $|\calK_r| \leq \aS_{L(1+\epsilon)}$ by \pref{lem:calK.easy}, we have that \pref{line:fill N.easy} contributes to at most $\tilo{L {\aS_{L(1+\epsilon)}}^2A\log^2(SALr_\tot/\delta)}$ sample complexity and the total sample complexity of Case 1) is thus $\tilo{L^3 {\aS_{L(1+\epsilon)}}^2A\log^2(SALr_\tot/\delta)}$.

    We now conclude the sample complexity proof depending on whether \pref{assum:id} is considered or not.

    \paragraph{Without \pref{assum:id}}

    Plugging the regret bound of \pref{lem:regret.easy} into \pref{lem:bound r.easy}, using $p=2$, $c_1 = L \log(SAL/\delta)^2 \sqrt{\aS_{L(1+\epsilon)}\Gamma_{L(1+\epsilon)}A}$, $c_2 = L{\aS_{L(1+\epsilon)}}^2A \log(SAL/\delta)^3$, $c_3 = L$, $c_4 = L{\aS_{L(1+\epsilon)}}^2A \log(SAL/\delta)^3 / \epsilon$,
    \begin{align*}
        r_\tot &\lesssim \left( \log(SAL/\delta)^4 \aS_{L(1+\epsilon)}\Gamma_{L(1+\epsilon)}A  + {\aS_{L(1+\epsilon)}}^2A \log(SAL/\delta)^3 \epsilon + \log(SAL/\delta)^2 \aS_{L(1+\epsilon)}\sqrt{\Gamma_{L(1+\epsilon)}}A \right) \log^{4}\left(\frac{SAL}{\epsilon}\right)
        \\ &\lesssim \left(\aS_{L(1+\epsilon)}\Gamma_{L(1+\epsilon)}A  + {\aS_{L(1+\epsilon)}}^2A \epsilon \right) \log^{8}\left(\frac{SAL}{\epsilon\delta}\right)
    \end{align*}
    and
    \begin{align*}
        C_K &\lesssim \frac{L}{\epsilon^2}\left(\aS_{L(1+\epsilon)}\Gamma_{L(1+\epsilon)}A  + {\aS_{L(1+\epsilon)}}^2A \epsilon \right) \log^{8}\left(\frac{SAL}{\epsilon\delta}\right) + \frac{L{\aS_{L(1+\epsilon)}}^2A}{\epsilon}  \log^{5}\left(\frac{SAL}{\epsilon\delta}\right),
        \\ &\lesssim  \left(\frac{\aS_{L(1+\epsilon)}\Gamma_{L(1+\epsilon)}AL}{\epsilon^2}  + \frac{{\aS_{L(1+\epsilon)}}^2AL}{\epsilon} \right) \log^{8}\left(\frac{SAL}{\epsilon\delta}\right).
    \end{align*}
    Thus, the total sample complexity of the algorithm (which is given by $C_K$ plus the sample complexity of case 1) is
    \begin{align*}
        \left(\frac{\aS_{L(1+\epsilon)}\Gamma_{L(1+\epsilon)}AL}{\epsilon^2}  + \frac{{\aS_{L(1+\epsilon)}}^2AL}{\epsilon} + L^3 {\aS_{L(1+\epsilon)}}^2|\calA|\right) \log^{8}\left(\frac{SAL}{\epsilon\delta}\right).
    \end{align*}

    \paragraph{With \pref{assum:id}}

    Plugging the regret bound of \pref{lem:regret-improved.easy} into \pref{lem:bound r.easy}, using $p=2$, $c_1 = L \log(SAL/\delta)^2 \sqrt{\aS_{L(1+\epsilon)}A}$, $c_2 = L{\aS_{L(1+\epsilon)}}^2A \log(SAL/\delta)^3$, $c_3 = L$, $c_4 = L{\aS_{L(1+\epsilon)}}^2A \log(SAL/\delta)^3 / \epsilon$,
    \begin{align*}
        r_\tot &\lesssim \left( \log(SAL/\delta)^4 \aS_{L(1+\epsilon)}A  + {\aS_{L(1+\epsilon)}}^2A \log(SAL/\delta)^3 \epsilon + \log(SAL/\delta)^2 \aS_{L(1+\epsilon)}\sqrt{\Gamma_{L(1+\epsilon)}}A \right) \log^{4}\left(\frac{SAL}{\epsilon}\right)
        \\ &\lesssim \left(\aS_{L(1+\epsilon)}A  + {\aS_{L(1+\epsilon)}}^2A \epsilon \right) \log^{8}\left(\frac{SAL}{\epsilon\delta}\right)
    \end{align*}
    and
    \begin{align*}
        C_K &\lesssim \frac{L}{\epsilon^2}\left(\aS_{L(1+\epsilon)}A  + {\aS_{L(1+\epsilon)}}^2A \epsilon \right) \log^{8}\left(\frac{SAL}{\epsilon\delta}\right) + \frac{L{\aS_{L(1+\epsilon)}}^2A}{\epsilon}  \log^{5}\left(\frac{SAL}{\epsilon\delta}\right),
        \\ &\lesssim  \left(\frac{\aS_{L(1+\epsilon)}AL}{\epsilon^2}  + \frac{{\aS_{L(1+\epsilon)}}^2AL}{\epsilon} \right) \log^{8}\left(\frac{SAL}{\epsilon\delta}\right).
    \end{align*}
    Thus, the total sample complexity of the algorithm (which is given by $C_K$ plus the sample complexity of case 1) is
    \begin{align*}
        \left(\frac{\aS_{L(1+\epsilon)}AL}{\epsilon^2}  + \frac{{\aS_{L(1+\epsilon)}}^2AL}{\epsilon} + L^3 {\aS_{L(1+\epsilon)}}^2|\calA|\right) \log^{8}\left(\frac{SAL}{\epsilon\delta}\right).
    \end{align*}
    A union bound over the events of adopted lemmas (\pref{lem:calK.easy}, \pref{lem:update calK.easy}, Lemma 6 of \cite{rosenberg2020adversarial}, \pref{lem:bound r.easy}, and \pref{lem:regret.easy} without \pref{assum:id} or \pref{lem:regret-improved.easy} with \pref{assum:id}) yields the result with probability at least $1-23\delta$.
\end{proof}

\clearpage

\section{Analysis of \pref{alg:SD}}
\label{app:logsa}


\begin{algorithm2e*}[t]
    \caption{Improved Layer-Aware State Discovery (\algop)}
    \label{alg:SD}
    \SetKwProg{proc}{Procedure}{}{}
    \SetKwFunction{add}{ComputeU}
    \DontPrintSemicolon
    \LinesNumbered
    \KwIn{$L\geq1$, $\epsilon\in(0, 1]$, and $\delta\in(0, 1)$.}
    Let $\tau\leftarrow 1$, $\frakN=\{2^j\}_{j\geq 0}$, $z\leftarrow 2$.\label{line:trial}\;
    \While{True}{\label{line:size.improved}
        Let $\calK\leftarrow \varnothing, \calU \leftarrow \varnothing$, $\calK'\leftarrow \{s_0\}$, $\Pi_{\calK} = \{\tilpi_{s_0}\text{ a random policy}\}$, $\N(\cdot, \cdot)\leftarrow 0, \N(\cdot,\cdot,\cdot) \leftarrow 0$, $\nmin\leftarrow 1$, $k\leftarrow 0$.\;
        \For{round $r=1,\ldots$}{\label{line:round.improved}
            \lIf{$|\calK\cup\calK'| \geq z$}{$z\leftarrow 2|\calK\cup\calK'|$, $\tau\overset{+}{\leftarrow}1$, and return to \pref{line:size.improved}.}\label{line:z.improved}
            $\epsilon_{\VI}\leftarrow 1/\max\{16, \sum_{s,a}\N(s,a)\}$.\;
            Let $\gstar=\argmin_{g\in\calU}\big\{V_{\calK,g}(s_0)\big\}$ where $(Q_{\calK,g}, V_{\calK,g}, \pi_g)=\VISGO(\calK, g, \epsilon_{\VI}, \N, \frac{\delta}{4\tau^2z^4AL})$ (see \pref{alg:VISGO}).\label{line:compute V.improved}\;
            \uIf{$\gstar$ does not exist or $V_{\calK,\gstar}(s_0)>L$}{\label{line:goal condition.improved}
                \tcc{Expand or Terminate}
                \lIf{$\calK'=\varnothing$}{\textbf{return} $\calK$ and $\Pi_{\calK}$.}\label{line:terminate.improved}
                Set $\calK\leftarrow\calK\cup\calK'$, $\calK'=\varnothing, \calU=\varnothing$.
                \;
                $\calU \leftarrow $\add{$\calK$, $\Pi_{\calK}$, $\frac{\delta}{4\tau^2r^2}$}.\label{line:add}
            }
            \uElseIf{$\rtest(\Pi_{\calK}, \pi_{\gstar}, \gstar, \frac{\delta}{4(\tau r)^2})=$ \textbf{False} (see \pref{alg:rtest})}{\label{line:rtest.improved}
                $\nmin\leftarrow 2\nmin$.\;
                $(\N,\_) \leftarrow \fillc(\calK, \Pi_{\calK},\N,\nmin)$ (see \pref{alg:fillc}).
            }\Else{
                \tcc{Policy evaluation}
                Let $\hattau\leftarrow 0$, $\lambda\leftarrow N_{\dev}(32L, \frac{\epsilon}{256}, \frac{\delta}{2 r^2})\lesssim \frac{1}{\epsilon^2}\ln^4(\frac{Lr}{\epsilon\delta})$ (defined in \pref{lem:V pi mean}).\label{line:PE.improved}
                
                \For{$j=1,\ldots,\lambda$}{\label{line:episode.improved}
                    $k\overset{+}{\leftarrow}1$, $i\leftarrow 1$, and reset to $s^k_1\leftarrow s_0$ by taking action $\reset$.\;
                    \While{$s^k_i\neq \gstar$}{
                        Take $a^k_i=\pi_{\gstar}(s^k_i)$, and transits to $s^k_{i+1}$.
                        Increase $\N(s^k_i, a^k_i)$, $\N(s^k_i, a^k_i, s^k_{i+1})$, and $i$ by $1$.\;
                        \lIf{$\sum_{s,a}\N(s,a)\in\frakN$ or ($s^k_i\in\calK$ and $\N(s^k_i, a^k_i)\in\frakN$)}{ return to \pref{line:round.improved} (skip round).\label{line:skip.improved}}
                        Set $\hattau\overset{+}{\leftarrow} \frac{c(s^k_i, a^k_i)}{\lambda}$.
                    }
                    \lIf{$\hattau>V_{\calK,\gstar}(s_0) + \epsilon L/2$}{ return to \pref{line:round.improved} (failure round).  \label{line:failure.improved}}
                }
                $\calK'\leftarrow\calK'\cup\{\gstar\}$, $\calU\leftarrow\calU\setminus\{\gstar\}$, $\Pi_{\calK}=\Pi_{\calK} \cup \{\tilpi_{\gstar}:=\pi_{\gstar}\}$ 
                (success round).
            }
        }
    }
    \proc{\add{$\calX$, $\Pi_{\calX}$, $\delta$}}{
        $(\_,\calU')\leftarrow\fillc(\calX,\Pi_{\calX}, 0, 2L\ln\frac{4LA|\calX|}{\delta})$ (see~\pref{alg:fillc}). \label{line:compute calU'.improved}\;
        $(\N', \_) \leftarrow \fillc(\calX,\Pi_{\calX}, 0, N_1(|\calX|, \frac{\delta}{4|\calU'|}))$ where $N_1$ is defined in \pref{lem:bounded error fresh}.\label{line:gather N'.improved}\;
        Let $\calU=\{g\in\calU': V'_{\calX,g}(s_0)\leq L\}$ where $(\_,V'_{\calX,g},\pi'_g)=\VISGO(\calX,g,\frac{1}{16},\N',\frac{\delta}{4|\calU'|})$.\label{line:filter calU'.improved}\;
        \Return{$\calU$}
    }
\end{algorithm2e*}

\paragraph{Notation} Define $\calN(\calK, p)=\{s'\notin\calK: P(s'|s,a)\geq p\text{ for some }(s, a)\in\calK\times\calA \}$.
Fix any ordering $\acalO_L=(s_1,\ldots,s_n)$ of states in $\acalS_L$ such that it can be partitioned into $J$ (defined in \pref{lem:SL.operator}) segments with states in the $j$-th segment belonging to $\calKstar_j\setminus\calKstar_{j-1}$.
For an arbitrary $z\in\fN_+$, also define $\{\calKstar_{z,j}\}_j$, such that $\calKstar_{z,j}=\calKstar_j$ when $|\calKstar_j|< z$, and $\calKstar_{z,j}=\{s_1,\ldots,s_z\}$ when $|\calKstar_j|\geq z$.
Therefore, $\calKstar_{z,z}=(s_1,\ldots,s_{z})$ (the first $z$ elements of $\acalO_L$) or $\acalS_L$ by definition.
Define $\calUstar_z=\rS{\calKstar_{z,z}}{2L}$.
Clearly, $\calUstar_z\subseteq\{s'\in\calS: \exists s\in \calKstar_{z,z}, a\in\calA, P(s'|s, a)\geq \frac{1}{2L}\}$, and thus $|\calUstar_z|\leq 2zAL$.

\subsection{\pfref{thm:sd id improved}}
\label{app:sd id improved}
\begin{proof}
    We condition on the events of \pref{lem:output}, \pref{lem:calU}, and \pref{lem:calK}, which happen with probability at least $1-7\delta$.
    By the events of \pref{lem:calK} and \pref{lem:output}, the output $\calK$ and $\Pi_{\calK}=\{\tilpi_g\}_{g\in\calK}$ clearly satisfy the statement.
    By \pref{lem:bound z}, there are at most $\bigo{\ln\aS_{L(1+\epsilon)}}$ trials.
    Thus, it suffices to bound the number of samples used in each trial.
    Define $\iota=\ln\frac{L\aS_{L(1+\epsilon)}A}{\delta\epsilon}$.
    Each round in a trial can be classified into one of the following cases: 1) \pref{line:goal condition.improved} is verified, 2) \pref{line:rtest.improved} is verified, and 3) policy evaluation is performed (\pref{line:PE.improved}).
    In case 1), the algorithm terminates or at least one state is added into $\calK$ (\pref{line:terminate.improved}).
    Thus, the number of rounds satisfying case 1) in each trial is at most $1+\aS_{L(1+\epsilon)}$ by \pref{lem:calK}.
    By \pref{lem:bound nmin} and the update rule of $n_{\min}$, the number of rounds satisfying case 2) is of order $\bigo{\ln(L\aS_{L(1+\epsilon)})}$.
    By \pref{lem:bound r} and \pref{lem:regret}, with probability at least $1-8\delta$, the total number of rounds satisfying case 3) is of order $\bigo{ \aS_{L(1+\epsilon)}\Gamma_{L(1+\epsilon)}A\iota^6 + {\aS_{L(1+\epsilon)}}^2A\epsilon\iota^6}$.
    So the total number of rounds in each trial is at most $\bigo{ \aS_{L(1+\epsilon)}\Gamma_{L(1+\epsilon)}A\iota^6 + {\aS_{L(1+\epsilon)}}^2A\epsilon\iota^6 }$.
    
    Now it suffices to bound the number of samples collected in a round satisfying each of the cases above in a trial.
    In a round satisfying case 1), if the algorithm terminates, then no samples are collected.
    Otherwise, \add is called, and $\bigo{L^3{\aS_{L(1+\epsilon)}}^2A\iota^2}$ samples are collected with probability at least $1-\delta$ by \pref{lem:calU each} (\pref{line:add} and a union bound over all trials and rounds).
    In a round satisfying case 2), with probability at least $1-4\delta$, $\bigo{L\aS_{L(1+\epsilon)}\iota^2}$ samples are collected in performing \rtest by \pref{lem:output} and \pref{lem:rtest} (\pref{line:rtest.improved} and a union bound over all trials and rounds), and $\bigo{L^3{\aS_{L(1+\epsilon)}}^2A\iota^2}$ samples are collected in executing \fillc by \pref{lem:bound nmin} and \pref{lem:sc fillc}.
    In a round satisfying case 3), with probability at leat $1-\delta$, $\bigo{L\aS_{L(1+\epsilon)}\iota^2}$ samples are collected in performing \rtest similar to that of case 2), and $\bigo{L\iota^5/\epsilon^2}$ samples are collected by the value of $\lambda$ and the fact that $\pi_{\gstar}$ passes the test in \pref{line:rtest.improved} (\pref{lem:rtest} and a union bound over all trials and rounds).
    Thus, the total sample complexity is
    \begin{align*}
        &\sum_{i=1}^3\text{[\#rounds satisfying case $i$}]\cdot[\text{\#samples in a round satisfying case $i$}]\cdot\iota\\
        &\lesssim \aS_{L(1+\epsilon)}\cdot L^3{\aS_{L(1+\epsilon)}}^2A\iota^3 + L^3{\aS_{L(1+\epsilon)}}^2A\iota^4 + (\aS_{L(1+\epsilon)}\Gamma_{L(1+\epsilon)}A + {\aS_{L(1+\epsilon)}}^2A\epsilon)\cdot\rbr{\frac{L}{\epsilon^2}+L\aS_{L(1+\epsilon)}}\iota^{12}\\
        &\lesssim \rbr{\frac{L\aS_{L(1+\epsilon)}\Gamma_{L(1+\epsilon)}A}{\epsilon^2} + \frac{L{\aS_{L(1+\epsilon)}}^2A\epsilon}{\epsilon} + L^3{\aS_{L(1+\epsilon)}}^3A}\iota^{12}.
    \end{align*}
    This completes the proof.
    To prove the second statement, we can simply follow the proof above except that we involve \pref{lem:regret-improved} instead of \pref{lem:regret} when applying \pref{lem:bound r} to bound the total number of rounds satisfying case 3), which holds with probability at least $1-20\delta$.
\end{proof}

\begin{lemma}
    \label{lem:bound nmin}
    With probability at least $1-2\delta$, if the events of \pref{lem:calK} and \pref{lem:bcalU} hold, then $\nmin\lesssim L^2\aS_{L(1+\epsilon)}\ln\aS_{L(1+\epsilon)}$ throughout the execution of \pref{alg:SD}.
\end{lemma}
\begin{proof}
    In any trial $\tau$, when $\nmin\geq N^{\rightarrow}_0(\frac{\delta}{4\tau^2z^4AL})$ (defined in \pref{lem:bounded error}), we have with probability at least $1-\frac{\delta}{2\tau^2}$, $\norm{V^{\pi_{\gstar}}_{\gstar}}_{\infty}\leq 2\norm{V_{\calK,\gstar}}_{\infty}\leq 2(1 + V_{\calK,\gstar}(s_0)) \leq 4L$ in any round such that $\gstar$ exists and $V_{\calK,\gstar}(s_0)\leq L$.
    This implies that with probability at least $1-\sum_{r=1}^{\infty}\frac{\delta}{4\tau^2r^2}\geq 1-\frac{\delta}{2\tau^2}$, the condition of \pref{line:rtest.improved} is always false by \pref{lem:rtest}, and the value of $\nmin$ will no longer change within this trial.
    A union bound over all trials and noting the update rule of $\nmin$ completes the proof.
\end{proof}

\begin{lemma}  
    \label{lem:bound z}
    Conditioned on the event of \pref{lem:calK}, we have $z\leq 2\aS_{L(1+\epsilon)}+2$ and $\tau\leq 1 + \log_2(\aS_{L(1+\epsilon)}+1)$ throughout the execution of \pref{alg:SD}.
\end{lemma}
\begin{proof}
    The proof of \pref{lem:calK} shows that $s\notin\acalS_{L(1+\epsilon)}$ will never be added to $\calK'$, which implies $\calK\cup\calK'\subseteq\acalS_{L(1+\epsilon)}$ throughtout the execution of \pref{alg:SD}.
    Thus, when $z\geq \aS_{L(1+\epsilon)}+1$, $z$ will not be updated again.
    Then, the statement is proved by the update rule of $z$ and $\tau$.
\end{proof}

\subsection{Lemmas for Policy Evaluation}
\paragraph{Notation} Let $g_k$, $\calK_k$, $V_k$, $Q_k$, $\optV_k$ be the values of $\gstar$, $\calK$, $V_{\calK,\gstar}$, $Q_{\calK,\gstar}$, and $\optV_{\calK,\gstar}$ in episode $k$ respectively.
Denote by $I_k$ the number of steps in episode $k$.
Note that $I_k<\infty$ with probability $1$ by \pref{line:skip.improved}, and $s^k_{I_k+1}\neq g_k$ only when a skip round is triggered in episode $k$.
Denote by $\calF_k$ the $\sigma$-algebra of events up to episode $k$.
Define $K$ as the total number of episodes throughout the execution of \pref{alg:SD}.
For any sequence of indicators $\calI=\{\one_k\}_k$ and $K'\leq K$, define $R_{K',\calI}=\sumkp(I_k - V_k(s_0))\one_k$ and $C_{K'}=\sumkp I_k$.
Define $P^k_i=P_{s^k_i,a^k_i}$.
In episode $k$, when $s^k_i\in\calK$, denote by $\P^k_i$, $\tilP^k_i$, $\N^k_i$, $b^k_i$ the values of $\P_{s^k_i,a^k_i}$, $\tilP_{s^k_i, a^k_i}$, $n^+(s^k_i, a^k_i)$, and $b^{(l)}(s^k_i, a^k_i)$, where $\P$, $n^+$, $b^{(l)}$ are used in \pref{alg:VISGO} to compute $V_k$ and $l$ is the final value of $i$ in \pref{alg:VISGO};
when $s^k_i\notin \calK$, define $\P^k_i=\Ind_{s_0}$, $\N^k_i=\infty$, and $b^k_i=0$.
Also define $\epsilon_k$ as the value of $\epsilon_{\VI}$ used in \pref{alg:VISGO} to compute $V_k$.

\begin{lemma}
    \label{lem:regret}
    With probability at least $1-5\delta$, if the events of \pref{lem:calK} and \pref{lem:bcalU} hold, then in any trial, for any sequence of indicators $\calI=\{\one_k\}_k$ with $\one_k\in\calF_{k-1}$, we have $R_{K',\calI} \lesssim \sqrt{\aS_{L(1+\epsilon)}\Gamma_{L(1+\epsilon)}AL^2K'\iota} + L{\aS_{L(1+\epsilon)}}^2A\iota$ for any $K'\leq K$, where $\iota=\ln^2\frac{L\aS_{L(1+\epsilon)}AK'}{\delta}$.
\end{lemma}
\begin{proof}
    Note that by \pref{lem:def Vk},
    \begin{align*}
        \sumkp\rbr{I_k - V_k(s_0)}\one_k &\leq \sumkp\sum_{i=1}^{I_k}\rbr{1 + V_k(s^k_{i+1}) - V_k(s^k_i)}\one_k\\
        &\lesssim \sumkp\sum_{i=1}^{I_k}\rbr{(\Ind_{s^k_{i+1}} - P^k_i)V_k + (P^k_i - \P^k_i)V_k + b^k_i + \epsilon_k}\one_k.
    \end{align*}
    We bound the sums above separately.
    By \pref{lem:anytime freedman} and $\norm{V_k}_{\infty}\leq 2L$, with probability at least $1-\delta$,
    \begin{align*}
        \sumkp\sumi (\Ind_{s^k_{i+1}} - P^k_i)V_k\one_k \lesssim \sqrt{\sumkp\sumi\fV(P^k_i, V_k)\ln\frac{LC_{K'}}{\delta}} + L\ln\frac{LC_{K'}}{\delta}.
    \end{align*}
    By \pref{lem:dPV}, $\calK_k\in\acalS_{L(1+\epsilon)}$ (\pref{lem:calK}), $g_k\in \bcalU\setminus\calK_k$ (\pref{lem:bcalU}), Cauchy-Schwarz inequality, and \pref{lem:sum N}, with probability at least $1-\delta$,
    \begin{align*}
        \sumkp\sumi(P^k_i - \P^k_i)V_k\one_k &\lesssim \sumkp\sumi\one_k\sqrt{\frac{\Gamma_{L(1+\epsilon)}\fV(P^k_i, V_k)\iota'}{\N^k_i}} + \frac{L\aS_{L(1+\epsilon)}\iota'}{\N^k_i} \tag{$\N^k_i=\infty$ when $s^k_i\notin\calK_k$ and $\iota'=\ln\frac{\aS_{L(1+\epsilon)}AC_{K'}}{\delta}$}\\
        &\lesssim \sqrt{\aS_{L(1+\epsilon)}\Gamma_{L(1+\epsilon)}A\sumkp\sumi\fV(P^k_i, V_k)\iota'} + L{\aS_{L(1+\epsilon)}}^2A\iota'. \tag{$\iota'=\ln\frac{\aS_{L(1+\epsilon)}AC_{K'}}{\delta}\ln(C_{K'})$}
    \end{align*}
    Finally, by \pref{lem:sum b} and \pref{lem:sum eps}, with probability at least $1-\delta$,
    \begin{align*}
        \sumkp\sumi (b^k_i+\epsilon_k)\one_k \lesssim \sqrt{\aS_{L(1+\epsilon)}A\sumkp\sumi\fV(P^k_i, V_k)\iota'} + L{\aS_{L(1+\epsilon)}}^{1.5}A\iota'. \tag{$\iota'=\ln\frac{\aS_{L(1+\epsilon)}AC_{K'}}{\delta}$}
    \end{align*}
    Plugging these back, we have with probability at least $1-2\delta$,
    \begin{align}
        \sumkp\rbr{I_k - V_k(s_0)}\one_k &\lesssim \sqrt{\aS_{L(1+\epsilon)}\Gamma_{L(1+\epsilon)}A\sumkp\sum_{i=1}^{I_k}\fV(P^k_i, V_k)\iota'} + L{\aS_{L(1+\epsilon)}}^2A\iota'\notag\\
        &\lesssim \sqrt{\aS_{L(1+\epsilon)}\Gamma_{L(1+\epsilon)}ALC_{K'}\iota'} + L{\aS_{L(1+\epsilon)}}^2A\iota', \label{eq:var}
    \end{align}
    where $\iota'=\ln\frac{L\aS_{L(1+\epsilon)}AC_{K'}}{\delta}\ln(C_{K'})$ and in the last step we apply \pref{lem:sum var Vk}.
    Now assuming $\one_k=1$ for all $k$ and solving a ``quadratic'' inequality (\pref{lem:quad log}) w.r.t.~$C_{K'}$, we have
    \begin{align*}
        C_{K'}\lesssim \sumkp V_k(s_0) + L{\aS_{L(1+\epsilon)}}^2A\iota' \lesssim LK' + L{\aS_{L(1+\epsilon)}}^2A\iota'. \tag{$\iota'=\ln^2\frac{L\aS_{L(1+\epsilon)}AK'}{\delta}$}
    \end{align*}
    Plugging this back to \pref{eq:var} completes the proof.
\end{proof}

\begin{lemma}
    \label{lem:regret-improved}
    With \pref{assum:id}, with probability at least $1-12\delta$, if the events of \pref{lem:calU}, \pref{lem:bound z}, \pref{lem:calK id}, and \pref{lem:calU id} hold, in any trial, for any sequence of indicators $\calI=\{\one_k\}_k$ with $\one_k\in\calF_{k-1}$, we have $R_{K',\calI} \lesssim L\sqrt{\aS_{L(1+\epsilon)}AK'\iota} + L{\aS_{L(1+\epsilon)}}^2A\iota$ for any $K'\leq K$, where $\iota=\ln^2\frac{L\aS_{L(1+\epsilon)}AK'}{\delta}$.
\end{lemma}
\begin{proof}
    Note that with \pref{assum:id} and by \pref{lem:calK id} and \pref{lem:calU id}, in any episode, $\calK=\calKstar_j$ for some $j\leq z$ and $\gstar\in\calUstar_z$.
    Thus by \pref{lem:anytime bernstein} and a union bound over $\{\optV_{\calKstar_{z,j},g}\}_{j\in[z], g\in\calUstar_z}$ and $(s, a)\in\acalS_{L(1+\epsilon)}\times\calA$, we have with probability at least $1-\delta$,
    \begin{equation}
        (P^k_i -\P^k_i)\optV_k\lesssim \sqrt{\frac{\fV(P^k_i, \optV_k)\iota'}{\N^k_i}} + \frac{L\iota'}{\N^k_i},\label{eq:dP optV}
    \end{equation}
    where $\iota'=\ln\frac{\aS_{L(1+\epsilon)}AC_{K'}}{\delta}$.
    Thus, with probability at least $1-\delta$,
    \begin{align*}
        &\sumkp\rbr{I_k - V_k(s_0)}\one_k \leq \sumkp\sumi\rbr{1 + V_k(s^k_{i+1}) - V_k(s^k_i)}\one_k\\
        &\lesssim \sumkp\sumi\rbr{(\Ind_{s^k_{i+1}} - P^k_i)V_k + (P^k_i - \P^k_i)V_k + b^k_i + \epsilon_k}\one_k \tag{\pref{lem:def Vk}}\\
        &\lesssim \sqrt{\sumkp\sumi\fV(P^k_i, V_k)\ln\frac{LC_{K'}}{\delta}} + L\ln\frac{LC_{K'}}{\delta} + \sumkp\sumi\rbr{(P^k_i-\P^k_i)\optV_k\one_k + (P^k_i-\P^k_i)(V_k-\optV_k)\one_k + b^k_i },
    \end{align*}
    where the last step is by \pref{lem:anytime freedman} and \pref{lem:sum eps}.
    Note that by \pref{eq:dP optV}, \pref{lem:dPV}, and $\norm{\optV_k}_{\infty}\leq 2L+1$ by \pref{lem:calU} and \pref{lem:init bound}, with probability at least $1-2\delta$,
    \begin{align*}
        &\sumkp\sumi\rbr{(P^k_i-\P^k_i)\optV_k\one_k + (P^k_i-\P^k_i)(V_k-\optV_k)\one_k + b^k_i }\\
        &\lesssim \sumkp\sumi\rbr{\sqrt{\frac{\fV(P^k_i, \optV_k)\iota'}{\N^k_i}} + \sqrt{\frac{\Gamma_{L(1+\epsilon)}\fV(P^k_i, V_k-\optV_k)\iota'}{\N^k_i}} + \frac{L\Gamma_{L(1+\epsilon)}\iota'}{\N^k_i} + b^k_i} \tag{$\iota'=\ln\frac{\aS_{L(1+\epsilon)}AC_{K'}}{\delta}$}\\
        &\lesssim \sqrt{\aS_{L(1+\epsilon)}A\sumkp\sumi\fV(P^k_i, V_k)\iota'} + \sqrt{{\aS_{L(1+\epsilon)}}^2A\sumkp\sumi\fV(P^k_i, V_k-\optV_k)\iota'} + L{\aS_{L(1+\epsilon)}}^2A\iota' \tag{$\iota'=\ln^2\frac{\aS_{L(1+\epsilon)}AC_{K'}}{\delta}$}.
	\end{align*}
    where the last step is by \pref{lem:sum N}, Cauchy-Schwarz inequality, $\var[X+Y]\leq 2(\var[X]+\var[Y])$, and \pref{lem:sum b}.
    Plugging this back, applying \pref{lem:sum dV} with \pref{lem:opt} on $\{\optV_{\calKstar_j,g}\}_{j\in[z], g\in\calUstar_z\setminus\calKstar_j}$ (where all $\optV_k$ lies in), \pref{lem:calK id}, and \pref{lem:calU id}, and then applying AM-GM inequality, we have with probability at least $1-8\delta$,
    \begin{align*}
        \sumkp\rbr{I_k - V_k(s_0)}\one_k &\lesssim \sqrt{\aS_{L(1+\epsilon)}A\sumkp\sumi\fV(P^k_i, V_k)\iota'} + L{\aS_{L(1+\epsilon)}}^2A\iota'\\
        &\lesssim \sqrt{L\aS_{L(1+\epsilon)}AC_{K'}\iota'} + L{\aS_{L(1+\epsilon)}}^2A\iota', \tag{\pref{lem:sum var Vk}}
    \end{align*}
    where $\iota'=\ln^2\frac{L\aS_{L(1+\epsilon)}AC_{K'}}{\delta}$.
    Now assuming $\one_k=1$ for all $k$ and solving a ``quadratic'' inequality (\pref{lem:quad log}), we have
    \begin{align*}
        C_{K'} \lesssim \sumkp V_k(s_0) + L{\aS_{L(1+\epsilon)}}^2A\iota' \leq LK' + L{\aS_{L(1+\epsilon)}}^2A\iota'. \tag{$\iota'=\ln^2\frac{L\aS_{L(1+\epsilon)}AK'}{\delta}$}
    \end{align*}
    Plugging this back completes the proof.
\end{proof}


\begin{lemma}
    \label{lem:bound r}
    In any trial, with probability at least $1-8\delta$, if for any sequence of indicators $\calI=\{\one_k\}_k$ with $\one_k\in\calF_{k-1}$, we have $R_{K',\calI}\lesssim c_1\sqrt{K'\ln^p(c_3K')}+c_2\ln^p(c_3K')$ with $c_1,c_2\geq 1$, and $c_3=\frac{L\aS_{L(1+\epsilon)}A}{\delta}$ for any $K'\leq K$, then the total number of rounds with at least one epsiode is of order $\bigo{\aS_{L(1+\epsilon)}A\iota^4 + \frac{c_1^2}{L^2}\iota^{p+4} + c_2\epsilon\iota^p/L}$, where $\iota=\ln\frac{c_1c_2c_3}{\epsilon\delta}$.
\end{lemma}
\begin{proof}
	For any $R'\geq 1$, let $K'$ be the total number of episodes in the first $R'$ rounds.
    Denote by $r_{\tot}$ the total number of rounds with at least one episode, and $r_f$ the number of failure rounds in the first $R'$ rounds.
	First note that by $V_k(s_0)\leq L$ (\pref{line:goal condition.improved}) and setting $\one_k=1$, the regret guarantee in the assumption gives $C_{K'}\lesssim LK' + c_1\sqrt{K'\ln^p(c_3K')} + c_2\ln^p(c_3K')$, which gives $\ln(C_{K'})\lesssim \ln(c_1c_2c_3K')$.
	Moreover, $K'\lesssim \frac{r_{\tot}}{\epsilon^2}\ln^4\frac{Lr_{\tot}}{\epsilon\delta}$ by the value of $\lambda$ in each round (\pref{line:PE.improved}).
	Thus, $\ln(C_{K'})\lesssim \ln\frac{c_1c_2c_3r_{\tot}}{\epsilon\delta}$ and $\ln(c_3K')\lesssim \ln\frac{c_1c_2c_3r_{\tot}}{\epsilon\delta}$.
	
    Fixed a trial, denote by $\bV_r$, $\barpi_r$ and $\bar{g}_r$ the values of $V_{\calK,\gstar}$, $\pi_{\gstar}$, and $\gstar$ used for policy evaluation in round $r$ respectively.
    It is clear that in the first $R'$ rounds, the number of success round is at most $\aS_{L(1+\epsilon)}$ by \pref{lem:calK}, and the number of skip rounds is at most $\bigo{\aS_{L(1+\epsilon)}A\ln(C_{K'})}$ since we have a skip round only when the total number of steps or the  number of visits of some state-action pair in $\calK\times\calA$ is doubled.
Therefore, $r_{\tot}\lesssim r_f + \aS_{L(1+\epsilon)}A\ln(C_{K'}) \lesssim r_f + \aS_{L(1+\epsilon)}A\ln\frac{c_1c_2c_3r_{\tot}}{\epsilon\delta}$.
	By \pref{lem:quad log}, we have $r_{\tot} \lesssim r_f + \aS_{L(1+\epsilon)}A\ln\frac{c_1c_2c_3r_f}{\epsilon\delta}$.
	Now define $\iota(r_f)=\ln\frac{c_1c_2c_3r_f}{\epsilon\delta}$.
    It remains to bound $r_f$.
    Define $\calW=\{r: V^{\barpi_r}_{\bar{g}_r}(s_0)>\bV_r(s_0)\}$.
    Note that $\calW$ includes all failure rounds with probability at least $1-\delta$, since when $V^{\barpi_r}_{\bar{g}_r}(s_0)\leq \bV_r(s_0)$ and $r$ is not a skip round, by \pref{lem:V pi mean} and the value of $\lambda$ in round $r$ we have $\hattau\leq \bV_r(s_0) + \epsilon L/2$ in round $r$.
    Define $\calI=\{\one_k\}_k$ such that $\one_k=\Ind\{r\in\calW\}\in\calF_{k-1}$ for any episode $k$ in round $r$, the regret within these rounds satisfies $R_{K',\calI}\lesssim \frac{c_1}{\epsilon}\sqrt{r_f+\aS_{L(1+\epsilon)}A} + c_2$.
    \begin{align*}
    		R_{K',\calI} &\lesssim c_1\sqrt{K'\ln^p(c_3K')}+c_2\ln^p(c_3K') \lesssim \frac{c_1}{\epsilon}\sqrt{(r_f + \aS_{L(1+\epsilon)}A\iota(r_f))\iota(r_f)^{p+4}} + c_2\iota(r_f)^p\\
		&\lesssim \frac{c_1}{\epsilon}\sqrt{r_f\iota(r_f)^{p+4}} + \frac{c_1^2\iota(r_f)^{p+4}}{L\epsilon} + \frac{L\aS_{L(1+\epsilon)}A\iota(r_f)}{\epsilon}+c_2\iota(r_f)^p. \tag{AM-GM inequality}
    \end{align*}
    
    For each failure round $r$, let $C$ be the total cost within this round and $m$ the number of episodes within this round.
    By definition, regret within this round satisfies $C-mV_{\calK,\gstar}(s_0) \geq C-\lambda V_{\calK,\gstar}(s_0)=\lambda(\hattau-V_{\calK,\gstar}(s_0))>\frac{\lambda\epsilon L}{2}=\lowo{L/\epsilon}$.
    By \pref{lem:V pi dev}, with probability at least $1-\delta$, for each success and skip round $r$ in $\calW$ ($V^{\barpi_r}_{g_r}(s_0)>\bV_r(s_0)$), 
    \begin{align*}
        \sum_{j=u_r}^{u'_r}\rbr{I_j - \bV_r(s_0)} \gtrsim \sum_{j=u_r}^{u'_r-1}\rbr{I_j - V^{\bpi_r}_{\bar{g}_r}(s_0)} - L \gtrsim -L\sqrt{\lambda}\ln^2\frac{L\lambda}{\delta} = -\frac{L}{\epsilon}\ln^4\frac{Lr}{\delta\epsilon},
    \end{align*}
    where $\{u_r,\ldots,u'_r\}$ are the episodes in round $r$, and we lower bound the regret in the last episode by $\lowo{-L}$ since the last trajectory in a skipped round is truncated.
    Since there are at most $\tilo{\aS_{L(1+\epsilon)}A}$ these rounds, we have
    \begin{align*}
        \frac{Lr_f}{\epsilon} - \frac{L\aS_{L(1+\epsilon)}A}{\epsilon}\ln^4\frac{Lr_f}{\epsilon\delta} \lesssim \frac{c_1}{\epsilon}\sqrt{r_f\iota(r_f)^{p+4}} + \frac{c_1^2\iota(r_f)^{p+4}}{L\epsilon} + \frac{L\aS_{L(1+\epsilon)}A\iota(r_f)}{\epsilon}+c_2\iota(r_f)^p.
    \end{align*}
    This gives $r_f \lesssim \aS_{L(1+\epsilon)}A\iota^4 + \frac{c_1^2}{L^2}\iota^{p+4} + c_2\epsilon\iota^p/L$, where $\iota=\ln\frac{c_1c_2c_3}{\epsilon\delta}$.
    Setting $R'$ to be the total number rounds completes the proof.
\end{proof}

\begin{lemma}
    \label{lem:output}
    With probability at least $1-2\delta$, throughout the execution of \pref{alg:SD}, for each $g\in\calK$ we have $V^{\tilpi_g}_g(s_0)\leq L(1+\epsilon)$ and $\norm{V^{\tilpi_g}_g}_{\infty}\leq 32L$.
\end{lemma}
\begin{proof}
    By \pref{lem:rtest} and a union bound over all trials and rounds, with probability at least $1-\delta$, we have $\norm{V^{\tilpi_g}_g}_{\infty}\leq 32L$ for each $g\in\calK$, since $\tilpi_g$ passes the test in \pref{line:rtest.improved}.
    Moreover, by the definition of success round, value of $\lambda$, and \pref{lem:V pi mean}, with probability at least $1-\delta$, for each $g\in\calK$, in the round that $g$ is added to $\calK$, we have $V^{\tilpi_g}_g(s_0)=V^{\pi_g}_g(s_0)\leq \hattau + \frac{L\epsilon}{2} \leq V_{\calK,g}(s_0) + L\epsilon \leq L(1+\epsilon)$.
\end{proof}
\subsection{Properties of the sets built by \pref{alg:SD}}

\begin{lemma}[Restricted Optimism]
    \label{lem:V calK}
    With probability at least $1-\delta$ over the randomness of \pref{alg:SD}, at any trial and any round, after executing \pref{line:compute V.improved}, if $\calKstar_{z,j}\subseteq\calK$ for some $j\in[z]$, then
    $V_{\calK,g}(s) \leq \optV_{\calKstar_{z,j},g}(s)$ for any $s\in\calS$ and $g\in\calKstar_{z,j+1}\setminus\calK$. 
\end{lemma}
\begin{proof}
    For any $\tau'\geq 1$, $z'\geq 1$, $j\in[z']$, $g\in\calKstar_{z',j+1}\setminus\calKstar_{z',j}$, by \pref{lem:opt} and $\norm{\optV_{\calKstar_{z',j},g}}_{\infty}\leq L+1$ (\pref{lem:init bound}), with probability at least $1-\frac{\delta}{4(z')^4(\tau')^2}$, for any status of $\N$ and $\xi>0$, we have $V(s)\leq \optV_{\calKstar_{z',j},g}(s)$ for all $s\in\calS$ where $(\_,V,\_)=\VISGO(\calKstar_{z',j}, g, \xi, \N, \frac{\delta}{4(\tau')^2(z')^4AL})$.
    By a union bound, all events above hold simultaneously with probability at least $1-\delta$.

    At any trial $\tau$ and round, after executing \pref{line:compute V.improved}, let $(\_,V_{\calKstar_{z,j}, g}, \_)=\VISGO(\calKstar_{z,j}, g, \epsilon_{\VI},\N,\delta')$ (no need to compute explicitly) for any $j\in[z]$, and $g\in\calKstar_{z,j+1}\setminus\calKstar_{z,j}$, where $\delta'=\frac{\delta}{4\tau^2z^4AL}$.
    The union bound above implies that $V_{\calKstar_{z,j}, g}(s) \leq \optV_{\calKstar_{z,j}, g}(s)$ for any $s\in\calS$.
    Then by \pref{lem:subset opt}, we also have $V_{\calK,g}(s) \leq \optV_{\calKstar_{z,j},g}(s)$ if $\calKstar_{z,j}\subseteq\calK$ ($V_{\calK,g}$ is computed in \pref{line:compute V.improved}).
\end{proof}

\begin{lemma}
	 \label{lem:update calK}
	For a given trial $(\tau, z)$, denote by $\calK_r$ the set $\calK$ at the end of each round $r$. With probability at least $1-2\delta$, for any $j \geq 1$ and round $r \geq 1$ in any trial in which $\calK_r$ is updated or returned (i.e., \pref{line:compute calU'.easy} is executed) and $\calK_{r-1} \supseteq \calKstar_{j}$, we have $\calK^\star_{j+1} \subseteq \calK_r$.
\end{lemma}
\begin{proof}
	In this lemma we denote by $\calU_r$ the value of $\calU$ \textit{at the end of} round $r$.
    Define the event $E := \{ \text{for any trial, }\forall r\geq 1 \text{ in which $\calK_r$ is updated}: \calT_L(\calK_r) \setminus \calK_r \subseteq \calU_r\}$. By \pref{lem:calU}, it holds with probability at least $1-\delta$. Let us carry out the proof conditioned on $E$ holding. 
    
    In any trial, take some round $r$ such that \pref{line:compute calU'.easy} is executed and $\calK_{r-1} \supseteq \calKstar_{j}$. Let $r'<r$ be the last round where $\calK_{r'}$ was updated (and thus $\calU_{r'}$ was created). Note that $\calK_{r'} = \calK_{r-1} \supseteq \calKstar_j$. Then, event $E$ and the definition of the sets $(\calKstar_j)_j$ directly imply that $\calKstar_{j+1} := \calT_L(\calKstar_j) \subseteq \calT_L(\calK_{r'}) \subseteq \calU_{r'} \cup \calK_{r'}$. Since $\calK_r$ can only be formed by adding states in $\calU_{r'}$ to $\calK_{r'}$, and the union of these sets contains $\calKstar_{j+1}$, if $\calKstar_{z,j+1} \not\subseteq \calK_r$, it must be that there exists $g\in\calU_{r-1} \cap \calKstar_{z,j+1}$ s.t. $V_{\calK_{r-1},g}(s_0) > L$. However, \pref{lem:V calK}, which holds with probability $1-\delta$, implies that, at any round $r\geq 1$, if $\calKstar_{j}\subseteq\calK_{r-1}$ (which implies that $z>|\calKstar_j|$ and $\calKstar_j=\calKstar_{z,j}$ by \pref{line:z.improved}), then
    $V_{\calK_{r-1},g}(s_0) \leq \optV_{\calKstar_{j},g}(s_0) \leq L$ for any $g\in\calKstar_{z,j+1}\setminus\calK_{r-1}$. This is a contradiction, which implies that $\calU_{r-1} \cap \calKstar_{z,j+1} = \emptyset$ and, thus, all states in $\calKstar_{z,j+1}$ must have been added to $\calK_r$.
    Moreover, since a new trial is not triggered in round $r$, by \pref{line:z.improved}, we have $z>|\calKstar_{z,j+1}|$ and $\calKstar_{z,j+1}=\calKstar_{j+1}$.
    This completes the proof.
\end{proof}

\begin{lemma}
    \label{lem:calK}
    For a given trial $(\tau, z)$, denote by $\calK_r$ the set $\calK$ at the end of each round $r$ inside the trial.
    With probability at least $1-4\delta$, at any trial $(\tau,z)$, we have $\calK_r \subseteq \acalS_{L(1+\epsilon)}$ for any round $r$, and $\acalS_L\subseteq\calK_r$ if the algorithm terminates at round $r$.
\end{lemma}
\begin{proof}
    Fix any trial $(\tau,z)$. 
    Clearly, $\calK_1\subseteq \acalS_{L(1+\epsilon)}$. 
    To prove the first statement, consider a round $r\geq 1$ and suppose $\calK_r\subseteq \acalS_{L(1+\epsilon)}$. If, in this round, the algorithm selects a goal $\gstar\in \calU\setminus\acalS_{L(1+\epsilon)}$, $\pi_{\gstar}$ passes the test of \pref{line:rtest.improved}, and a skip round is not triggered, then we show that the ``failure test'' in \pref{line:failure.improved} is triggered. 

    Since $\pi_{\gstar}$ passed the test of \pref{line:rtest.improved}, we have $\|V^{\pi_{\gstar}}_{\gstar}\|_\infty \leq 32L$ with probability at least $1-\delta$ by \pref{lem:rtest} and a union bound over all trials and rounds. Combining this with \pref{lem:V pi mean} and the value of $\lambda$ (\pref{line:PE.improved}) (again by a union bound over all trials and rounds), we have $\hattau\geq V^{\pi_{\gstar}}_{\gstar}(s_0) - L\epsilon/2$ with probability at least $1-2\delta$. By assumption on $g^\star$ and since $\pi_{\gstar}$ is restricted on $\calK_r\subseteq\acalS_{L(1+\epsilon)}$, we have $V^{\pi_{\gstar}}_{\gstar}(s_0) \geq V^{\star}_{\calK_r,\gstar}(s_0) \geq V^{\star}_{\acalS_{L(1+\epsilon)},\gstar}(s_0) > L(1+\epsilon)$, which implies that $\hattau\geq L(1+\epsilon/2) \geq V_{\calK_r, \gstar}(s_0) + \epsilon L/2$, where the last inequality is from the goal-selection rule. 
    Therefore, the failure test triggers and $g^\star$ is not added to $\calK'$.
    Overall, any $g\notin\acalS_{L(1+\epsilon)}$ will never be added to $\calK$ or $\calK'$ throughout the execution of \pref{alg:SD}.

    To prove the second statement, let us consider any trial $(\tau,z)$ where the algorithm stops.
    Clearly, $\calKstar_1\subseteq \calK_1$ at the end of round $r=1$ in this last trial.
    Then, if $r$ is the round where the algorithm terminates, and $\calKstar_j\subseteq\calK_{r-1}$ for some $j\geq 1$, we have $\calKstar_{j+1}\subseteq\calK_r$ with probability at least $1-2\delta$ by \pref{lem:update calK}.
    Moreover, since $\calK'=\varnothing$ in round $r$, we have $\calKstar_{j+1}\subseteq\calK_{r-1}=\calK_r$.
    By a recursive application of \pref{lem:update calK}, we have $\calKstar_j\subseteq\calK_r$ for any $j\geq 1$ (note that $\calK'=\varnothing$ at the beginning of round $r$).
    \pref{lem:SL.operator} then implies the statement.
\end{proof}

\begin{lemma}
    \label{lem:bcalU}
    Conditioned on the events of \pref{lem:calU} and \pref{lem:calK}, $\calU\subseteq\bcalU$ at the beginning of any round in any trial.
\end{lemma}
\begin{proof}
    This is clearly true at the beginning of the first round of any trial since $\calU=\varnothing$.
    Then by the events of \pref{lem:calU} and \pref{lem:calK}, $\calU \subseteq \rS{\calK}{2L}\setminus\calK\subseteq\bcalU$ every time after executing \pref{line:add}.
    Moreover, we only remove elements from $\calU$ except when executing \pref{line:add}.
    This completes the proof.
\end{proof}

\begin{lemma}
    \label{lem:calK id}
    Denote by $\calK_r$ the set $\calK$ at the end of each round $r$.
    With \pref{assum:id}, with probability at least $1-8\delta$ over the randomness of \pref{alg:SD}, we have that $\calK_r = \calKstar_j$ for some $j \in [\aS_L]$ at any round $r$ and, $\calK_{r}=\acalS_L$ if the algorithm terminates at round $r$.
\end{lemma}
\begin{proof}
    By \pref{lem:calK}, with probability at least $1-4\delta$, we have $\acalS_L \subseteq \calK \subseteq \acalS_{L(1+\epsilon)}$ if the algorithm terminates.
    By \pref{rem:id}, $\calK = \acalS_L$.
    Thus, it suffices to show that at any trial $\calK=\calKstar_j$ for some $j \leq \aS_L$.

    The algorithm is such that $\calKstar_1 = \calK_1 = \{s_0\}$.
    Suppose at the end of a round $r$ we have that $\calK_r=\calKstar_j$ for some $j\geq 1$.
    By \pref{lem:update calK}, with probability at least $1-2\delta$, if the condition of \pref{line:goal condition.improved} is verified the first time in some round $r'>r$, then we must have $\calKstar_{j+1}\subseteq\calK_{r'}$.
    If we also have $\calK_{r'}\subseteq\calKstar_{j+1}$, then the statement is proved.

    In any round $r$ such that $\calK=\calKstar_j$, $\gstar\in \calU\setminus\calKstar_{j+1}$, $\pi_{\gstar}$ passes the test of \pref{line:rtest.improved}, and a skip round is not triggered, by \pref{lem:V pi mean}, the value of $\lambda$, and \pref{lem:rtest} (applying a union bound over all trials and rounds), we have $\hattau\geq V^{\pi_{\gstar}}_{\gstar}(s_0) - L\epsilon/2$ with probability at least $1-2\delta$.
    By assumption on $g^\star$ and since $\pi_{\gstar}$ is restricted on $\calK\subseteq\calKstar_j$, we have $V^{\pi_{\gstar}}_{\gstar}(s_0) \geq V^{\star}_{\calK,\gstar}(s_0) \geq V^{\star}_{\calKstar_j,\gstar}(s_0) > L(1+\epsilon)$, which implies that $\hattau\geq L(1+\epsilon/2) \geq V_{\calK, \gstar}(s_0) + \epsilon L/2$, where the last inequality is from the goal-selection rule.  Therefore, the failure test triggers and $g^\star$ is not added to $\calK'$ or $\calK$.
    This proves $\calK\subseteq\calKstar_{j+1}$ in round $r'$.
\end{proof}

\begin{lemma}
    \label{lem:calU id}
	With \pref{assum:id}, conditioned on the events of \pref{lem:calU} and \pref{lem:calK id}, in any trial, $\calU\subseteq\calUstar_z$ at the beginning of any round.
\end{lemma}
\begin{proof}
    By \pref{lem:calK id}, in any trial, we have $\calK = \calKstar_j\subseteq\calKstar_{z,z}$ for some $j\leq z$ at the end of any round.
    Then by \pref{lem:calU}, we have $\calU\subseteq \rS{\calK}{2L}\setminus\calK\subseteq\calUstar_z$ every time \pref{line:add} is executed.
\end{proof}

\subsection{Properties of $\calU$} 


Given $\calX$, $\Pi_{\calX}=\{\pi_g\}_{g\in\calX}$ and $\delta$ as input of \add, let $\calD_0$ and $\calD_1$ be the random samples collected respectively in \pref{line:compute calU'.improved} and \pref{line:gather N'.improved}.
Define 
\begin{align*}
    \calE_0(\calD_0) &= \left\{\calN(\calX, \frac{1}{2L})\not\subseteq\calU' \right\},\\
    \calE_1(\calD_0, \calD_1) &= \left\{ \exists g \in \calU', V'_{\calX,g}(s_0) > V^\star_{\calX, g}(s_0)\right\},\\
    \calE_2(\calD_0, \calD_1) &= \left\{ \exists g \in \calU', V^{\pi_g}_g(s) > 2 V'_{\calX,g}(s) \right\}.
\end{align*}
In this section we use $\mathbb{E}$ and $\mathbb{P}$ to denote expectation and probability w.r.t.\ these two random generation processes.

\begin{lemma}
    \label{lem:calU each}
    With any $\calX$, $\{\pi_g \in \Pi(\calX)\}_{g \in \calX}$ such that $\norm{V^{\pi_g}_g}_{\infty} = \bigo{L}$, and $\delta \in (0,1)$ as input, \add ensures
    \[
        \mathbb{P}\rbr{\rS{\calX}{L}\setminus \calX \subseteq \calU \subseteq \rS{\calX}{2L}\setminus\calX } \geq 1-\delta.
    \]
    With the same probability, the sample complexity of \add is bounded by $\bigo{L^3|\calX|^2A\ln^2\frac{L|\calX|A}{\delta}}$.
\end{lemma}
\begin{proof}
    Denote by $\{s_{i,s,a}\}_{i,s,a}$ the set of next state samples collected in \pref{line:compute calU'.improved} for each $(s, a)$.
    Let $\mu = 2L\ln(4LA|\calX|/\delta)$, then
    \begin{align*}
        \mathbb{P}\left( \calE_0(\calD_0)\right) &= P\left(\exists s' \in \calN(\calX, \frac{1}{2L}), \forall (s,a) \in \calX \times \calA, \forall i \in [\mu]: s_{i,s,a} \neq s'\right) \\
        &\leq \sum_{s' \in \calN(\calX, \frac{1}{2L})} P\left(\forall  (s,a) \in \calX \times \calA, \forall i \in [\mu]: s_{i,s,a} \neq s'\right)\\
        & \leq \sum_{s' \in \calN(\calX, \frac{1}{2L})} \prod_{(s,a) \in \calX \times \calA} \prod_{i \in [\mu]} (1-P(s'|s,a)) \leq \sum_{s' \in \calN(\calX, \frac{1}{2L})} \left(1-P(s'|\bar{s},\bar{a})\right)^{\mu}  \tag{$\bar s, \bar a$ such that $P(s'|\bar s, \bar a)\geq\frac{1}{2L}$}\\
        & \leq \sum_{s' \in \calN(\calX, \frac{1}{2L})} \left(1-\frac{1}{2L}\right)^{\mu} \leq \sum_{s' \in \calN(\calX, \frac{1}{2L})}\frac{\delta}{4LA|\calX|} \leq \delta/2. \tag{$|\calN(\calX,\frac{1}{2L})|\leq 2LA|\calX|$}
    \end{align*}
    Let $N_1$ be defined as in \pref{lem:bounded error fresh}. Then, from \pref{lem:opt} and \pref{lem:bounded error fresh}, by using $\delta/(4|\calU'|)$, we have that $\mathbb{P}\left( \calE_1(\calD_0,\calD_1) | \calD_0\right) \leq \delta/4$ and  $\mathbb{P}\left( \calE_2(\calD_0,\calD_1) | \calD_0\right) \leq \delta/4$.
    Then, we can write that
    \begin{align*}
        \mathbb{P}(\calE_0(\calD_0)  \cup \calE_1(\calD_0,\calD_1)\cup \calE_2(\calD_0,\calD_1)) 
        &\leq 
        \mathbb{P}(\calE_0(\calD_0)) +  \mathbb{P}(\calE_1(\calD_0,\calD_1) \cup \calE_2(\calD_0,\calD_1))\\
        &\leq \delta/2 + \sum_{\calD_0} \mathbb{P}(\calD_0) \underbrace{\mathbb{P}(\calE_1(\calD_0,\calD_1)\cup \calE_2(\calD_0,\calD_1) | \calD_0)}_{\leq \delta/2,  \forall \calD_0} = \delta
    \end{align*} 
    We then carry out the proof under event $E = \neg (\calE_1(\calD_0)\cup \calE_1(\calD_0,\calD_1)\cup \calE_2(\calD_0,\calD_1))$ which hold with probability $1-\delta$.

    Since $\pi'_g$ is restricted on $\calX$, we have that $V^\star_{\calX,g}(s_0) \leq V^{\pi'_g}_g(s_0)$ by the definition of optimal policy.
    We have that, for any $g\in\calU$, $\optV_{\calX,g}(s_0)\leq V^{\pi'_g}_g(s_0)\leq 2V'_{\calX,g}(s_0)\leq 2L$ by the definition of $\calU$. This implies that $\calU \subseteq \rS{\calX}{2L}\cap\calU'\subseteq \rS{\calX}{2L}\setminus\calX$ since $\calU' \cap \calX = \emptyset$ by definition.

    Finally, note that, by the definition of $\rS{\calX}{L}$ and the event $\neg \calE_0$, $\rS{\calX}{L}\setminus \calX \subseteq \calN(\calX, \frac{1}{2L}) \subseteq \calU'$ w.h.p. Furthermore, under the event $\neg \calE_1(\calD_0,\calD_1)$, we have that for any $g \in \calU'$, if $V^\star_{\calX,g}(s_0) \leq L$, then $V'_{\calX,g}(s_0) \leq V^\star_{\calX,g}(s_0) \leq L$. Thus, $\rS{\calX}{L}\setminus \calX \subseteq \calU$.

    \paragraph{Sample complexity.} Since $\|V^{\pi_g}_{g}\|_{\infty} = \bigo{L}$, by \pref{lem:sc fillc} with $\bar{n}=\mu$ and $N_1(|\calX|, \frac{\delta}{4|\calU'|})$, with probability at least $1-\delta$, the sample complexity is $\bigo{L|\calX|An'\ln\frac{|\calX|An'}{\delta}}$,
    where $n'=\mu+N_1(|\calX|, \delta/(4|\calU'|)$.
    Given that $N_1(|\calX|, \frac{\delta}{4|\calU'|}) = \bigo{L^2 |\calX| \ln(|\calU'||\calX|/\delta)}$ (see \pref{lem:bounded error fresh}), we have $n'=\bigo{L^2 |\calX| \ln(L|\calX|A/\delta)}$.
    Plugging this back, the sample complexity is $\bigo{L^3|\calX|^2A\ln^2\frac{L|\calX|A}{\delta}}$.    
\end{proof}

\begin{lemma}
    \label{lem:calU}
    With probability at least $1-\delta$ over the randomness of \pref{alg:SD}, at any trial and round, $\rS{\calK}{L}\setminus \calK \subseteq \calU \subseteq \rS{\calK}{2L}\setminus\calK$ after executing \pref{line:add} (if it is executed).
\end{lemma}
\begin{proof}
    This is simply by \pref{lem:calU each} and the choice of confidence level in \pref{line:add} in each trial and round.
\end{proof}

\subsection{\rtest and \fillc} 
Here we show auxiliary algorithms and related lemmas used in \pref{alg:SD}.

\begin{algorithm2e}[t]
    \DontPrintSemicolon
    \caption{\fillc}
    \label{alg:fillc}
    \KwIn{States $\calX$, policies $\Pi = \{\pi_x\}_{x\in\calX}$ such that $\norm{V^{\pi_x}_x}_{\infty} = \bigo{L}$, counters $n$, target value $\bar{n}$.}
    $\Snext\leftarrow \varnothing$.\;
    \For{$(x,a)\in\calX\times\calA$}{
        \While{$n(x, a) < \bar{n}$}{
            Reset to $s_0$ and execute $\pi_x$ until reaching $x$.\;
            Execute action $a$, observe $x'\sim P_{x,a}$, and update $n(x,a,x') \overset{+}{\leftarrow} 1$.\;
            \lIf{$x'\notin \calX$}{$\Snext\leftarrow\Snext\cup\{x'\}$.}
        }
    }
    \Return{$n$ and $\Snext$.}
\end{algorithm2e}

\begin{algorithm2e}[t]
    \DontPrintSemicolon
    \caption{\rtest}
    \label{alg:rtest}
    \KwIn{reaching policy $\{\pi_s\}_{s\in\calX}$, test policy $\pi\in\Pi(\calX)$, goal state $g$, and failure probability $\delta$.}

    Let $n=2^{10}\ln\frac{2|\calX|}{\delta}$.
    
    \For{$s\in\calX$}{
        $i_s\leftarrow 0$.

        \For{$j=1,\ldots,n$}{
            Reset to $s_0$ and execute $\pi_s$ until $s$ is reached.

            Execute $\pi$ until $g$ is reached or $8L$ steps is taken.

            \lIf{$g$ is reached}{$i_s\overset{+}{\leftarrow}1$}
        }

        \lIf{$i_s/n < \frac{7}{16}$}{\Return{\false.} }
    }

    \Return{\true.}
\end{algorithm2e}

\begin{lemma}
    \label{lem:rtest}
    For any $\calX\subseteq\calS$, $\{\pi_g\}_{g\in\calX}$, policy $\overline{\pi}\in\Pi(\calX)$, goal state $g\in\calS$, and $\delta\in(0,1)$, we have
    \begin{align*}
        \mathbb{P}\rbr{\left. \rtest(\calX, \{\pi_g\}_{g\in\calX}, \pi, g, \delta) = \true \right| \norm{V^{\overline{\pi}}_g}_{\infty} \leq 4L} &\geq 1- \delta,\\
        \mathbb{P}\rbr{ \rtest(\calX, \{\pi_g\}_{g\in\calX}, \pi, g, \delta) = \true \implies \norm{V^{\overline{\pi}}_g}_{\infty}\leq 32L } &\geq 1- \delta.
    \end{align*}
    Moreover, if $\norm{V^{\pi_g}_g}_{\infty}=\bigo{L}$ for any $g\in\calX$, then with probability at least $1-\delta$, the sample complexity is $\tilo{L|\calX|\ln^2\frac{|\calX|}{\delta}}$.
\end{lemma}
\begin{proof}
    Let $\{\eta_i\}_{i\in [n]}$ be rollouts of length at most $\bar{l}$ generated running $\overline{\pi}$ from state $s$, and denote by $p_{\bar{l},g}^{\overline{\pi}}(s)$ the probability of reaching the goal $g$ in at most $\bar{l}$ steps by following policy $\overline{\pi}$ starting from $s$. 
    Let $\one(\eta) = 1$ if the goal has been reached in rollout $\eta$, zero otherwise.
    $X_i = \one_g(\eta_i) - p_g^\pi(s)$ is a martingale difference sequence ($|X_i| \leq 1$) and by Azuma's inequality (see \pref{lem:azuma}), setting $n = 2^{10}\ln(\frac{2|\calX|}{\delta})$, we have
    \begin{equation}
        \mathbb{P} \left( \forall s \in \calX,  \frac{1}{n}\abr{\sum_{i=1}^n X_i} \leq \frac{1}{16} \right) \geq 1 - \delta.
        \label{eq:azuma.is}
    \end{equation}
    1) If $\norm{V^{\overline{\pi}}_g}_{\infty} \leq 4L$, by Markov's inequality, $p_{\bar{l},g}^{\overline{\pi}}(s) \geq 1/2$ when $\bar{l}= 8L$.
    This gives $\frac{i_s}{n} = \sum_i \frac{\one_g(\eta_i)}{n} \geq p_g^{\overline{\pi}}(s) - \frac{1}{16}\geq\frac{7}{16}$ for any $s\in\calX$, and thus the algorithm returns \true on termination.

    2) If the output is \true, then $\frac{i_s}{n} \geq \frac{7}{16}$ for all $s \in\calX$. By~\eqref{eq:azuma.is}, we have that $p_g^{\overline{\pi}}(s) \geq \frac{i_s}{n} - \frac{1}{16} \geq \frac{3}{8}$.
    Thus for any $s\in\calX$, $V^{\pi}_g(s)\leq 8L + \frac{5}{8}\norm{V^{\pi}_g}_{\infty}$, which gives $\norm{V^{\pi}_g}_{\infty}\leq 1+8L+\frac{5}{8}\norm{V^{\pi}_g}_{\infty}$ by $\pi\in\Pi(\calX)$.
    This implies $\norm{V^{\pi}_g}_{\infty}\leq 32L$.

    \paragraph{Sample complexity.} 
    If $\norm{V^{\pi_s}_{s}}_{\infty} = \bigo{L}$ for any $s \in \calX$, by \pref{lem:hitting}, with probability $1-\delta$, all trajectories generated by $\pi_s$ for some $s\in\calX$ reaches state $s$ in $\bigo{L\ln(2n|\calX|/\delta)}$ steps.
    Noting that we generate $n$ trajectories for each $s\in\calX$ completes the proof.
\end{proof}

\begin{lemma}
    \label{lem:sc fillc}
    For any $\calX\subseteq\calS$, $\Pi=\{\pi_x\}_{x\in\calX}$, counter $n$, threshold $\bar{n}\geq 1$, and $\delta\in(0, 1)$, with probability at least $1-\delta$, the sample complexity of $\fillc(\calX,\Pi,n,\bar{n})$ is $\bigo{L|\calX|A\bar{n}\ln\frac{|\calX|A\bar{n}}{\delta}}$.
\end{lemma}
\begin{proof}
    For any $x\in\calX$, since $\|V^{\pi_x}_{x}\|_{\infty} = \bigo{L}$, by \pref{lem:hitting}, with probability $1-\delta'$  it takes $\bigo{L\ln(1/\delta')}$ steps to reach the goal state following $\pi_x$ from any $s \in \calX$.
    Therefore, by setting $\delta'=\frac{\delta}{|\calX|A\bar{n}}$, with probability $1-\delta$, all trajectories reach the desired goal state within $\bigo{L\ln(1/\delta')}$ steps.
    Given that there are at most $|\calX|A\bar{n}$ trajectories, with probability at least $1-\delta$, the total sample complexity is
    $\bigo{L|\calX|A\bar{n}\ln\frac{|\calX|A\bar{n}}{\delta}}$.
\end{proof}

\clearpage

\section{Analysis of Policy Consolidation}\label{app:consolidation}



In this section, we bound the sample complexity of \pref{alg:PC}.
\paragraph{Notation} We assume that all episodes lie in one (artificial) trial.
Let $g_k$, $\tset_k$, $V_k$ $\optV_k$ be the values of $\gstar$, $\tset\setminus\{\gstar\}$, $\hatV$, and $\optV_{\tset,\gstar}$ in episode $k$ respectively.
Denote by $I_k$ the number of steps in episode $k$.
Note that $I_k<\infty$ with probability $1$ by \pref{line:skip PC}, and $s^k_{I_k+1}\neq g_k$ only when a skip round is triggered in episode $k$.
Denote by $\calF_k$ the $\sigma$-algebra of events up to episode $k$.
Define $K$ as the total number of episodes throughout the execution of \pref{alg:PC}.
For any $K'\leq K$, define $R_{K'}=\sumkp(I_k - V_k(s_0))$ and $C_{K'}=\sumkp I_k$.
Define $P^k_i=P_{s^k_i,a^k_i}$.
In episode $k$, when $s^k_i\in\calK$, denote by $\P^k_i$, $\tilP^k_i$, $\N^k_i$, $b^k_i$ the values of $\P_{s^k_i,a^k_i}$, $\tilP_{s^k_i, a^k_i}$, $n^+(s^k_i, a^k_i)$, and $b^{(l)}(s^k_i, a^k_i)$, where $\P$, $n^+$, $b^{(l)}$ are used in \pref{alg:VISGO} to compute $V_k$ and $l$ is the final value of $i$ in \pref{alg:VISGO};
when $s^k_i\notin \calK$, define $\P^k_i=\Ind_{s_0}$, $\N^k_i=\infty$, and $b^k_i=0$.
Also define $\epsilon_k$ as the value of $\epsilon_{\VI}$ used in \pref{alg:VISGO} to compute $V_k$.
In this section, $\calK\subseteq\acalS_{L(1+\epsilon)}$ is an input of \pref{alg:PC} and thus does not have randomness.


\begin{proof}[\pfref{thm:PC}]
    By \pref{lem:output PC}, the output policies $\{\tilpi_g\}_g$ clearly satisfies the statement.
    Define $\iota=\ln\left(\frac{L\aS_{L(1+\epsilon)}A}{\delta\epsilon}\right)$.
    It suffices to bound the number of samples collected in \pref{line:nu} and policy evaluation.
    With probability at least $1-\delta$, the number of samples collected in \pref{line:nu} is of order $\bigo{L^3{\aS_{L(1+\epsilon)}}^2A\iota^2}$ by \pref{lem:sc fillc} and \pref{lem:bounded error fresh}.
    With probability at least $1-16\delta$, by \pref{lem:bound C PC} and \pref{lem:reg PC} ($c_1=\sqrt{L\aS_{L(1+\epsilon)}A}$, $c_2=L{\aS_{L(1+\epsilon)}}^2A$, and $p=2$), the number of samples collected in policy evaluation is of order $\tilO{\frac{L\aS_{L(1+\epsilon)}A\iota^{10}}{\epsilon^2}+\frac{L{\aS_{L(1+\epsilon)}}^2A\iota^{10}}{\epsilon}}$.
    Combining all cases completes the proof.
\end{proof}

\begin{lemma}
    \label{lem:bound C PC}
    With probability at least $1-4\delta$, if $R_{K'}\lesssim c_1\sqrt{\sumkp V_k(s_0)\ln^p(c_3K')}+c_2\ln^p(c_3K')$ for any $K'\geq 1$ with $c_1,c_2\geq 1$ and $c_3=\frac{L\aS_{L(1+\epsilon)}A}{\delta}$, then $C_K \lesssim \frac{L\aS_{L(1+\epsilon)}A\iota^8}{\epsilon^2} + \frac{c_1^2\iota^{p+8}}{\epsilon^2} + \frac{c_2\iota^{p+4}}{\epsilon}$, where $\iota=\ln\frac{c_1c_2c_3}{\epsilon\delta}$.
\end{lemma}
\begin{proof}
    For any $R'\geq 1$, let $K'$ be the total number of episodes in the first $R'$ rounds.
    Let $Z_{K'}=\sumkp V_k(s_0)$.
    First note that the regret gives $C_{K'}\lesssim Z_{K'} + c_1\sqrt{Z_{K'}\ln^p(c_3K')} + c_2\ln^p(c_3K')$ and thus $\ln(C_{K'})\lesssim \ln(c_1c_2c_3Z_{K'})$.
    By $K'\lesssim C_{K'}$ and solving a ``quadratic'' inequality (\pref{lem:quad log}), we have $C_{K'}\lesssim Z_{K'} + (c_1^2+c_2)\ln^p(c_1c_2c_3Z_{K'})$.
    Denote by $\bar{g}_r$, $\bV_{r}$, $\bpi_r$ the value of $\gstar$, $\hatV$, and $\hatpi$ in round $r$ respectively.
    For each failure round $r$, let $C$ be the total cost within this round and $m$ the number of episodes within this round.
    By definition, regret within this round satisfies $C-m\bV_r(s_0) \geq C-\lambda \bV_r(s_0)=\lambda(\hattau-\bV_r(s_0))>\frac{\lambda\epsilon \bV_r(s_0)}{2}=\lowo{\bV_r(s_0)/\epsilon}$.
    For each success and skip round $r$, by \pref{lem:opt PC}, \pref{lem:nu}, \pref{lem:V pi dev}, and the value of $\lambda$, we have
    \begin{align*}
        \sum_{j=u_r}^{u'_r}\rbr{I_j - \bV_r(s_0)} \gtrsim \sum_{j=u_r}^{u'_r-1}\rbr{I_j - V^{\bpi_r}_{\bar{g}_r}(s_0)} - L \gtrsim -L\sqrt{\lambda}\ln^2\frac{L\lambda}{\delta} \gtrsim -\frac{L}{\epsilon}\ln^4\frac{Lr}{\delta\epsilon} \gtrsim -\frac{L}{\epsilon}\ln^4\frac{LC_{K'}}{\delta\epsilon},
    \end{align*}
    where $\{u_r,\ldots,u'_r\}$ are the episodes in round $r$, and we lower bound the regret in the last episode by $\lowo{-L}$ since the last trajectory in a skipped round is truncated.
    Denote by $\calR_f$ the total number of failure rounds within the first $R'$ rounds.
    By the assumption in \pref{alg:PC} that $\tset\subseteq\acalS_{L(1+\epsilon)}$, in the first $R'$ rounds, the number of success round is at most $\aS_{L(1+\epsilon)}$ and the number of skip rounds is at most $\bigo{\aS_{L(1+\epsilon)}A\ln(C_{K'})}$.
    Since there are at most $\bigo{\aS_{L(1+\epsilon)}A\ln(C_{K'})}$ these rounds, in each round there are at most $\tilo{\frac{\ln^4\frac{LC_{K'}}{\delta\epsilon}}{\epsilon^2}}$ episodes (\pref{line:PE PC}), and $\bV_r(s_0)\leq 2L$ in any round $r$ by \pref{lem:opt PC}, we have
    \begin{align*}
    		Z_{K'} 
            &\lesssim \frac{\sum_{r\in\calR_f}\bV_r(s_0)\ln^4\frac{LC_{K'}}{\delta\epsilon}}{\epsilon^2} + \frac{L\aS_{L(1+\epsilon)}A\ln^5\frac{c_1c_2c_3Z_{K'}}{\delta\epsilon}}{\epsilon^2} \\
            &\lesssim \frac{\sum_{r\in\calR_f}\bV_r(s_0)\ln^4\frac{c_1c_2c_3Z_{K'}}{\delta\epsilon}}{\epsilon^2} + \frac{L\aS_{L(1+\epsilon)}A\ln^5\frac{c_1c_2c_3Z_{K'}}{\delta\epsilon}}{\epsilon^2}.
    \end{align*}
    By \pref{lem:quad log}, this gives
    \begin{align*}
    		Z_{K'} \lesssim \frac{\sum_{r\in\calR_f}\bV_r(s_0)\ln^4(c_4\sum_{r\in\calR_f}\bV_r(s_0))}{\epsilon^2} + \frac{L\aS_{L(1+\epsilon)}A\ln^5(c_4\sum_{r\in\calR_f}\bV_r(s_0))}{\epsilon},
    \end{align*}
    and $\ln(Z_{K'})\lesssim \ln(\frac{c_1c_2c_3\sum_{r\in\calR_f}\bV_r(s_0)}{\delta\epsilon})\triangleq \ln(c_4\sum_{r\in\calR_f}\bV_r(s_0))$, where $c_4=\frac{c_1c_2c_3}{\delta\epsilon}$. 
    Therefore, the regret upper and lower bound and $\ln(K')\leq \ln(C_{K'})\lesssim \ln(c_1c_2c_3Z_{K'})\lesssim \ln(c_4\sum_{r\in\calR_f}\bV_r(s_0))$ give
    \begin{align*}
        &\frac{\sum_{r\in\calR_f}\bV_r(s_0)}{\epsilon} - \frac{L\aS_{L(1+\epsilon)}A}{\epsilon}\ln^4\frac{LC_{K'}}{\delta\epsilon} \lesssim c_1\sqrt{Z_{K'}\ln^p(c_3K')}+c_2\ln^p(c_3K')\\
        &\lesssim \frac{c_1}{\epsilon}\sqrt{\rbr{\sum_{r\in\calR_f}\bV_r(s_0) + L\aS_{L(1+\epsilon)}A\ln\rbr{c_4\sum_{r\in\calR_f}\bV_r(s_0)}}\ln^{p+4}\rbr{c_4\sum_{r\in\calR_f}\bV_r(s_0)} } + c_2\ln^p\rbr{c_4\sum_{r\in\calR_f}\bV_r(s_0)}.
    \end{align*}
    Applying \pref{lem:quad log} gives $\sum_{r\in\calR_f}\bV_r(s_0) \lesssim L\aS_{L(1+\epsilon)}A\ln^4(c_4) + c_1^2\ln^{p+4}(c_4) + c_2\epsilon\ln^p(c_4)$ and $\ln(\sum_{r\in\calR_f}\bV_r(s_0))\lesssim \ln(c_4)$.
    Now by the regret bound and AM-GM inequality, we have
    \begin{align*}
    		C_{K'} &\lesssim Z_{K'} + c_1\sqrt{Z_{K'}\ln^p(c_3K')} + c_2\ln^p(c_3K') \lesssim Z_{K'} + (c_1^2 + c_2)\ln^p(c_4)\\
		&\lesssim \frac{\sum_{r\in\calR_f}\bV_r(s_0)\ln^4(c_4Z_{K'})}{\epsilon^2} + \frac{L\aS_{L(1+\epsilon)}A\ln^5(c_4Z_{K'})}{\epsilon^2} + (c_1^2 + c_2)\ln^p(c_4)\\
		&\lesssim \frac{L\aS_{L(1+\epsilon)}A\ln^8(c_4)}{\epsilon^2} + \frac{c_1^2\ln^{p+8}(c_4)}{\epsilon^2} + \frac{c_2\ln^{p+4}(c_4)}{\epsilon}.
    \end{align*}    
    Setting $R'$ to be the total number of rounds, we have $K'=K$ and the proof completes.
\end{proof}

\begin{lemma}
    \label{lem:output PC}
    With probability at least $1-4\delta$, we have $V^{\tilpi_g}_g(s_0) \leq \optV_{\tset,g}(s_0)(1 + \epsilon)$ for $g\in\tset$ throughout the execution of \pref{alg:PC}.
\end{lemma}
\begin{proof}
    By \pref{lem:nu} and \pref{lem:init bound}, with probability at least $1-2\delta$, we have $V^{\hatpi}_{\gstar}(s) \leq 2\optV_{\tset,\gstar}(s)\leq 4\optV_{\tset,\gstar}(s_0)\leq \min\{8L, 4V^{\hatpi}_{\gstar}(s_0)\}$ for any $s\in\calS$ throughout the execution.
    For any $g\in\tset$, at the round that $\tilpi_g$ is determined (where $\gstar=g$), by \pref{lem:V pi mean}, value of $\lambda$ and definition of success round, $V^{\tilpi_g}_g(s_0) = V^{\hatpi}_g(s_0) \leq \hattau + \frac{\epsilon}{256}\norm{V^{\hatpi}_g}_{\infty} \leq \hattau + \frac{\epsilon}{4}V^{\hatpi}_g(s_0) \leq \hatV(s_0)(1+\frac{\epsilon}{2}) + \frac{\epsilon}{4}V^{\hatpi}_g(s_0)$.
    This gives $V^{\tilpi_g}_g(s_0) \leq \frac{1 + \frac{\epsilon}{2}}{1-\frac{\epsilon}{4}}\hatV(s_0) \leq (1+\epsilon)\optV_{\tset,g}(s_0)$ by $\hatV(s_0)\leq\optV_{\tset,g}(s_0)$ (\pref{lem:opt PC}) and $\epsilon\in(0, 1]$.
\end{proof}

\begin{lemma}
    \label{lem:reg PC}
    With probability at least $1-12\delta$, for any $K'\leq K$, we have 
    $R_{K'}\lesssim \sqrt{L\aS_{L(1+\epsilon)}A\sumkp V_k(s_0)\iota} + L{\aS_{L(1+\epsilon)}}^2A\iota$, where $\iota=\ln^2\frac{L\aS_{L(1+\epsilon)}AK'}{\delta}$.
\end{lemma}
\begin{proof}
    By \pref{lem:anytime bernstein} and a union bound on $\{\optV_{\tset,g}\}_{g\in\tset}$ and $(s, a)\in\tset\times\calA$, with probability at least $1-\delta$, $(P^k_i -\P^k_i)\optV_k\lesssim \sqrt{\frac{\fV(P^k_i, \optV_k)\iota'}{\N^k_i}} + \frac{L\iota'}{\N^k_i}$ for any $k\in[K']$ and $i\in[I_k]$ (note that this holds even if $s^k_i\notin\calK$), where $\iota'=\ln\frac{\aS_{L(1+\epsilon)}AC_{K'}}{\delta}$.
    Moreover, with probability at least $1-\delta$,
    \begin{align*}
        &\sumkp\rbr{I_k - V_k(s_0)} \leq \sumkp\sumi\rbr{1 + V_k(s^k_{i+1}) - V_k(s^k_i)}\\
        &\lesssim \sumkp\sumi\rbr{(\Ind_{s^k_{i+1}} - P^k_i)V_k + (P^k_i - \P^k_i)V_k + b^k_i + \epsilon_k} \tag{\pref{lem:def Vk}}\\
        &\lesssim \sqrt{\sumkp\sumi\fV(P^k_i, V_k)\ln\frac{LC_{K'}}{\delta}} + \sumkp\sumi\rbr{(P^k_i-\P^k_i)\optV_k + (P^k_i-\P^k_i)(V_k-\optV_k) + b^k_i} + L\ln\frac{LC_{K'}}{\delta}.
    \end{align*}
    where the last step is by \pref{lem:sum eps} and \pref{lem:anytime freedman}.
    Now note that with probability at least $1-2\delta$,
    \begin{align*}
        &\sumkp\sumi\rbr{(P^k_i-\P^k_i)\optV_k + (P^k_i-\P^k_i)(V_k-\optV_k) + b^k_i}\\
        &\lesssim \sumkp\sumi\rbr{\sqrt{\frac{\fV(P^k_i, \optV_k)\iota'}{\N^k_i}} + \sqrt{\frac{\Gamma_{L(1+\epsilon)}\fV(P^k_i, V_k-\optV_k)\iota'}{\N^k_i}} + \frac{\Gamma_{L(1+\epsilon)}L\iota'}{\N^k_i} + b^k_i} \tag{\pref{lem:dPV}, $\norm{\optV_k}_{\infty}\leq 2L+1$, $\iota'=\ln\frac{\aS_{L(1+\epsilon)}AC_{K'}}{\delta}$}\\
        &\lesssim \sqrt{\aS_{L(1+\epsilon)}A\sumkp\sumi\fV(P^k_i, V_k)\iota'} + \sqrt{\aS_{L(1+\epsilon)}\Gamma_{L(1+\epsilon)}A\sumkp\sumi\fV(P^k_i, V_k-\optV_k)\iota'} + L{\aS_{L(1+\epsilon)}}^2A\iota',
    \end{align*}
    where in the last step $\iota'=\ln^2\frac{\aS_{L(1+\epsilon)}AC_{K'}}{\delta}$ and we apply \pref{lem:sum N}, Cauchy-Schwarz inequality, \pref{lem:sum b}, and $\var[X+Y]\leq 2(\var[X]+\var[Y])$.
    Thus, by \pref{lem:sum dV} with \pref{lem:opt PC} and AM-GM inquality, with probability at least $1-8\delta$, we continue with
    \begin{align*}
        C_{K'} - \sumkp V_k(s_0)&\lesssim \sqrt{\aS_{L(1+\epsilon)}A\sumkp\sumi\fV(P^k_i, V_k)\iota'} + L{\aS_{L(1+\epsilon)}}^2A\iota'\\
        &\lesssim \sqrt{L\aS_{L(1+\epsilon)}AC_{K'}\iota'} + L{\aS_{L(1+\epsilon)}}^2A\iota', \tag{\pref{lem:sum var Vk}}
    \end{align*}
    where $\iota'=\ln^2\frac{L\aS_{L(1+\epsilon)}AC_{K'}}{\delta}$.
    Solving a ``quadratic'' inequality w.r.t $C_{K'}$ (\pref{lem:quad log}), we have $C_{K'}\lesssim \sumkp V_k(s_0) + L{\aS_{L(1+\epsilon)}}^2A\ln^2\frac{L\aS_{L(1+\epsilon)}AK'}{\delta}$.
    Plugging this back to the last inequality above completes the proof.
\end{proof}


\begin{lemma}
    \label{lem:nu}
    With probability at least $1-2\delta$, throughout the execution of \pref{alg:PC}, $V_{\gstar}^{\hatpi}(s)\leq 2\optV_{\tset,\gstar}(s)$ for any $s\in\calS$.
\end{lemma}
\begin{proof}
    By \pref{lem:opt PC}, value of $\nu$ (\pref{line:nu}), and applying \pref{lem:bounded error fresh} with $\calX=\tset\setminus\{g\}$ for each $g\in\tset$, we have $V^{\hatpi}_{\gstar}(s)\leq 2\hatV(s)\leq 2\optV_{\tset,\gstar}(s)$ for all $s\in\calS$.
\end{proof}

\begin{lemma}
    \label{lem:opt PC}
    With probability at least $1-\delta$, throughout the execution of \pref{alg:PC}, $\hatV(s)\leq\optV_{\tset,\gstar}(s)$ for any $s\in\calS$.
\end{lemma}
\begin{proof}
    This is simply by the value of $\hatV$ in each round and applying \pref{lem:opt} on $\{\optV_{\tset,g}\}_{g\in\tset}$.
\end{proof}
\section{Lemmas for Policy Evaluation}
\label{app:evaluation}
In this section, we present a set of lemmas related to regret analysis shared among \pref{alg:LOGSSD}, \pref{alg:SD}, and \pref{alg:PC}.
In \pref{alg:SD}, a trial is indexed by $\tau$, and each trial corresponds to a value of $z$ estimating $\aS_{L(1+\epsilon)}$ (\pref{line:trial}).
In \pref{alg:LOGSSD} and \pref{alg:PC}, we assume the whole learning procedure lies in an artificial trial.
Note that when lemmas below are involved, we have $b^k_i=0$, $\N^k_i=\infty$, and $\P^k_i=\Ind_{s_0}$ when $s^k_i\notin\calK_k$.

\begin{lemma}
    \label{lem:sum var Vk}
	Let $\calG$ be the goal set such that $\acalS_{L(1+\epsilon)}\subseteq\calG\subseteq\calS$.
    In any trial, with probability at least $1-2\delta$, for any $K'\in[K]$, if $\calK_k\subseteq\acalS_{L(1+\epsilon)}$ and $g_k\in\calG\setminus\calK_k$ for any $k\in[K']$, then $\sumkp\sumi \fV(P^k_i, V_k)\lesssim LC_{K'} + L^2\Gamma_{L(1+\epsilon)}\aS_{L(1+\epsilon)}A\iota$, where $\iota = \bigo{\ln(|\calG|ALC_{K'}/\delta)\ln(C_{K'})}$.
\end{lemma}
\begin{proof}
	Note that $\norm{V_k}_{\infty}\leq 2L$ by the stopping condition (\pref{line:bound V}) of \pref{alg:VISGO}, and with probability at least $1-\delta$,
	\begin{align*}
        &\sumkp\sumi \rbr{V_k(s^k_i)^2 - (P^k_iV_k)^2} \lesssim L\sumk\sumi(V_k(s^k_i) - P^k_iV_k)_+ \tag{$a^2-b^2\leq (a+b)(a-b)_+$ for $a,b\geq 0$} \\
		&\lesssim L\sumkp\sumi\rbr{1 + (\P^k_i - P^k_i)V_k + \frac{1}{\N^k_i} + \epsilon_k}_+ \tag{\pref{lem:def Vk}}\\
		&\lesssim LC_{K'} + L\sumkp\sumi\rbr{\sqrt{\frac{\Gamma_{L(1+\epsilon)}\fV(P^k_i, V_k)\iota'}{\N^k_i}} + \frac{L\Gamma_{L(1+\epsilon)}\iota'}{\N^k_i} + \epsilon_k} \tag{\pref{lem:dPV} and $\N^k_i=\infty$ when $s^k_i\notin\calK_k$}\\
		&\lesssim LC_{K'} + L\sqrt{\Gamma_{L(1+\epsilon)}\aS_{L(1+\epsilon)}A\sumkp\sumi\fV(P^k_i, V_k)\iota'\ln(C_{K'})} + L^2\Gamma_{L(1+\epsilon)}\aS_{L(1+\epsilon)}A\iota'\ln(C_{K'}),
	\end{align*}
    where $\iota'=\ln(|\calG|AC_{K'}/\delta)$, and the last step is by Cauchy-Schwarz inequality, \pref{lem:sum N}, and \pref{lem:sum eps}.
    Now let $Z_{K'}=\sumkp\sumi\fV(P^k_i, V_k)$.
	Applying \pref{lem:sum var} and $\sumkp V_k(s^k_{I_k+1})^2\lesssim L^2\aS_{L(1+\epsilon)}A\iota'$ (this is because $V_k(s^k_{I_k+1})$ is non-zero only in skip rounds), we have with probability a least $1-\delta$,
	\begin{align*}
		Z_{K'} \lesssim LC_{K'} + L\sqrt{\Gamma_{L(1+\epsilon)}\aS_{L(1+\epsilon)}AZ_{K'}\iota} + L^2\Gamma_{L(1+\epsilon)}\aS_{L(1+\epsilon)}A\iota,
	\end{align*}
	where $\iota = \bigo{\ln(|\calG|ALC_{K'}/\delta)\ln(C_{K'})}$.
	Solving a quadratic inequality completes w.r.t.~$Z_{K'}$ the proof.
\end{proof}

\begin{lemma}
    \label{lem:sum dV}
    In any trial, with probability at least $1-5\delta$, for any $K'\in[K]$ if 1) $\{\optV_k\}_{k\in[K']}\subseteq\calV$ where $\calV$ is determined at the beginning of the trial, $|\calV|$ is upper bounded by polynomials of $\aS_{L(1+\epsilon)}$, and $\norm{V}_{\infty}=\bigo{L}$ for any $V\in\calV$, 2) $V_k(s)\leq\optV_k(s)$ for any $k\in[K']$ and $s\in\calS$, 3) $\calK_k\subseteq\acalS_{L(1+\epsilon)}$ for any $k\in[K']$, and 4) $g_k\in\bcalU\setminus\calK_k$ for any $k\in[K']$, then $\sumkp\sumi\fV(P^k_i, \optV_k - V_k)\lesssim L\sqrt{\aS_{L(1+\epsilon)}A\sumkp\sumi \fV(P^k_i, V_k)\iota'} + L^2{\aS_{L(1+\epsilon)}}^2A\iota'$, where $\iota'=\ln^2\frac{L\aS_{L(1+\epsilon)}AC_{K'}}{\delta}$.
\end{lemma}
\begin{proof}
    First note that
    \begin{align*}
        &\sumkp\sumi\rbr{(\optV_k(s^k_i) - V_k(s^k_i))^2 - (P^k_i(\optV_k-V_k))^2}\\
        &\lesssim L\sumkp\sumi(\optV_k(s^k_i) - V_k(s^k_i) - P^k_i\optV_k + P^k_iV_k)_+ \tag{$V_k(s)\leq\optV_k(s)$ for all $s$ and $a^2-b^2\leq (a+b)(a-b)_+$ for $a,b\geq 0$}\\
        &\lesssim L\sumkp\sumi(1 + P^k_iV_k - V_k(s^k_i))_+ \tag{$\optV_k(s^k_i)\leq 1 + P^k_i\optV_k$}.
    \end{align*}
    Let $\P_{s,a}(s')=\frac{\N(s,a,s')}{\N^+(s,a)}$.
    By \pref{lem:anytime bernstein}, with probability at least $1-\delta$, for any $(s, a)\in\acalS_{L(1+\epsilon)}\times\calA$, $V\in\calV$, and status of counter $\N$:
    \begin{align}
        (P_{s,a} - \P_{s,a})V \lesssim \sqrt{\frac{\fV(P_{s,a}, V)\iota'}{\N(s, a)}} + \frac{L\iota'}{\N(s, a)},\label{eq:calV}
    \end{align}
    where $\iota'=\ln\frac{\aS_{L(1+\epsilon)}AC_{K'}}{\delta}$.
    By \pref{lem:def Vk}, with probability at least $1-2\delta$, we continue with
    \begin{align*}
        &\lesssim L\sumkp\sumi( (P^k_i-\P^k_i)\optV_k + (P^k_i-\P^k_i)(V_k - \optV_k) + b^k_i + \epsilon_k)_+\\
        &\lesssim L\sumkp\sumi\rbr{\sqrt{\frac{\fV(P^k_i, \optV_k)\iota'}{\N^k_i}} + \sqrt{\frac{\Gamma_{L(1+\epsilon)}\fV(P^k_i, V_k-\optV_k)\iota'}{\N^k_i}} + \frac{\Gamma_{L(1+\epsilon)}L\iota'}{\N^k_i} + b^k_i + \epsilon_k} \tag{\pref{eq:calV}, \pref{lem:dPV}, conditions 3) and 4), $\iota'=\ln\frac{\aS_{L(1+\epsilon)}AC_{K'}}{\delta}$}\\
        &\lesssim L\rbr{\sqrt{\aS_{L(1+\epsilon)}A\sumkp\sumi\fV(P^k_i, V_k)\iota'} + \sqrt{{\aS_{L(1+\epsilon)}}^2A\sumkp\sumi\fV(P^k_i, V_k - \optV_k)\iota'}} + L^2{\aS_{L(1+\epsilon)}}^2A\iota',
    \end{align*}
    where in the last step $\iota'=\ln^2\frac{\aS_{L(1+\epsilon)}AC_{K'}}{\delta}$ and we apply $\var[X_1+X_2]\leq\var[X_1]+\var[X_2]$, Cauchy-Schwarz inequality, \pref{lem:sum N}, \pref{lem:sum eps}, and \pref{lem:sum b}.
    Then applying \pref{lem:sum var} with $\norm{\optV_k-V_k}_{\infty}\lesssim L$ and solving a quadratic inequality w.r.t.~$\sumkp\sumi\fV(P^k_i, \optV_k-V_k)$, we have with probability at least $1-\delta$,
    \begin{align*}
        &\sumkp\sumi\fV(P^k_i, \optV_k - V_k)\\
        &\lesssim \sumkp (\optV_k(s^k_{I_k+1})-V_k(s^k_{I_k+1}))^2 + L\sqrt{\aS_{L(1+\epsilon)}A\sumkp\sumi \fV(P^k_i, V_k)\iota'} + L^2{\aS_{L(1+\epsilon)}}^2A\iota'. \tag{$\iota'=\ln^2\frac{L\aS_{L(1+\epsilon)}AC_{K'}}{\delta}$}
    \end{align*}
    The proof is completed by noting that $\optV_k(g)=V_k(g)=0$ and $\sumkp\Ind\{s^k_{I_k+1}\neq g\}\lesssim \aS_{L(1+\epsilon)}A$.
\end{proof}

\begin{lemma}
    \label{lem:sum var}
    Let $K\in\mathbb{N}$ and $\{V_k\}_{k\in[K]}$ be a sequence of value functions with $V_k\in[0, B]^{\calS}$ for $B>0$. With probability at least $1-\delta$, for any $K'\in [K]$, $$\sumkp\sumi  \fV(P^k_i,V_k)\lesssim \sumkp V_k(s^k_{I_k+1})^2 + \sumkp\sumi \rbr{V_k(s^k_i)^2 - (P^k_iV_k)^2}  + B^2\iota,$$
    where $\iota = \ln(BC_{K'}/\delta)$.
\end{lemma}
\begin{proof}
    We decompose the sum as follows:
    \begin{align*}
        \sumkp\sumi \fV(P^k_i, V_k) = \sumkp\sumi  \rbr{P^k_i(V_k)^2 - V_k(s^k_{i+1})^2}
         + \sumkp\sumi  \rbr{V_k(s^k_{i+1})^2 - V_k(s^k_i)^2} + \sumkp\sumi  \rbr{V_k(s^k_i)^2 - (P^k_iV_k)^2}.
    \end{align*}
    For the first term, by \pref{lem:anytime freedman}, \pref{lem:quad}, and $I_k<\infty$ for any $k\in[K]$ by the skip-round condition, with probability at least $1-\delta$, for all $K'\in[K]$,
    \begin{align*}
        \sumkp\sumi \rbr{P^k_i(V_k)^2 - V_k(s^k_{i+1})^2} &\lesssim \sqrt{\sumkp\sumi \fV(P^k_i, (V_k)^2)\iota} + B^2\iota\\
        &\lesssim B\sqrt{\sumkp\sumi \fV(P^k_i, V_k)\iota} + B^2\iota,
    \end{align*}
    where $\iota = \bigo{\ln(BC_{K'}/\delta)}$.
    The second term is clearly upper bounded by $\sumkp  V_k(s^k_{I_k+1})^2$.
    Putting everything together and solving a quadratic inequality w.r.t.~$\sumkp\sumi \fV(P^k_i,V_k)$ completes the proof. 
\end{proof}


\begin{lemma}
    \label{lem:sum b}
	Let $\calG$ be the goal set such that $\acalS_{L(1+\epsilon)}\subseteq\calG\subseteq\calS$.
    In any trial, with probability at least $1-\delta$, for any $K'\in[K]$, if $\calK_k\subseteq\acalS_{L(1+\epsilon)}$ and $g_k\in\calG\setminus\calK_k$ for any $k\in[K']$, then
    $\sumkp\sumi b^k_i\lesssim \sqrt{\aS_{L(1+\epsilon)} A\sumkp\sumi\fV(P^k_i, V_k)\iota} + L{\aS_{L(1+\epsilon)}}^{1.5}A\iota$, where $\iota=\ln(|\calG|AC_{K'}/\delta)$.
\end{lemma}
\begin{proof}
    Note that with probability at least $1-\delta$,
    \begin{align*}
        \sumkp\sumi b^k_i &\lesssim \sumkp\sumi\rbr{\sqrt{\frac{\fV(\P^k_i, V_k)\iota}{\N^k_i}} + \frac{L\iota}{\N^k_i}} \tag{definition of $b^k_i$ and $\max\{a,b\}\leq a + b$}\\
        &\lesssim \sumkp\sumi\rbr{ \sqrt{\frac{\fV(P^k_i, V_k)\iota}{\N^k_i}} + \frac{L\sqrt{\aS_{L(1+\epsilon)}}\iota}{\N^k_i} } \tag{\pref{lem:barPV to PV}}\\
        &\lesssim \sqrt{\aS_{L(1+\epsilon)}A\sumkp\sumi\fV(P^k_i, V_k)\iota} + L{\aS_{L(1+\epsilon)}}^{1.5}A\iota. \tag{Cauchy-Schwarz inequality and \pref{lem:sum N}}
    \end{align*}
    This completes the proof.
\end{proof}

\begin{lemma}
    \label{lem:sum N}
    In any trial, for any $K'\in[K]$, if $\calK_k \subseteq \acalS_{L(1+\epsilon)}$ for any $k\in[K']$, we have $\sumkp\sumi \frac{1}{\N^k_i}\lesssim \aS_{L(1+\epsilon)}A\log_2(C_{K'})$.
\end{lemma}
\begin{proof} 
    Note that, for any $i,k$, if $s_i^k \notin \acalS_{L(1+\epsilon)}$ we must have $s_i^k \notin \calK_k$, which implies that the corresponding count $N_i^k$ is $\infty$. Then,
    \begin{align*}
        \sumk\sumi\frac{1}{\N^k_i} 
        &\leq \sum_{s \in \acalS_{L(1+\epsilon)},a\in\calA} ~\sum_{0\leq h \leq \log_2(C_K)} \sumk\sumi \Ind\big[(s_i^k,a_i^k) = (s,a), \N_i^k(s,a)= 2^h\big] \frac{1}{2^h}\\
        &\leq |\acalS_{L(1+\epsilon)}| A \log_2(C_k).
    \end{align*}
\end{proof}


\begin{lemma}
    \label{lem:sum eps}
    In any trial, for any $K'\in[K]$, $\sumkp\sumi\epsilon_k=\bigo{\ln C_{K'}}$.
\end{lemma}

\begin{lemma}
    \label{lem:def Vk}
    In any trial, $1 + \P^k_iV_k - 2b^k_i - \epsilon_k\leq V_k(s^k_i)\leq 1 + \P^k_iV_k + \epsilon_k$ for any $k\in[K], i\in[I_k]$.
\end{lemma}
\begin{proof}
    When $s^k_i\notin\calK_k$, we have $b^k_i=\frac{1}{\N^k_i}=0$ and $\P^k_iV_k=V_k(s_0)$.
    Thus, the statement holds.
    When $s^k_i\in\calK_k$, by the definition of $V_k$ and the stopping rule of \pref{alg:VISGO}, we have
    \begin{align*}
        V_k(s^k_i) &\geq 1 + \tilP^k_iV_k - b^k_i - \epsilon_k \geq 1 + \P^k_iV_k - b^k_i - \epsilon_k - \frac{\P^k_iV_k}{\N^k_i} \tag{definition of $\tilP^k_i$}\\
        &\geq 1 + \P^k_iV_k - 2b^k_i - \epsilon_k,
    \end{align*}
    where the last step is by $\frac{\P^k_iV_k}{\N^k_i}\leq\frac{2L}{\N^k_i}\leq b^k_i$.
    Moreover, $V_k(s^k_i)\leq 1 + \tilP^k_iV_k + \epsilon_k \leq 1 + \P^k_iV_k + \epsilon_k$.
    This completes the proof.
\end{proof}
\clearpage

\section{Auxiliary Results}\label{app:auxiliary}

\begin{lemma}
    \label{lem:example 2L}
    For any $S \geq 1$, $A \geq 2$, $\frac{3}{2} \leq L \leq \frac{1}{2} + \frac{\ln(S/2)}{2\ln(A)}$, and $0 < \epsilon < \frac{L-1}{L}$, there exists an MDP with $S$ states and $A$ actions (including action $\reset$) such that $\aS_{L(1+\epsilon)}\Gamma_{L(1+\epsilon)}=1$ while $\aS_{2L}\geq A^{2(L-1)}$.
\end{lemma}
\begin{proof}
    Consider an MDP with the following structure.
    At $s_0$, taking any action transits to one of $\{s_1,\ldots,s_L\}$ with probability $\frac{1}{L}$.
    At any state in $\{s_1,\ldots,s_L\}$, taking any action transits to state $s^{\star}$.
    States reachable from $s^{\star}$ form a full $A$-ary tree with depth $2(L-1)$.
    The rest of the states are ignored (note that $S\geq 2A^{2L-1}\geq 1 + L + \sum_{i=0}^{2(L-1)}A^i$).
    It is not hard to see that it takes $2L-1$ steps to reach any $s_i$ for $i\in[L]$ by a policy restricted on $\{s_0\}$.
    Therefore, all unignored states are $2L$ incrementally controllable and thus $\aS_{2L} \geq A^{2(L-1)}$ states.
    On the other hand, by $L(1+\epsilon)<2L-1$, $\acalS_{L(1+\epsilon)}=\{s_0\}$ and $\Gamma_{L(1+\epsilon)}=1$ (note that the agent can reach $s_0$ from $s_0$ by taking $\reset$).
\end{proof}

\begin{remark}
    The construction in \pref{lem:example 2L} also have $\aS_{2L} = \lowo{S}$ while $\aS_{L(1+\epsilon)}\Gamma_{L(1+\epsilon)} = \bigo{1}$.
\end{remark}

\begin{lemma}
    \label{lem:init bound}
    For any $\calX\subseteq\calS$ and $g\in\calS$, we have $\norm{\optV_{\calX,g}}_{\infty} \leq 1 + \optV_{\calX,g}(s_0)$.
\end{lemma}
\begin{proof}
    Clearly $\optV_{\calX,g}(g)=0\leq 1 + \optV_{\calX,g}(s_0)$ and $\optV_{\calX,g}(s)=1+\optV_{\calX,g}(s_0)$ for any $s\in\calS\setminus(\calX\cup\{g\})$.
    For any $s\in\calX\setminus\{g\}$, by Bellman optimality and $\reset\in\calA$ we have $\optV_{\calX,g}(s) \leq 1 + \optV_{\calX,g}(s_0)$.
\end{proof}

\begin{lemma}
    \label{lem:barPV to PV}
    Let $n$ be a counter incrementally collecting samples from transition function $P$, and define $\P^n_{s,a}(s'):=\frac{n(s, a, s')}{n^+(s, a)}$.
    Let $\calG$ be the goal set such that $\acalS_{L(1+\epsilon)}\subseteq\calG\subseteq\calS$.
    With probability at least $1-\delta$, for any status of $n$, $(s, a)\in\acalS_{L(1+\epsilon)}\times\calA$, $\calX\subseteq\acalS_{L(1+\epsilon)}$, $g\in\calG\setminus\calX$, and value function $V$ restricted on $\calX\cup\{g\}$ with $\norm{V}_{\infty}\leq B$ for some $B>0$, we have $\fV(\P^n_{s,a}, V) \lesssim \fV(P_{s,a}, V) + \frac{\Gamma_{L(1+\epsilon)}B^2\iota'_{s,a}}{n^+(s, a)}$, where $\iota'_{s,a}=\bigo{\ln \frac{|\calG|An^+(s,a)}{\delta}}$.
\end{lemma}
\begin{proof}
    Note that
    \begin{align*}
        \fV(\P_{s,a}, V) &\leq \P_{s,a}(V - P_{s,a}V)^2 \tag{$\frac{\sum_ip_ix_i}{\sum_ip_i}=\argmin_z\sum_ip_i(x_i-z)^2$}\\
        &= \fV(P_{s,a}, V) + (\P_{s,a}-P_{s,a})(V - P_{s,a}V)^2\\
        &\lesssim \fV(P_{s,a}, V) +  B\sqrt{\frac{\Gamma_{L(1+\epsilon)}\fV(P_{s,a}, V)\iota'_{s,a}}{n^+(s, a)}} + \frac{\Gamma_{L(1+\epsilon)}B^2\iota'_{s,a}}{n^+(s, a)} \tag{\pref{lem:dPV} and \pref{lem:quad}}\\
        &\lesssim \fV(P_{s,a}, V) + \frac{\Gamma_{L(1+\epsilon)}B^2\iota'_{s,a}}{n^+(s, a)}. \tag{AM-GM inequality}
    \end{align*}
    This completes the proof.
\end{proof}

\begin{lemma}
    \label{lem:dPV}
    Let $n$ be a counter incrementally collecting samples from transition function $P$, and define $\P^n_{s,a}(s'):=\frac{n(s, a, s')}{n^+(s, a)}$.
    Let $\calG$ be the goal set such that $\acalS_{L(1+\epsilon)}\subseteq\calG\subseteq\calS$.\footnote{In most cases, we apply this lemma with $\calG\in\{\acalS_{L(1+\epsilon)}, \calS\}$.}
    With probability at least $1-\delta$, for any status of $n$, $(s, a)\in\acalS_{L(1+\epsilon)}\times\calA$, $\calX\subseteq\acalS_{L(1+\epsilon)}$, $g\in\calG\setminus\calX$, and value function $V$ restricted on $\calX\cup\{g\}$ with $\norm{V}_{\infty}\leq B$ for some $B>0$, we have 
    $$|(P_{s, a}-\P_{s, a}^n)V| \lesssim \sqrt{\frac{\min\{|\calX|,\Gamma^{s, a}_{L(1+\epsilon)}\}\fV(P_{s, a}, V)\iota_{s,a}'}{n^+(s, a)}} + \frac{B\min\{|\calX|, \Gamma^{s, a}_{L(1+\epsilon)}\}\iota_{s,a}'}{n^+(s, a)},$$
    where $\iota_{s,a}' = \bigo{\ln \frac{\aS_{L(1+\epsilon)}A\Gamma^2_{L(1+\epsilon)}|\calG|n^+(s,a)}{\delta}}$.
\end{lemma}
\begin{proof}
    By \pref{lem:anytime bernstein} and a union bound, for any $\delta'\in(0,1)$, with probability at least $1-\frac{\delta'}{\aS_{L(1+\epsilon)}A\Gamma_{L(1+\epsilon)}{\Gamma^{s, a}_{L(1+\epsilon)} \choose i}|\calG|}$, for each status of $n$, $(s, a)\in \acalS_{L(1+\epsilon)}\times\calA$, size $i\in[\Gamma^{s,a}_{L(1+\epsilon)}]$, subset $y'\subseteq \calN^{s,a}_{L(1+\epsilon)}$ with $|y'|=i$, and $g\in\calG\setminus y'$,
	\begin{align*}
		|P_{s, a}(y) - \P_{s,a}^n(y)| \leq 2\sqrt{2\frac{P_{s,a}(y)(1-P_{s,a}(y))\ln(2n^+(s,a)/\delta')}{n^+(s, a)}} + \frac{\ln(2n^+(s,a)/\delta')}{n^+(s, a)},
	\end{align*}
    where $y=\calS\setminus(y'\cup\{g\})$. Let $y'=\calX'\triangleq\calX\cap\calN^{s,a}_{L(1+\epsilon)}$ such that $y=\calS\setminus(\calX'\cup\{g\})$. By another application of \pref{lem:anytime bernstein} and a union bound, for any $\delta'\in(0,1)$, with probability at least $1-\frac{\delta'}{|\calG|}$, for all $s'\in\calX'\cup\{g\} \subseteq \calG$,
    \begin{align*}
		|P_{s, a}(s') - \P_{s,a}^n(s')| \leq 2\sqrt{2\frac{P_{s,a}(s')(1-P_{s,a}(s'))\ln(2n^+(s,a)/\delta')}{n^+(s, a)}} + \frac{\ln(2n^+(s,a)/\delta')}{n^+(s, a)}.
	\end{align*}
    Thus, setting $\delta' = \delta / 2\aS_{L(1+\epsilon)}A\Gamma_{L(1+\epsilon)}{\Gamma^{s, a}_{L(1+\epsilon)} \choose i}|\calG|$ and using ${n \choose i}\leq n^{\min\{i, n - i\}}$, the two inequalities above simplify as
    \begin{align}
		|P_{s, a}(y) - \P_{s,a}^n(y)| &\lesssim \sqrt{\frac{i\cdot P_{s,a}(y)(1-P_{s,a}(y))\iota'_{s,a}}{n^+(s, a)}} + \frac{i\iota'_{s,a}}{n^+(s, a)}, \label{eq:ineq1}
        \\ |P_{s, a}(s') - \P_{s,a}^n(s')| &\lesssim \sqrt{\frac{P_{s,a}(s')(1-P_{s,a}(s')) \iota'_{s,a}}{n^+(s, a)}} + \frac{\iota'_{s,a}}{n^+(s, a)}.\label{eq:ineq2}
	\end{align}
	These hold with probability at least $1-\delta$. Now define, for all $s'\in\calS$,
    \begin{align*}
        V'(s')=\begin{cases}
            V(s'),& s'\in\calX'\cup\{g\}\\
            V(\calS\setminus(\calX\cup\{g\})),& \text{otherwise}
        \end{cases}
    \end{align*}
    and $\dV(s')=V'(s')-P_{s,a}V'$ for all $s'$.
	Clearly, $V'$ and $\dV$ are restricted on $\calX'\cup\{g\}$.
    Moreover, $V(s')\neq V'(s')\implies s'\in\calX\setminus y'\implies s'\in\calX\setminus\calN^{s,a}_{L(1+\epsilon)}\implies P_{s,a}(s')=0$ by $\calX\subseteq \acalS_{L(1+\epsilon)}$.
    Thus, $P_{s,a}V=P_{s,a}V'$, and
	\begin{align*}
		&(P_{s, a} - \P_{s, a}^n)V = (P_{s, a} - \P_{s, a}^n)V' = (P_{s,a}-\P^n_{s,a})\dV\\
		&= \sum_{s'\in\calX'}(P_{s, a}(s') - \P_{s, a}^n(s'))\dV(s') + (P_{s, a}(g) - \P_{s, a}^n(g))\dV(g) + (P_{s, a}(y) - \P_{ s, a}^n(y))\dV(y)\\
		&\lesssim \sum_{s'\in\calX'\cup\{g\}}\sqrt{\frac{P_{s, a}(s')\iota'_{s,a} }{n^+(s, a)}}|\dV(s')| + \sqrt{\frac{|\calX'| P_{s, a}(y)\iota'_{s,a}}{n^+(s, a)}}|\dV(y)| + \frac{B|\calX'|\iota'_{s,a}}{n^+(s, a)} \tag{\pref{eq:ineq1} and \pref{eq:ineq2}}\\
		&\lesssim \sqrt{\frac{|\calX'|\fV(P_{s, a}, V)\iota'_{s,a}}{n^+(s, a)}} + \frac{B|\calX'|\iota'_{s,a}}{n^+(s, a)}.
	\end{align*}
    where in the last step we apply Cauchy-Schwarz inequality and
    \begin{align*}
        \sum_{s'}P_{s,a}(s')\dV(s')^2 &= \sum_{s'}P_{s,a}(s')(V'(s') - P_{s,a}V)^2 \tag{$P_{s,a}V=P_{s,a}V'$}\\
        &= \sum_{s'}P_{s,a}(s')(V(s') - P_{s,a}V)^2 \tag{$P_{s,a}(s')=0$ when $V'(s')\neq V(s')$}\\
        &= \fV(P_{s,a},V).
    \end{align*}
    This completes the proof.
\end{proof}

\begin{lemma}
	\label{lem:quad log}
	If $x\leq a\sqrt{x\ln^p(dx)} + b\ln^p(dx) + c$ for some $a, b, c \geq 0$, $d>0$ and some absolute constant $p\geq 1$, then $x=\bigo{(a^2+b)\ln^p((a+b+c)d) + c}$.
\end{lemma}
\begin{proof}
	By AM-GM inequality and $\ln x < x$ for $x>0$, we have
	\begin{align*}
		x\leq a\sqrt{x\ln^p(dx)} + b\ln^p(dx) + c \leq \frac{x}{2} + (a^2/2+b)\ln^p(dx) + c \leq \frac{x}{2} + (a^2/2+b)(2p)^p\sqrt{dx} + c.
	\end{align*}
	Solving a quadratic inequality w.r.t.~$x$ gives $x=\bigo{(a^2+b)^2d+c}$.
	Plugging this back to the original inequality gives $x\leq a\sqrt{x\iota} + b\iota + c$, where $\iota=\ln^p((a+b+c)d)$.
	Further solving a quadratic inequality w.r.t~$x$ completes the proof.
\end{proof}

\begin{lemma}\citep[Lemma 40]{chen2022reaching}
    \label{lem:quad}
    For any random variable $X\in[-B,B]$, for some $B>0$, we have $\var[X^2]\leq 4B^2\var[X]$.
\end{lemma}

\begin{lemma}\citep[Lemma C.2]{cai2022near}
    \label{lem:mvp}
    For some $B>0$, let $\Upsilon=\{v\in\fR^{\calS}_{\geq 0}: v(g)=0, \norm{v}_{\infty}\leq B\}$ and $f:\Delta_{\calS}\times\Delta_{\calS}\times\Upsilon\times\fR_+\times\fR_+\rightarrow\fR$ with $f(\tilp,p,v,n,\iota)=\tilp v - \max\cbr{c_1\sqrt{\frac{\fV(p, v)\iota}{n}}, c_2\frac{B\iota}{n}}$ with some constants $c_1\geq 0$ and $c_2\geq 2c_1^2$.
    Then $f$ ensures for all $v$, $n$, $\iota$, and $\tilp$, $p$ s.t.~$\tilp(s)-\frac{1}{2}p(s)\geq0$ for all $s\neq g$,
    \begin{enumerate}
        \item $f(\tilp,p,v,n,\iota)$ is non-decreasing in $v(s)$, that is,
        \begin{align*}
            \forall v, v'\in\Upsilon, v\leq v' \implies f(\tilp,p,v,n,\iota) \leq f(\tilp,p,v',n,\iota);
        \end{align*}
        \item if $\tilp(g)>0$, then $f(\tilp,p,v,n,\iota)$ is $\rho_{\tilp}$-contractive in $v(s)$, with $\rho_{\tilp}=1-\tilp(g)<1$, that is,
        \begin{align*}
            \forall v,v'\in\Upsilon, \abr{f(\tilp,p,v,n,\iota) - f(\tilp,p,v',n,\iota)} \leq \rho_{\tilp}\norm{v-v'}_{\infty}.
        \end{align*}
    \end{enumerate}
\end{lemma}

\begin{lemma}
    \label{lem:V pi mean}
    There exist a function $N_{\dev}(L_0, \epsilon, \delta) =\bigo{ \ln^4\frac{L_0}{\epsilon\delta}/\epsilon^2}$, such that for any $g\in\calS$ and policy $\pi$ with $\norm{V^{\pi}_g}_{\infty}\leq L_0$ for some $L_0>0$, we have with probability at least $1-\delta$, for all $n\geq N_{\dev}(L_0, \epsilon, \delta)$ simultaneously, $|\hattau_n - V^{\pi}_g(s_0)| \leq \norm{V^{\pi}_g}_{\infty}\epsilon$, where $\hattau_n=\frac{1}{n}\sum_{i=1}^n C_i$ and each $C_i$ is a realization of the total cost incurred by following $\pi$ starting from $s_0$ with goal state $g$. 
\end{lemma}
\begin{proof}
    By \pref{lem:V pi dev}, with probability at least $1-\delta$, $\abr{\hattau_n - V^{\pi}_g(s_0)}\leq \frac{8\norm{V^{\pi}_g}_{\infty}}{\sqrt{n}}\ln^2\frac{8n^2\norm{V^{\pi}_g}_{\infty}}{\delta}$ for all $n\geq 1$.
    Solving the range of $n$ for the inequality $\frac{8\norm{V^{\pi}_g}_{\infty}}{\sqrt{n}}\ln^2\frac{8n^2L_0}{\delta}\leq \norm{V^{\pi}_g}_{\infty}\epsilon$ (\pref{lem:quad log}) completes the proof.
\end{proof}

\begin{lemma}
    \label{lem:V pi dev}
    For any $g\in\calS$ and policy $\pi$ with $\norm{V^{\pi}_g}_{\infty}\leq L_0$ for some $L_0\geq 1$, we have with probability at least $1-\delta$, for all $n\geq 1$ simultaneously, $|\hattau_n - V^{\pi}_g(s_0)| \leq \frac{8L_0}{\sqrt{n}}\ln^2\frac{8n^2L_0}{\delta}$, where $\hattau_n=\frac{1}{n}\sum_{i=1}^n C_i$ and each $C_i$ is a realization of the total cost incurred by following $\pi$ starting from $s_0$ with goal state $g$. 
\end{lemma}
\begin{proof}
    By \pref{lem:hitting} and a union bound,
    \begin{align*}
        \mathbb{P}\left( \exists i \geq 1 :  C_i > 4L_0\ln\frac{8i^2L_0}{\delta}\right) \leq \sum_{i\geq 1}\mathbb{P}\left(  C_i > 4L_0\ln\frac{8i^2L_0}{\delta}\right) \leq \sum_{i \geq 1} \frac{\delta}{4i^2L_0} \leq \frac{\delta}{2}.
    \end{align*}
    Then, under the complement of the event above (which holds with probability at least $1-\frac{\delta}{2}$), we have $\bar{\tau}_n=\hattau_n$ for all $n\geq 1$, where $\bar{\tau}_n=\frac{1}{n}\sum_{i=1}^n C_i\Ind\{C_i\leq 4L_0\ln\frac{8n^2L_0}{\delta}\}$. Moreover, by \pref{lem:azuma} and a union bound,
    \begin{align*}
        \mathbb{P}\left( \exists n \geq 1 :  |\bar{\tau}_n-\E[\bar{\tau}_n]| > 4L_0\ln\frac{8n^2L_0}{\delta}\sqrt{\frac{2\ln\frac{8n^2}{\delta}}{n}}\right) \leq \sum_{n \geq 1} \frac{\delta}{4n^2} \leq \frac{\delta}{2}.
    \end{align*}
    A union bound on the complement of the two events above yields that, with probability at least $1-\delta$, for all $n\geq 1$ simultaneously,
    \begin{align*}
        \hattau_n - V^{\pi}_g(s_0) = \bar{\tau}_n - V^{\pi}_g(s_0) \leq \bar{\tau}_n - \E[\bar{\tau}_n] \leq 4L_0\ln\frac{8n^2L_0}{\delta}\sqrt{\frac{2\ln\frac{8n^2}{\delta}}{n}},
    \end{align*}
    and by \pref{lem:hitting},
    \begin{align*}
       V^{\pi}_g(s_0) - \hattau_n \leq \E[\bar{\tau}_n] - \bar{\tau}_n + L_0\cdot\frac{1}{2nL_0} \leq 4L_0\ln\frac{8n^2L_0}{\delta}\sqrt{\frac{2\ln\frac{8n^2}{\delta}}{n}} + \frac{1}{2n}.
    \end{align*}
    Combining these two cases gives $\abr{\hattau_n - V^{\pi}_g(s_0)}\leq \frac{8L_0}{\sqrt{n}}\ln^2\frac{8n^2L_0}{\delta}$.
\end{proof}

\begin{lemma}{\citep[Lemma B.5]{cohen2020near}}
	\label{lem:hitting}
	For a given $g\in\calS$, let $\pi$ be a policy such that $\norm{V^{\pi}_g}_{\infty}\leq\tau$.
    Then, for any $n\in\mathbb{N}$, the probability that the cost of $\pi$ to reach the goal state starting from any state is more than $n$, is at most $2e^{-\frac{n}{4\tau}}$.
\end{lemma}


\begin{lemma}[Azuma's inequality]
    \label{lem:azuma}
    Let $\{X_t\}_{t=1}^n$ be a martingale difference sequence with $|X_t|\leq B$.
    Then with probability at least $1-\delta$, $|\sum_{t=1}^nX_i|\leq B\sqrt{2n\ln\frac{2}{\delta}}$.
\end{lemma}

\begin{lemma}\citep[Lemma 34]{chen2021implicit}
	\label{lem:anytime bernstein}
	Let $\{X_t\}_t$ be a sequence of i.i.d random variables with mean $\mu$, variance $\sigma^2$, and $0\leq X_t \leq B$.
	Then with probability at least $1-\delta$, the following holds for all $n\geq 1$ simultaneously:
	\begin{align*}
		\abr{\sum_{t=1}^n(X_t-\mu)} &\leq 2\sqrt{2\sigma^2 n\ln\frac{2n}{\delta}} + 2B\ln\frac{2n}{\delta}.\\
		\abr{\sum_{t=1}^n(X_t-\mu)} &\leq 2\sqrt{2\hat{\sigma}^2_nn\ln\frac{2n}{\delta}} + 19B\ln\frac{2n}{\delta}.
	\end{align*}
	where $\hat{\sigma}_n^2=\frac{1}{n}\sum_{t=1}^nX_t^2 - (\frac{1}{n}\sum_{t=1}^nX_t)^2$.
\end{lemma}

\begin{lemma}\citep[Lemma 50]{chen2022policy}
	\label{lem:anytime freedman}
	Let $\{X_i\}_{i=1}^{\infty}$ be a martingale difference sequence adapted to the filtration $\{\calF_i\}_{i=0}^{\infty}$ and $|X_i|\leq B$ for some $B>0$.
	Then with probability at least $1-\delta$, for all $n\geq 1$ simultaneously,
	\begin{align*}
		\abr{\sum_{i=1}^nX_i}\leq 3\sqrt{\sum_{i=1}^n\E[X_i^2|\calF_{i-1}]\ln\frac{4B^2n^3}{\delta} } + 2B\ln\frac{4B^2n^3}{\delta}.
	\end{align*}
\end{lemma}

\end{document}